\documentclass[11pt]{article}

\usepackage[margin=1in]{geometry}

\usepackage[utf8]{inputenc}
\usepackage[T1]{fontenc}
\usepackage{microtype}
\usepackage{array}
\usepackage{tabularx}
\usepackage{amsmath,amssymb,amsfonts,amsthm,mathtools}
\usepackage{url}
\usepackage{booktabs}
\usepackage{nicefrac}
\usepackage{algorithm2e}
\usepackage{enumitem}
\usepackage{bbm}
\usepackage[authoryear,round]{natbib}
\usepackage{subcaption}
\usepackage{multirow}
\usepackage[dvipsnames]{xcolor}
\usepackage[hidelinks]{hyperref}
\hypersetup{
    colorlinks=true,
    linkcolor=MidnightBlue,
    citecolor=MidnightBlue,
    urlcolor=MidnightBlue
}
\usepackage{tikz}
\usepackage{amsmath,amssymb}
\usepackage{subcaption}    
\usetikzlibrary{calc,arrows.meta}

\newtheorem{theorem}{Theorem}[section]
\newtheorem{proposition}[theorem]{Proposition}
\newtheorem{corollary}[theorem]{Corollary}
\newtheorem{lemma}[theorem]{Lemma}

\theoremstyle{definition}
\newtheorem{definition}[theorem]{Definition}
\newtheorem{assumption}[theorem]{Assumption}
\newtheorem{example}[theorem]{Example}
\newtheorem{principle}[theorem]{Principle}

\DeclareMathOperator{\Proj}{Proj}
\DeclareMathOperator{\dist}{dist}

\DeclareMathOperator*{\argmin}{arg\,min}

\newcommand{\R}{\mathbb{R}}
\newcommand{\N}{\mathbb{N}}

\newcommand{\Y}{\mathcal{Y}}
\newcommand{\F}{\mathcal{F}}

\newcommand{\tarimpgap}{\Delta_T^{\text{gap}}}
\newcommand{\baseregret}{\Delta_T^{\text{base}}}
\newcommand{\tarmvment}{\Delta_T^{\text{mv}}}
\newcommand{\ip}[2]{\langle #1, #2 \rangle}
\newcommand{\norm}[1]{\| #1 \|}
\newcommand{\dnorm}[1]{\| #1 \|_*}
\newcommand{\distnorm}{\dist_{\norm{\cdot}}}
\newcommand{\Projnorm}[2]{\Proj^{\norm{\cdot}}_{#1}\!\left(#2\right)}
\newcommand{\abs}[1]{| #1 |}
\newcommand{\1}{\mathbbm{1}}
\newcommand{\0}{\mathbf{0}}
\newcommand{\FigureGraphic}[2][]{%
  \IfFileExists{figures/#2}{\includegraphics[#1]{figures/#2}}{%
    \IfFileExists{#2}{\includegraphics[#1]{#2}}{\includegraphics[#1]{figures/#2}}%
  }%
}
\newcommand{\QueueInput}[1]{%
  \IfFileExists{figures/queue_diagram/#1}{\input{figures/queue_diagram/#1}}{%
    \IfFileExists{#1}{\input{#1}}{\input{figures/queue_diagram/#1}}%
  }%
}

\SetKwInput{KwInput}{Input}
\RestyleAlgo{ruled}
\allowdisplaybreaks

\title{Optimal Hidden-Target Learning for Online Inventory Optimization on General Convex Sets}
\author{
  Anthony Pineci \\
  UIUC \\
  \texttt{apineci2@illinois.edu}
  \and
  Yunzong Xu \\
  UIUC \\
  \texttt{xyz@illinois.edu}
}
\date{}

\begin{document}
\addtocontents{toc}{\protect\setcounter{tocdepth}{-1}}

\maketitle

\begin{abstract}
Online inventory optimization (OIO) is online convex optimization with physical memory: inventory carryover makes the feasible action set depend on the past. A natural principle, used in stochastic inventory learning and recently in OIO under a single linear capacity constraint, is to maintain a hidden target chosen by an online learner and implement its projection onto the currently feasible order-up-to set. We prove that this simple principle is optimal for OIO on arbitrary bounded convex capacity sets. With online gradient descent as the base learner, the method improves the best known regret guarantee for OIO on general convex sets from inverse to inverse-square-root dependence on the common-demand probability, and we prove a matching lower bound. The same principle gives the first polylogarithmic regret guarantee for strongly convex losses and the first dynamic regret guarantee adapting to Euclidean path variation on general convex capacity sets.

The analysis introduces a norm alignment principle: the right state variable is the distance from the hidden target to the feasible set, measured in the same norm as the projection. Under norm alignment, this distance evolves pathwise as a scalar queue, with target movement as arrival and common demand as service. This reduction to one-dimensional queue control resolves the state dependence and extends the guarantees to general convex capacity sets, beyond the reach of prior productwise approaches. Experiments on synthetic and real-world inventory data corroborate the theory.
\end{abstract}

\section{Introduction}
\label{sec:introduction}

Inventory control is a foundational problem in operations research and stochastic control~\citep{arrow1951optimal,arrow1958studies,dvoretzky1952inventory,dvoretzky1952inventory2,bellman1955optimal,porteus2002foundations}; classical inventory models were among the early motivations for dynamic programming.\footnote{According to \cite{bellman1958review}, ``Two of the most interesting classes of dynamic programming processes, viewed from the vantage points of both analysis and application, are those of inventory control and production smoothing.''} This paper studies its canonical periodic-review formulation: in each period, a retailer, warehouse, or fulfillment system selects target stock levels for multiple products, subject to general capacity constraints, while demand depletes inventory and costs are incurred over time. Recent work has brought this classical model into contact with online learning; see the survey by \citet{chao2023online}. From an online-learning viewpoint, each period in inventory control resembles a round in online convex optimization (OCO): the learner chooses an action, observes local cost information, and is evaluated against a comparator. The analogy is incomplete, however, because it misses the central physical feature of inventory control: stock carries over. Inventory can be replenished componentwise up to a chosen order-up-to vector, but it can decrease only when demand consumes it. Thus a gradient step may point toward a lower stock vector, while the inventory already on the shelf makes that target physically infeasible.

The online inventory optimization (OIO) model of \citet*{hihat2023online} is a recent formalization at the interface of inventory control and adversarial online learning. In OIO, the learner chooses an order-up-to vector in a capacity set $\Y\subseteq\R_+^n$, but the chosen vector must also dominate the current inventory state. If the inventory state were always zero, OIO would reduce to standard OCO on $\Y$~\citep{zinkevich2003online,shalev2012online,hazan2016introduction,orabona2019modern}. Otherwise, inventory carryover makes feasibility state-dependent: past actions constrain future choices. Compared with classical periodic-review inventory models~\citep{snyder_stochastic_2019}, OIO does not require a known stochastic demand model. Compared with data-driven inventory learning under censored demand~\citep{huh_nonparametric_2009,shi_nonparametric_nodate,zhang_technical_2018,lyu_minibatch_2024,guo2026online}, OIO broadens the learning formulation in two directions: losses and demand may vary adversarially over time, and capacity constraints may be arbitrary bounded convex sets.

The first general-convex OIO algorithm, MaxCOSD~\citep{hihat2023online}, handles the state-dependent feasibility through adaptive cycles: it waits for moments when a proposed update to the implemented order level is compatible with the inventory state. This idea is robust enough to work for arbitrary bounded convex capacity sets and adversarial losses. Under uniformly probably positive demand (UPPD), a non-degeneracy condition on the common demand (i.e., minimum demand over all products) with probability parameter $\mu$, its high-probability static regret is $\widetilde O(\sqrt T/\mu)$. The $1/\mu$ factor controls how long the system may wait before all products receive enough demand to reduce excess stock, and in large or intermittent-demand systems this waiting time can dominate the guarantee.

A different algorithmic principle has long been attractive in inventory learning. Rather than tracking the implemented level directly, a \textit{base learner} maintains a hidden desired level and projects it onto what the current state permits. This \textit{hidden-target} idea appears in queueing analyses for single-product learning~\citep{huh_nonparametric_2009}, in multiproduct algorithms under a linear warehouse constraint~\citep{shi_nonparametric_nodate}, and most explicitly in the recent hidden-target OIO method of \citet*{ichikawa2026nonstationary}. The principle is also empirically appealing: \citet{hihat2023online} reports that this idea can improve MaxCOSD's empirical performance in multi-product, multi-constraint settings, although no guarantee was known at MaxCOSD's level of generality. However, extending existing analytical techniques to the general capacity setting faces a methodological barrier. These analyses rely on the linear warehouse constraint and productwise sell-out structure, but Appendix~\ref{app:geometry} shows why such productwise control cannot be extended to general convex capacity sets. The reliance on productwise control also leads to additional polynomial dependence on the number of products $n$, which is undesirable from the norm-based viewpoint standard in OCO.

\begin{table}[htbp]
\centering
\small
\setlength{\tabcolsep}{3pt}
\renewcommand{\arraystretch}{1.0}
\begin{tabularx}{\linewidth}{
>{\raggedright\arraybackslash}p{0.27\linewidth}
>{\raggedright\arraybackslash}p{0.10\linewidth}
>{\raggedright\arraybackslash}p{0.15\linewidth}
>{\raggedright\arraybackslash}p{0.13\linewidth}
>{\raggedright\arraybackslash}X}
\toprule
Method (Reference) & Capacity Set & Demand Primitive & Regret type & Bound \\
\midrule
MaxCOSD~\citep{hihat2023online} & general & UPPD $\mu$ & static & $\widetilde{O}\!\bigl(\sqrt{T}/\mu\bigr)$ \\
\cmidrule(l){4-5}
\multirow{2}{=}{HT~\citep{ichikawa2026nonstationary}} & \multirow{2}{=}{linear} & \multirow{2}{=}{sell-out period $L_{\max}$} & static & $\widetilde O\!\bigl(n^{1/4}\sqrt{L_{\max}T}\bigr)$ \\
& & & dynamic & $\widetilde O\!\bigl(n^{1/4}\sqrt{L_{\max}(1+P_{T,1})T}\bigr)$ \\
\cmidrule(l){4-5}
\multirow{3}{=}{\textbf{HT~(Ours)}} & \multirow{3}{=}{general} & \multirow{3}{=}{UPPD $\mu$} & static (cvx) & \begin{tabular}[t]{@{}l@{}}$\widetilde O\!\bigl(\sqrt{T/\mu}\bigr)\,/\,\Omega\!\bigl(\sqrt{T/\mu}\bigr)$\\with further norm refinement\\via mirror descent\end{tabular} \\
 & & & static (s.c.) & $O(\log(T/\delta)\log (T) / \mu)$ \\
 & & & dynamic & $\widetilde O\!\bigl(\sqrt{(1+P_{T,2})T/\mu}\bigr)$ \\
\bottomrule
\end{tabularx}
\caption{Closest adversarial-demand OIO guarantees, suppressing problem-dependent constants. HT stands for hidden-target methods. For \citet{ichikawa2026nonstationary}, the $n^{1/4}$ term appears explicitly in intermediate bounds and is absorbed into the big-O notation of their final stated regret. Here $P_{T,1}$ and $P_{T,2}$ denote comparator path variation in $\ell_1$ and $\ell_2$, respectively. For \emph{(Ours)}, the static convex row lists the upper bound and matching lower bound; the mirror descent refinement keeps the same $\widetilde O(\sqrt{T/\mu})$ rate while exposing sharper norm-based parameters; the strongly convex (s.c.) and dynamic bounds appear on separate rows. Appendix~\ref{app:lmax} relates UPPD and $L_{\max}$.}
\label{tab:comparison}
\end{table}

\subsection{Our contributions} This paper shows that the simple and natural hidden-target projection method is already powerful enough for the full general-convex OIO model. The missing ingredient is not a more complicated algorithm, but a better geometric analysis. At each round, the base learner produces hidden targets $z_t \in \Y$. Given the current implementable set $J_t$, the algorithm implements the projection $y_t=\Projnorm{J_t}{z_t}$ measured in a given norm $\|\cdot\|$. We track the target-implementation gap $q_t=\distnorm(z_t,J_t)=\norm{z_t-y_t}$ in the \emph{same} norm $\norm{\cdot}$.
Letting $d_t^{\min}:=\min_i d_{t,i}$ denote common demand, we establish that $q_t$ is governed by a standard Lindley queue recursion:
\[
  q_{t+1}\le \bigl[q_t+\|z_{t+1}-z_t\|-{d}_t^{\min}\bigr]^+.
\]
This inequality is \emph{pathwise}: it holds for every realization of the inventory dynamics, with no averaging or distributional assumption. 
Its proof is purely geometric, using only the convexity of $\Y$, the inventory dynamics, and the norm alignment principle (Section~\ref{sec:norm_alignment}), rather than any special structure of the capacity set. Thus the target-implementation gap evolves as a one-dimensional queue: moving the learning target creates arrival while common demand supplies service.

At this level of generality, the queue reduction is surprising. Queueing ideas entered inventory learning through \citet{huh_nonparametric_2009} and \citet{shi_nonparametric_nodate}, who used them to track \emph{productwise} feasibility errors. Such queues naturally handle a single product or linear constraint, but extending them to general capacity constraints encounters genuine obstructions; see Appendix~\ref{app:geometry} and \cite{lyu_minibatch_2024}. Recent work has therefore replaced queueing with alternative devices, including minibatching, cyclic bookkeeping, or overshooting-loss control \citep{hihat2023online,lyu_minibatch_2024,ichikawa2026nonstationary,guo2026online}. Our reduction revives the queueing route by changing the state variable: the queue is the \emph{norm-aligned} distance from the hidden target to the implementable set. This geometric Lindley queue holds independent of specialized capacity structure, so its distribution-free control reduces OIO regret to the base learner's OCO regret plus a switching-cost term under the UPPD condition.

The reduction yields several state-of-the-art consequences, summarized in Table~\ref{tab:comparison}.
\begin{enumerate}
    \item Under an instantiation of the hidden-target method with a simple online gradient descent (OGD) base learner, our static regret results improve the $\mu$ dependence of the best known general-convex OIO guarantee from $\mu^{-1}$ to $\mu^{-1/2}$ and establish optimality via a new lower bound. Even under a single linear capacity constraint, our queue-based analysis removes the explicit $n^{1/4}$ factor exposed by the hidden-target analysis of \citet{ichikawa2026nonstationary}.
    \item Under strongly convex losses, the same algorithm obtains polylogarithmic regret, addressing an open direction raised by \citet{hihat2023online}.
    \item In dynamic environments, the bound depends on the Euclidean path variation $P_{T,2}$ rather than the larger $\ell_1$ variation $P_{T,1}$, sharpening \citet{ichikawa2026nonstationary}'s result on more general geometry. We instantiate this with OGD when $P_{T, 2}$ is known a priori and smoothed OGD (SOGD) \citep{zhang2022soco} when it is not; both achieve the same regret rate (up to polylogarithmic factors).
    \item The reduction is not tied to Euclidean projection. If the base learner is online mirror descent (OMD) with a strongly convex regularizer $R$, then a new HT-OMD algorithm obtains static regret $\widetilde O\!\bigl(G_*\sqrt{\mathcal B_R T/(\sigma\mu)}\bigr)$, where $G_*$ is the corresponding dual-gradient bound, $\sigma$ is the strong-convexity modulus of $R$, and $\mathcal B_R$ is the relevant Bregman diameter. Thus arbitrary bounded convex capacity sets admit geometry-sensitive refinements once paired with a suitable regularizer. This extends the recent use of OMD in stochastic inventory learning \citep{guo2026online} to adversarial OIO and  general convex capacity sets.
\end{enumerate} 

All our guarantees are \emph{dimension-free} in the standard norm-based sense: dimension enters through problem parameters such as the capacity set diameter, gradient norm, and $\mu$, but not through an additional explicit factor of dimension $n$. While the parameter $\mu$ can still deteriorate with dimension, our result isolates this difficulty in the demand primitive instead of adding a separate algorithmic dimension penalty.

\subsection{Relation to prior work}
\label{sec:related_work}

\paragraph{Online convex optimization.}
Projected online gradient descent gives the classical $O(\sqrt T)$ regret guarantee for Lipschitz convex losses~\citep{zinkevich2003online,shalev2012online,hazan2016introduction,orabona2019modern}, and logarithmic regret under curvature appears in the strongly convex and exp-concave OCO literature~\citep{hazan2007logarithmic}. Dynamic regret compares against a changing comparator sequence and is controlled by path variation~\citep{hall2013dynamical,jadbabaie2015online}. Our OGD and strongly-convex analyses build on standard OCO arguments; the contribution is to show that the additional complexity of carryover effects in inventory control can be reduced to a switching-cost term via hidden-target learning. For unknown path variation, we use the smoothed OCO meta-learner of \citet{zhang2022soco} as a base learner.

\paragraph{Stochastic inventory learning.}
Stochastic inventory learning asks how to make inventory decisions when demand distributions are unknown, often under censoring, lost sales, or capacity constraints; see \cite{chao2023online} for a survey. A recurring obstacle is implementability: a learning rule may propose an inventory target that is not feasible from the current state. One influential response is queueing-style analysis, from the single-product recursion of \citet{huh_nonparametric_2009} to the multiproduct  analysis of \citet{shi_nonparametric_nodate}. This queueing line has continued in several structured inventory-learning models
\citep{chen2020random_capacity,yuan2021fixed_costs,ding2024feature_based,yang_huh2024multiechelon}, but its arguments remain tied to a single product or linear capacity constraint. A complementary line has moved toward alternative algorithmic and analytical devices for refined guarantees or richer stochastic models. \citet{guo2026online} used cyclic online mirror descent to improve the product-dimension dependence in the linear-capacity setting of \citet{shi_nonparametric_nodate}; \citet{zhang_technical_2018} used cycle subgradients and a bridging construction for perishable inventory; and \citet{lyu_minibatch_2024} developed a minibatch-SGD metapolicy for several stochastic inventory systems, including multiproduct systems with multiple linear constraints and multiechelon serial systems.

Our contribution is complementary to both lines. We show that queueing structure persists in far more general capacity geometry once the right state variable is used. The hidden target can be generated by a standard online convex optimization method, while a geometric queue-clearing argument enforces implementability. Thus, in our setting, feasibility does not require cyclic schedules, bridging constructions, or minibatching. The result gives guarantees for adversarial loss sequences on general convex capacity sets.

\paragraph{Online inventory optimization.} To generalize stochastic inventory learning beyond i.i.d. demand, fixed dynamics, and newsvendor-specific losses, \citet{hihat2023online} introduced the adversarial OIO model, which contains OCO as a special case. Their MaxCOSD algorithm works for general convex capacity sets but controls learning through cycle boundaries, leading to $\mu^{-1}$ dependence under the UPPD parameter $\mu$. \citet{ichikawa2026nonstationary} analyzed hidden-target projection and obtained dynamic regret guarantees in nonstationary environments, but their proof crucially relies on the linear capacity set and specialized productwise sell-out analysis (see Appendix~\ref{app:geometry}). The present paper keeps the hidden-target principle and replaces the sell-out analysis by a Euclidean/norm-aligned queue, giving optimal $\mu$ dependence on general convex sets.

\paragraph{Learning-based inventory control.}
Deep reinforcement learning and simulation-based methods have recently become effective for large inventory systems with rich dynamics~\citep{madeka_deep_2022,alvo_deep_2025,maggiar_structure-informed_2025,bloem_simulated-based_2025}. These methods emphasize empirical performance, differentiable simulation, network structure, or supply-chain-scale policies. They are complementary to the present work, whose goal is to establish sharp adversarial regret guarantees for simple and lightweight online-learning methods.

\subsection{Notation and Organization} We denote $\N$ as the set of positive integers and write $[m] := \{1, \dots, m\}$ for any $m \in \N$. Let $\R^n$ denote $n$-dimensional Euclidean space, and let $\R^n_+$ and $\R^n_{++}$ denote its nonnegative and strictly positive orthants, respectively. For a scalar $a$, write $a^+:=\max\{a,0\}$; for a vector
$v\in\R^n$, $v^+$ denotes the coordinatewise nonnegative part,
$(v^+)_i:=\max\{v_i,0\}$. We use $a\succeq b$ for coordinatewise inequality. Throughout the framework, $\norm{\cdot}$ is an admissible norm with dual $\dnorm{\cdot}$ (see Section~\ref{sec:admissible_norms}), while instantiations in Section~\ref{sec:instantiations} make further specifications on the exact norm that is used.  Probabilistic statements are with respect to a filtration $(\F_t)_{t\ge0}$. The learner's action at round $t$ is $\F_{t-1}$-measurable, whereas the demand and loss chosen after the action are $\F_t$-measurable. For nonnegative quantities $f$ and $g$, we write $f=O(g)$ and $f=\Omega(g)$ to mean that $f\le Cg$ and $f\ge cg$, respectively, for constants $C,c>0$ independent of $T,\mu,L_{\max},n$. We write $f=\Theta(g)$ if both bounds hold. We use $\widetilde O(\cdot)$ to suppress polylogarithmic factors in $T$.

The rest of the paper is organized as follows. Section~\ref{sec:model} introduces the OIO model, the hidden-target projection algorithm, and the norm-alignment principle. Section~\ref{sec:queue_reduction} proves the geometric Lindley recursion and the general hidden-target reduction. Section~\ref{sec:instantiations} instantiates the reduction with OGD, SOGD, and OMD to obtain static, strongly convex, dynamic, and mirror-descent regret guarantees. Section~\ref{sec:experiments} presents numerical experiments, and Section~\ref{sec:conclusion} concludes. The appendix contains experiment details, the obstruction to productwise queue analyses on general convex sets, a discussion of UPPD, and all omitted proofs.

\section{Model and hidden-target projection}
\label{sec:model}

\subsection{The OIO model}

We work in the model of \citet{hihat2023online}, generalized to adaptive adversarial settings for the environment's choice of demand and loss function. The environment first sets the initial inventory state to zero, i.e. $x_1 = \0$, then for every round $t=1, \dots, T$:

\begin{enumerate}
    \item The learner observes an inventory state $x_t \in \R_+^n$ and must choose an order level $y_t \in \Y$ that obeys the \textit{inventory feasibility constraint}:
        \begin{equation}
        \label{eq:inv_feasibility_constraint}
          y_t \succeq x_t
        \end{equation}
        The inventory is instantly replenished to $y_t$. 
    \item The environment observes the learner's action and chooses a demand $d_t \in \R^n_+$, then a loss function $\ell_t : \Y \to \R$. The next inventory state $x_{t+1}\in \R_+^n$ is updated so that it adheres to the \textit{inventory dynamics constraint}:
        \begin{equation}
        \label{eq:inv_dynamics_constraint}
          x_{t+1} \preceq (y_t-d_t)^+
        \end{equation}
    \item The learner incurs a loss $\ell_t(y_t)$ and observes a subgradient $g_t \in \partial \ell_t(y_t)$ that may help them make a choice in the next round.
\end{enumerate}

We use the same convexity and boundedness assumptions as \citet{hihat2023online}:

\begin{assumption}[Convex and bounded OIO problem]
\label{ass:convex_and_bounded}
Fix an admissible norm $\norm{\cdot}$ on $\R^n$ (Definition~\ref{def:admissible_norm}) with dual norm $\dnorm{\cdot}$. The capacity set $\Y\subseteq\R_+^n$ is nonempty, closed, convex, and bounded with diameter $D:=\sup_{u,v\in\Y}\norm{u-v}$ measured in this norm. Each loss $\ell_t:\Y\to\R$ is convex, and every observed subgradient $g_t\in\partial\ell_t(y_t)$ satisfies $\dnorm{g_t}\le G_*$.
\end{assumption}

Equality in \eqref{eq:inv_dynamics_constraint} gives the usual lost-sales dynamics. The inequality allows for alternative dynamics incorporating additional depletion or perishability, as in \citet{hihat2023online}. If the state never binds, the only constraint is $y_t\in\Y$ and the model is ordinary OCO on $\Y$. The difficulty is precisely that the feasible action set at time $t$ is shaped by earlier order-up-to decisions.

Since the demand is chosen prior to the loss function, this framework accommodates the standard newsvendor loss used in the inventory literature:
\begin{equation}
\label{eq:newsvendor_loss}
  \ell_t(y;d_t)
  =
  \sum_{i=1}^n
  \Bigl(h_{t,i}(y_i-d_{t,i})^+ + p_{t,i}(d_{t,i}-y_i)^+\Bigr).
\end{equation}
Here $h_t, p_t \in \R_+^n$ correspond to the holding cost (overage cost) and lost-sales penalty (underage cost) respectively. Following the OIO framework \citep{hihat2023online}, we do not permit the loss function $\ell_t$ to depend on the current inventory $x_t$. However, under a lost-sales dynamic, newsvendor loss \eqref{eq:newsvendor_loss} can naturally accommodate purchasing costs using a standard cost transformation \citep[Paragraph~4.3.2.4]{snyder_stochastic_2019}.

The framework's requirement for subgradient observability does not require full demand observation. A realistic assumption is that the learner may only observe sales $s_t=\min\{y_t,d_t\}$ taken coordinatewise. In this case, the vector with coordinates given by 
\[
  h_{t,i}\1\{y_{t,i}>s_{t,i}\}-p_{t,i}\1\{y_{t,i}=s_{t,i}\}
\]
is an observable subgradient of \eqref{eq:newsvendor_loss} at $y_t$. This shows that our subgradient requirement is sufficient for the canonical inventory loss, while the analysis permits arbitrary convex losses with local subgradient feedback.

We measure algorithm performance using \textit{Regret}. For a comparator sequence $u_{1:T}\in\Y^T$, dynamic regret is
\[
  R_T(u_{1:T}) := \sum_{t=1}^T\bigl(\ell_t(y_t)-\ell_t(u_t)\bigr).
\]
Static regret is the special case $u_t\equiv u$.

\subsection{Hidden-target projection and the regret decomposition}

At time $t$, define the implementable set
\(
  J_t:=\Y\cap\{y\in\R^n\mid y\succeq x_t\}.
\)
This set is always nonempty: if $y_t\in J_t$, then \eqref{eq:inv_dynamics_constraint} gives $x_{t+1}\preceq y_t$, so the same point $y_t$ belongs to $J_{t+1}$. We work in the standard projection-oracle model of OCO, assuming nearest-point projections in the norm $\norm{\cdot}$ onto $\Y$ and onto $J_t$ can be computed.

Hidden-target projection lets an online learner move in the fixed set $\Y$ while a separate inventory layer enforces physical feasibility. When the base learner proposes $z_t\in\Y$, the implemented order-up-to level is
\begin{equation}
\label{eq:hidden-target-projection}
  y_t=\Projnorm{J_t}{z_t}\in\argmin_{y\in J_t}\norm{z_t-y}.
\end{equation}
This is the algorithmic principle behind projected base-stock methods such as DDM~\citep{shi_nonparametric_nodate} and the hidden-target OIO method of \citet{ichikawa2026nonstationary}. After implementing $y_t$, hidden-target projection observes only a subgradient $g_t \in \partial \ell_t(y_t)$ at the implemented point. The linear loss $f_t(z) = \ip{g_t}{z}$ is therefore the most informative function we can construct from this feedback, and so we take the base learner to be an online linear optimizer (OLO). The implementation is formalized in Algorithm~\ref{alg:hidden_target_meta}.

\begin{algorithm}[H]
\caption{Hidden-Target Projection Meta-Algorithm}
\label{alg:hidden_target_meta}
\KwInput{Capacity set $\Y$, base learner (OLO) on $\Y$, initial target $z_1\in J_1$}
\For{$t=1,2,\ldots,T$}{
Observe inventory state $x_t$ and set $J_t=\Y\cap\{y\in\R^n\mid y\succeq x_t\}$\;
Implement $y_t=\Projnorm{J_t}{z_t}$\;
Observe $g_t\in\partial\ell_t(y_t)$ and the next state $x_{t+1}$\;
Feed the linearized loss $f_t(z)=\ip{g_t}{z}$ to the base learner and receive $z_{t+1}\in\Y$\;
}
\end{algorithm}

The regret analysis begins with a decomposition inherited from earlier hidden-target arguments, now measured in the general norm:
\[
  \baseregret(u_{1:T}) := \sum_{t=1}^T \ip{g_t}{z_t-u_t},
  \qquad
  \tarimpgap := \sum_{t=1}^T \norm{z_t-y_t},
  \qquad
  \tarmvment := \sum_{t=1}^{T-1}\norm{z_{t+1}-z_t}.
\]
The first quantity is the ordinary OCO regret of the base learner. The second is the cost of physical implementation of the target. The third tracks target movement, which the queue reduction will charge as a switching cost on the base learner.

\begin{lemma}[Target regret plus implementation error]
\label{lem:regret_decomp}
Let $(z_t)_{t=1}^T\in\Y^T$ be any hidden-target sequence, and let $(y_t)_{t=1}^T$ be defined by \eqref{eq:hidden-target-projection}. Then for every comparator sequence $u_{1:T}\in\Y^T$,
\begin{equation}
\label{eq:decomp}
  R_T(u_{1:T})\le \baseregret(u_{1:T})+G_*\tarimpgap.
\end{equation}
\end{lemma} 

The lemma reduces regret to a base learner's OCO regret plus the cumulative implementation gap $\tarimpgap$. However, controlling this gap is challenging because past actions carry over into the current implementable set $J_t$. Instead, we relate $\tarimpgap$ to the base learner's target movement $\tarmvment$ through a queue recursion (Section~\ref{sec:queue_reduction}). At every round, this queue's length is the gap $q_t = \norm{z_t - y_t}$, with target movement $\norm{z_{t+1} - z_t}$ as arrival and common demand $d_t^{\min} = \min_{i} d_{t, i}$ as service. If additionally the base learner's cumulative target movement is bounded over every clearing window of length $B_\delta = \widetilde O (\mu^{-1})$, then UPPD ensures with high probability that the common demand clears the queue, giving a bound $\tarimpgap \leq B_\delta \tarmvment$. Thus, regret from the implementation gap becomes a switching cost weighted by the length of the clearing window.

For a single linear warehouse constraint, prior work~\citep{shi_nonparametric_nodate,ichikawa2026nonstationary} tracks the implementation term product by product, relying on the linear budget $\sum_i y_i \le D$ to tie the productwise bounds together. Curved convex sets have no such budget. Without it, a projection that drives one product up to meet inventory can force unrelated products down by arbitrarily larger amounts, breaking the per-product correspondence. Appendix~\ref{app:geometry} formalizes why no per-product analysis extends to general convex sets. The next section develops a different approach which controls the implementation term through a single geometric distance without using any per-product accounting.

\subsection{Admissible norms}
\label{sec:admissible_norms}

The queue recursion driving the analysis (Proposition~\ref{prop:target_implementation_gap_queue} below) holds for a broad class of norms, characterized by a single compatibility condition with the coordinatewise structure of the inventory model. We call this class \emph{admissible}.

\begin{definition}[Admissible norm]
\label{def:admissible_norm}
A norm $\norm{\cdot}$ on $\R^n$ is \emph{admissible} for OIO if it dominates the $\ell_\infty$ norm:
\[
  \norm{x}_\infty \le \norm{x} \qquad \text{for every } x\in\R^n.
\]
Equivalently, $\abs{x_i}\le\norm{x}$ for every $x\in\R^n$ and every $i\in[n]$.
\end{definition}

Geometrically, a norm is admissible if its unit ball is contained in the $\ell_\infty$ unit cube. This condition is precisely what is needed to convert movement in the norm $\norm{\cdot}$ to a coordinatewise bound, which is in turn what allows the inventory dynamics constraint \eqref{eq:inv_dynamics_constraint} to be satisfied through measuring the movement of order levels.

\begin{example}
\label{ex:admissible_norms}
Every $\ell_p$ norm with $p\in[1,\infty]$ is admissible. A diagonally weighted Euclidean norm $\norm{x}_W:=\sqrt{x^\top W x}$ with $W=\mathrm{diag}(w_1,\dots,w_n)$ is admissible if and only if $w_i\ge 1$ for every $i$.
\end{example}

\subsection{The norm alignment principle}
\label{sec:norm_alignment}

The queue recursion driving our analysis depends on a single methodological commitment that we now state explicitly. Throughout the paper, we require the projection, the queue state, the target movement, and every lemma and proposition in the queue analysis to be expressed in the \emph{same} admissible norm $\norm{\cdot}$. We call this requirement the \emph{norm alignment principle}:

\begin{principle}[Norm alignment]
\label{prin:norm_alignment}
A hidden-target analysis is \emph{norm-aligned} if a single admissible norm $\norm{\cdot}$ serves simultaneously as:
\begin{enumerate}[leftmargin=1.7em,itemsep=1pt,topsep=2pt]
  \item The projection norm, $y_t = \Projnorm{J_t}{z_t}$,
  \item The queue norm, $q_t = \norm{z_t - y_t}$,
  \item The movement norm, $\norm{z_{t+1} - z_t}$,
  \item The norm in which the queueing recursion (Proposition~\ref{prop:target_implementation_gap_queue}) and its supporting lemmas are stated and proved.
\end{enumerate}
\end{principle}

We introduce this principle because it is what makes the queue recursion close without any multiplicative factors. Under norm alignment, the queue recursion follows geometrically using only convexity of $\Y$, the norm axioms, and optimality of the nearest-point projection. The requirement is fundamental, not just a way to optimize constants or dimension factors. If the algorithm projects in one norm but the queue is measured in another, norm equivalence gives at best a multiplicative constant inside the queue recursion. Such a factor compounds over a clearing window, yielding an exponentially looser bound. Norm conversion is therefore safe only after the aligned queue recursion has been established. 

Hidden-target analyses that mix norms across these roles, such as the $\ell_1$ switching cost with $\ell_2$ projection used by \citet{shi_nonparametric_nodate} and \citet{ichikawa2026nonstationary}, must compensate with constraint-specific structure to bypass the geometric queuing argument presented in this paper. For more details on why previous analytical approaches did not generalize to arbitrary convex sets, see Appendix~\ref{app:geometry}.

\section{The target-implementation queue}
\label{sec:queue_reduction}

\begin{figure}
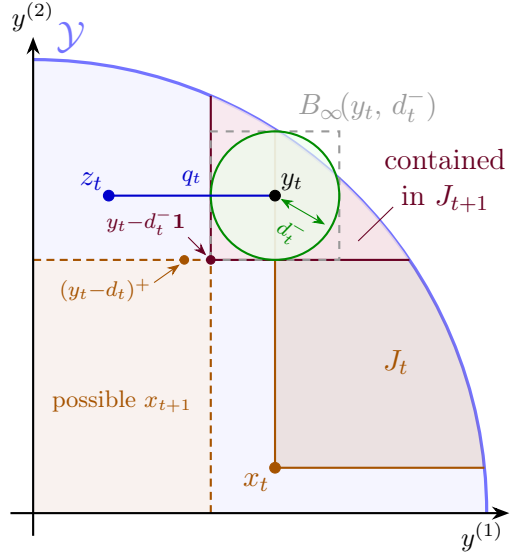

  \centering
\IfFileExists{figures/queue_diagram/tikz_common.tex}{%
  \input{figures/queue_diagram/tikz_common.tex}%
}{%
  \input{tikz_common.tex}%
}%
\begin{tikzpicture}[%
  >=Stealth,%
  every node/.style={font=\small},%
  scale=1.0%
]%

\pgfmathsetmacro{\boxlx}{\ytx - \dtmin}
\pgfmathsetmacro{\boxly}{\yty - \dtmin}
\pgfmathsetmacro{\ytmdtx}{\ytx - \dtx}
\pgfmathsetmacro{\ytmdty}{\yty - \dty}
\pgfmathsetmacro{\JtopY}{sqrt(\Yrad*\Yrad - \xtx*\xtx)}%
\pgfmathsetmacro{\JrightY}{sqrt(\Yrad*\Yrad - \xty*\xty)}%
\pgfmathsetmacro{\JarcStart}{asin(\xty/\Yrad)}%
\pgfmathsetmacro{\JarcEnd}{acos(\xtx/\Yrad)}%
\pgfmathsetmacro{\CntopY}{sqrt(\Yrad*\Yrad - \boxlx*\boxlx)}%
\pgfmathsetmacro{\CnrightY}{sqrt(\Yrad*\Yrad - \boxly*\boxly)}%
\pgfmathsetmacro{\CnarcStart}{asin(\boxly/\Yrad)}%
\pgfmathsetmacro{\CnarcEnd}{acos(\boxlx/\Yrad)}%

\fill[Yfill] (0,0) -- (\Yrad, 0) arc (0:90:\Yrad) -- cycle;%
\draw[Ycol, very thick] (\Yrad, 0) arc (0:90:\Yrad);%

\begin{scope}%
  \clip (0,0) rectangle (\Yrad, \Yrad);%
  \fill[Jcol, opacity=0.12]%
    (\xtx, \xty) -- (\JrightY, \xty)%
    arc (\JarcStart:\JarcEnd:\Yrad)%
    -- (\xtx, \JtopY) -- cycle;%
\end{scope}%
\draw[Jcol, thick] (\xtx, \xty) -- (\xtx, \JtopY);%
\draw[Jcol, thick] (\xtx, \xty) -- (\JrightY, \xty);%

\fill[possfill, opacity=0.5] (0, 0) rectangle (\boxlx, \boxly);%
\draw[posscolor, thick, densely dashed] (\boxlx, 0) -- (\boxlx, \boxly);%
\draw[posscolor, thick, densely dashed] (0, \boxly) -- (\boxlx, \boxly);%

\begin{scope}%
  \clip (\boxlx, \boxly) rectangle (\Yrad, \Yrad);%
  \fill[containfill, opacity=0.8]%
    (\boxlx, \boxly) -- (\CnrightY, \boxly)%
    arc (\CnarcStart:\CnarcEnd:\Yrad)%
    -- (\boxlx, \CntopY) -- cycle;%
\end{scope}%
\draw[containcol, thick] (\boxlx, \boxly) -- (\boxlx, \CntopY);%
\draw[containcol, thick] (\boxlx, \boxly) -- (\CnrightY, \boxly);%

\draw[inftycolor, dashed, thick]%
  (\boxlx, \boxly) rectangle (\ytx + \dtmin, \yty + \dtmin);%

\fill[ballfill, opacity=0.6] (\ytx, \yty) circle (\dtmin);%
\draw[ballcol, thick] (\ytx, \yty) circle (\dtmin);%

\draw[ztcol, thick] (\ztx, \zty) -- (\ytx, \yty);%

\filldraw[posscolor] (\ytmdtx, \ytmdty) circle (1.6pt);%

\filldraw[ztcol] (\ztx, \zty) circle (2pt);%
\node[above left, ztcol, font=\normalsize, inner sep=2pt] at (\ztx, \zty) {$z_t$};%

\filldraw[ytcol] (\ytx, \yty) circle (2pt);%
\node[above right, inner sep=2pt] at (\ytx, \yty) {$y_t$};%

\filldraw[xtcol] (\xtx, \xty) circle (2pt);%
\node[below left, xtcol, font=\normalsize, inner sep=2pt] at (\xtx, \xty) {$x_t$};%

\filldraw[cornercolor] (\boxlx, \boxly) circle (1.6pt);%

\draw[->, thick] (-0.3, 0) -- (\Yrad + 0.3, 0);%
\node[font=\small, anchor=north east, inner sep=2pt] at (\Yrad + 0.3, -0.04) {$y^{(1)}$};%
\draw[->, thick] (0, -0.3) -- (0, \Yrad + 0.3);%
\node[font=\small, above, inner sep=2pt] at (0, \Yrad + 0.3) {$y^{(2)}$};%


\path (\ztx, \zty) -- (\ytx, \yty)%
  node[ztcol, font=\footnotesize, pos=0.5, above] {$q_t$};%

\pgfmathsetmacro{\dtmar}{0.1*\dtmin}%
\draw[ballcol, thin, <->,
      shorten >=\dtmar cm,
      shorten <=\dtmar cm]%
  (\ytx, \yty) -- node[below, font=\scriptsize, sloped] {$d_t^-$}%
  ++(330:\dtmin);%


\node[Ycol, font=\Large] at (\Ycalx, \Ycaly) {$\mathcal{Y}$};%

\node[Jcol, font=\normalsize] at (\Jtx, \Jty) {$J_t$};%

\node[inftycolor, font=\normalsize, anchor=west]%
  at (\ytx + \dtmin*0.2, \yty + \dtmin*1.4)%
  {$B_{\infty}\!(y_t,\, d_t^-)$};%

\node[containcol, font=\normalsize, align=center] (containlabel)%
  at (5.45, 4.4)%
  {contained\\ in $J_{t+1}$};%
\draw[containcol, thin, -] (containlabel.south west) + (0.25, 0.2) -- (4.3, 3.6);%

\node[cornercolor, font=\scriptsize, anchor=south east] (cornerlabel)%
  at (\boxlx - 0.2, \boxly + 0.2)%
  {$y_t {-} d_t^-\mathbf{1}$};%
\draw[cornercolor, thin, ->] (cornerlabel.south east) + (-0.1, 0.1) -- (\boxlx - 0.06, \boxly + 0.06);%

\node[posscolor, font=\footnotesize, align=center]%
  at (\boxlx/2, \boxly/2 - 0.2)%
  {possible $x_{t+1}$};%

\node[posscolor, font=\scriptsize, anchor=east] (ydlabel)%
  at (\ytmdtx - 0.3, \ytmdty - 0.4)%
  {$(y_t {-} d_t)^+$};%
\draw[posscolor, thin, ->] (ydlabel.east) + (-0.1, 0.1) -- (\ytmdtx - 0.06, \ytmdty - 0.06);%

\ifshowDebug
  \draw[red, very thick, dashed]%
    (current bounding box.south west)%
    rectangle%
    (current bounding box.north east);%
  \draw[blue!30, very thin]%
    (current bounding box.south west) grid (current bounding box.north east);%
\fi

\end{tikzpicture}%
  \caption{A geometric illustration of the short move lemma. It shows that any $y' \in B_{\norm{\cdot}}(y_t, d_t^{\min})$ satisfies $y' \succeq x_{t+1}$, so if $y'$ is also in $\Y$, then $y' \in J_{t+1}$.}
  \label{fig:short-move}
\end{figure}

\begin{figure}[t]
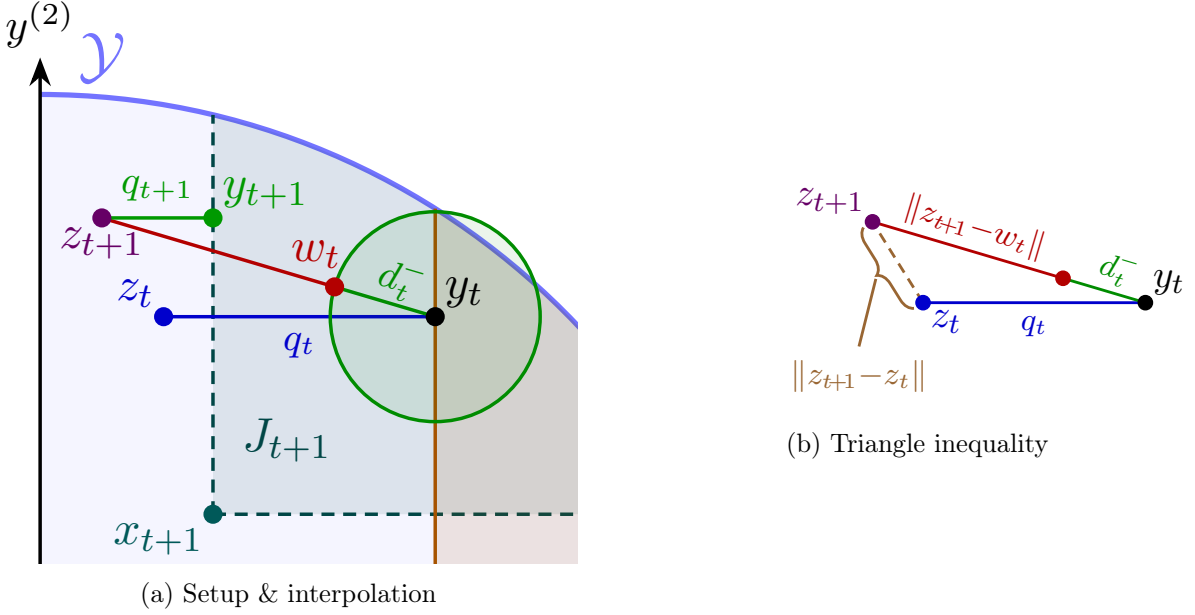

  \centering
  \hfill
  \begin{subfigure}[c]{0.46\textwidth}
    \centering
    \resizebox{\linewidth}{!}{%
\IfFileExists{figures/queue_diagram/tikz_common.tex}{%
  \input{figures/queue_diagram/tikz_common.tex}%
}{%
  \input{tikz_common.tex}%
}%
\begin{tikzpicture}[%
  >=Stealth,%
  every node/.style={font=\small},%
  scale=1.0%
]%

\ifcropToSection
  \clip (-0.3, \xtny - 0.4) rectangle (\Yrad-1.65, \Yrad+1);%
\fi

\pgfmathsetmacro{\seglen}{veclen(\ztnx-\ytx, \ztny-\yty)}%
\pgfmathsetmacro{\alphat}{\dtmin / \seglen}%
\pgfmathsetmacro{\wtx}{(1-\alphat)*\ytx + \alphat*\ztnx}%
\pgfmathsetmacro{\wty}{(1-\alphat)*\yty + \alphat*\ztny}%
\pgfmathsetmacro{\JtopY}{sqrt(\Yrad*\Yrad - \xtx*\xtx)}%
\pgfmathsetmacro{\JrightY}{sqrt(\Yrad*\Yrad - \xty*\xty)}%
\pgfmathsetmacro{\arcStart}{asin(\xty/\Yrad)}%
\pgfmathsetmacro{\arcEnd}{acos(\xtx/\Yrad)}%
\pgfmathsetmacro{\JntopY}{sqrt(\Yrad*\Yrad - \xtnx*\xtnx)}%
\pgfmathsetmacro{\JnrightY}{sqrt(\Yrad*\Yrad - \xtny*\xtny)}%
\pgfmathsetmacro{\JnarcStart}{asin(\xtny/\Yrad)}%
\pgfmathsetmacro{\JnarcEnd}{acos(\xtnx/\Yrad)}%

\fill[Yfill] (0,0) -- (\Yrad, 0) arc (0:90:\Yrad) -- cycle;%
\draw[Ycol, very thick] (\Yrad, 0) arc (0:90:\Yrad);%

\begin{scope}%
  \clip (0,0) rectangle (\Yrad, \Yrad);%
  \fill[Jncol, opacity=0.10]%
    (\xtnx, \xtny) -- (\JnrightY, \xtny)%
    arc (\JnarcStart:\JnarcEnd:\Yrad)%
    -- (\xtnx, \JntopY) -- cycle;%
\end{scope}%
\draw[Jncol, thick, densely dashed] (\xtnx, \xtny) -- (\xtnx, \JntopY);%
\draw[Jncol, thick, densely dashed] (\xtnx, \xtny) -- (\JnrightY, \xtny);%

\begin{scope}%
  \clip (0,0) rectangle (\Yrad, \Yrad);%
  \fill[Jcol, opacity=0.12]%
    (\xtx, \xty) -- (\JrightY, \xty)%
    arc (\arcStart:\arcEnd:\Yrad)%
    -- (\xtx, \JtopY) -- cycle;%
\end{scope}%
\draw[Jcol, thick] (\xtx, \xty) -- (\xtx, \JtopY);%
\draw[Jcol, thick] (\xtx, \xty) -- (\JrightY, \xty);%

\draw[->, thick] (-0.3, 0) -- (\Yrad + 0.3, 0);%
\node[font=\small, anchor=north east, inner sep=2pt] at (\Yrad + 0.3, -0.04) {$y^{(1)}$};%
\draw[->, thick] (0, -0.3) -- (0, \Yrad + 0.3);%
\node[font=\small, above, inner sep=2pt] at (0, \Yrad + 0.3) {$y^{(2)}$};%


\draw[ztcol, thick] (\ztx, \zty) -- (\ytx, \yty);%

\fill[ballcol, opacity=0.08] (\ytx, \yty) circle (\dtmin);%
\draw[ballcol, thick] (\ytx, \yty) circle (\dtmin);%

\draw[ballcol, thick] (\ytx, \yty) -- (\wtx, \wty);%


\draw[wtcol, thick] (\wtx, \wty) -- (\ztnx, \ztny);%

\draw[ytncolor, thick] (\ztnx, \ztny) -- (\ytnx, \ytny);%


\path (\ztx, \zty) -- (\ytx, \yty)%
  node[ztcol, font=\footnotesize, pos=0.5, below] {$q_t$};%

\path (\ytx, \yty) -- (\wtx, \wty)%
  node[ballcol, font=\footnotesize, pos=0.4, sloped, inner sep=0pt, above=1pt] {$d_t^-$};%

\path (\ztnx, \ztny) -- (\ytnx, \ytny)%
  node[ytncolor, font=\footnotesize, pos=0.5, above] {$q_{t+1}$};%

\filldraw[ztcol] (\ztx, \zty) circle (2pt);%
\node[above left, ztcol, font=\normalsize, inner sep=2pt] at (\ztx, \zty) {$z_t$};%

\filldraw[ytcol] (\ytx, \yty) circle (2pt);%
\node[above right, ytcol, font=\normalsize, inner sep=2pt] at (\ytx, \yty) {$y_t$};%

\filldraw[wtcol] (\wtx, \wty) circle (2pt);%
\node[above, wtcol, font=\normalsize, inner sep=2pt] at (\wtx - 0.15, \wty+0.075) {$w_t$};%

\filldraw[ztncol] (\ztnx, \ztny) circle (2pt);%
\node[below, ztncol, font=\normalsize, inner sep=2pt] at (\ztnx, \ztny) {$z_{t+1}$};%

\filldraw[xtcol] (\xtx, \xty) circle (2pt);%
\node[below left, xtcol, font=\normalsize, inner sep=2pt] at (\xtx, \xty) {$x_t$};%

\filldraw[xtncolor] (\xtnx, \xtny) circle (2pt);%
\node[below left, xtncolor, font=\normalsize, inner sep=2pt] at (\xtnx, \xtny) {$x_{t+1}$};%

\filldraw[ytncolor] (\ytnx, \ytny) circle (2pt);%
\node[above right, ytncolor, font=\normalsize, inner sep=2pt] at (\ytnx, \ytny) {$y_{t+1}$};%

\node[Ycol, font=\Large] at (\Ycalx, \Ycaly) {$\mathcal{Y}$};%
\node[Jcol, font=\normalsize] at (\Jtx, \Jty) {$J_t$};%
\node[Jncol, font=\normalsize] at (\Jtpx, \Jtpy) {$J_{t+1}$};%

\ifshowDebug
  \draw[red, very thick, dashed]%
    (current bounding box.south west)%
    rectangle%
    (current bounding box.north east);%
  \draw[blue!30, very thin]%
    (current bounding box.south west) grid (current bounding box.north east);%
\fi

\end{tikzpicture}%
%
    }
    \caption{Setup \& interpolation}
    \label{fig:queue-setup}
  \end{subfigure}
  \hfill
  \begin{subfigure}[c]{0.46\textwidth}
    \centering
    \resizebox{\linewidth}{!}{%
\IfFileExists{figures/queue_diagram/tikz_common.tex}{%
  \input{figures/queue_diagram/tikz_common.tex}%
}{%
  \input{tikz_common.tex}%
}%
\usetikzlibrary{decorations.pathreplacing}%
\begin{tikzpicture}[%
  >=Stealth,%
  every node/.style={font=\small},%
  scale=1.0%
]%

\ifcropToSection
  \clip (-0.35, \xtny + 0.5) rectangle (\Yrad-2.2, \Yrad);%
\fi

\pgfmathsetmacro{\seglen}{veclen(\ztnx-\ytx, \ztny-\yty)}%
\pgfmathsetmacro{\alphat}{\dtmin / \seglen}%
\pgfmathsetmacro{\wtx}{(1-\alphat)*\ytx + \alphat*\ztnx}%
\pgfmathsetmacro{\wty}{(1-\alphat)*\yty + \alphat*\ztny}%
\pgfmathsetmacro{\JntopY}{sqrt(\Yrad*\Yrad - \xtnx*\xtnx)}%
\pgfmathsetmacro{\JnrightY}{sqrt(\Yrad*\Yrad - \xtny*\xtny)}%
\pgfmathsetmacro{\JnarcStart}{asin(\xtny/\Yrad)}%
\pgfmathsetmacro{\JnarcEnd}{acos(\xtnx/\Yrad)}%

\ifremoveBackground\else
\fill[Yfill] (0,0) -- (\Yrad, 0) arc (0:90:\Yrad) -- cycle;%
\draw[Ycol, very thick] (\Yrad, 0) arc (0:90:\Yrad);%
\fi

\ifremoveBackground\else
\begin{scope}%
  \clip (0,0) rectangle (\Yrad, \Yrad);%
  \fill[Jncol, opacity=0.10]%
    (\xtnx, \xtny) -- (\JnrightY, \xtny)%
    arc (\JnarcStart:\JnarcEnd:\Yrad)%
    -- (\xtnx, \JntopY) -- cycle;%
\end{scope}%
\draw[Jncol, thick, densely dashed] (\xtnx, \xtny) -- (\xtnx, \JntopY);%
\draw[Jncol, thick, densely dashed] (\xtnx, \xtny) -- (\JnrightY, \xtny);%
\fi

\ifremoveBackground\else
\draw[->, thick] (-0.3, 0) -- (\Yrad + 0.3, 0);%
\node[font=\small, anchor=north east, inner sep=2pt] at (\Yrad + 0.3, -0.04) {$y^{(1)}$};%
\draw[->, thick] (0, -0.3) -- (0, \Yrad + 0.3);%
\node[font=\small, above, inner sep=2pt] at (0, \Yrad + 0.3) {$y^{(2)}$};%
\fi

\draw[ztcol, thick] (\ztx, \zty) -- (\ytx, \yty);%

\draw[movecol, thick, densely dashed] (\ztx, \zty) -- (\ztnx, \ztny);%
\draw[movecol, thick, decorate, decoration={brace, amplitude=4pt, raise=3pt}] (\ztx, \zty) -- (\ztnx, \ztny) coordinate[midway] (bracetip);%

\newcommand{\bracetipcoord}[5]{%
  \coordinate (#1) at ($($#2!0.5!#3$) !#4! #5:#3$);%
}%
\bracetipcoord{bracetip}{(\ztx, \zty)}{(\ztnx, \ztny)}{7pt}{90}%


\ifremoveBackground\else
  \fill[ballcol, opacity=0.08] (\ytx, \yty) circle (\dtmin);%
  \draw[ballcol, thick] (\ytx, \yty) circle (\dtmin);%
\fi

\draw[ballcol, thick] (\ytx, \yty) -- (\wtx, \wty);%

\draw[wtcol, thick] (\wtx, \wty) -- (\ztnx, \ztny);%

\ifremoveBackground\else
\draw[ytncolor, thick] (\ztnx, \ztny) -- (\ytnx, \ytny);%
\fi


\path (\ztx, \zty) -- (\ytx, \yty)%
  node[ztcol, font=\footnotesize, pos=0.5, below] {$q_t$};%



\node[movecol, font=\footnotesize, anchor=north] (movelabel)%
  at ([xshift=-5pt, yshift=-20pt]bracetip)%
  {$\|z_{t\!+\!1} {-} z_t\|$};%
\draw[movecol, thick, -, shorten >=1pt] (movelabel.north) -- (bracetip);%

\path (\ytx, \yty) -- (\wtx, \wty)%
  node[ballcol, font=\footnotesize, pos=0.4, sloped, inner sep=0pt, above=2pt] {$d_t^-$};%

\path (\wtx, \wty) -- (\ztnx, \ztny)%
  node[wtcol, font=\footnotesize, pos=0.5, sloped, above=2pt, inner sep=0pt] {$\|z_{t\!+\!1} {-} w_t\|$};%

\ifremoveBackground\else
\path (\ztnx, \ztny) -- (\ytnx, \ytny)%
  node[ytncolor, font=\footnotesize, pos=0.5, above] {$q_{t+1}$};%
\fi

\filldraw[ztcol] (\ztx, \zty) circle (2pt);%
\node[below right, ztcol, font=\normalsize, inner sep=2pt] at (\ztx, \zty) {$z_t$};%

\filldraw[ytcol] (\ytx, \yty) circle (2pt);%
\node[above right, ytcol, font=\normalsize, inner sep=2pt] at (\ytx, \yty) {$y_t$};%

\filldraw[wtcol] (\wtx, \wty) circle (2pt);%
\ifremoveBackground\else
\node[above, wtcol, font=\normalsize, inner sep=2pt] at (\wtx - 0.15, \wty+0.075) {$w_t$};%
\fi

\filldraw[ztncol] (\ztnx, \ztny) circle (2pt);%
\node[above left, ztncol, font=\normalsize, inner sep=2pt] at (\ztnx, \ztny) {$z_{t+1}$};%

\filldraw[xtncolor] (\xtnx, \xtny) circle (2pt);%
\node[below left, xtncolor, font=\normalsize, inner sep=2pt] at (\xtnx, \xtny) {$x_{t+1}$};%

\ifremoveBackground\else
\filldraw[ytncolor] (\ytnx, \ytny) circle (2pt);%
\node[above right, ytncolor, font=\normalsize, inner sep=2pt] at (\ytnx, \ytny) {$y_{t+1}$};%
\fi

\ifremoveBackground\else
\node[Ycol, font=\Large] at (\Ycalx, \Ycaly) {$\mathcal{Y}$};%
\node[Jncol, font=\normalsize] at (\Jtpx, \Jtpy) {$J_{t+1}$};%
\fi


\ifshowDebug
  \draw[red, very thick, dashed]%
    (current bounding box.south west)%
    rectangle%
    (current bounding box.north east);%
  \draw[blue!30, very thin]%
    (current bounding box.south west) grid (current bounding box.north east);%
\fi

\end{tikzpicture}%
%
    }
    \caption{Triangle inequality}
    \label{fig:queue-recursion}
  \end{subfigure}
  \caption{Geometric proof of the queue recursion
    $q_{t+1} \leq (q_t + \|z_{t+1} - z_t\| - d_t^{\min})^+$ for the case where $\norm{z_{t+1} - y_t} > d_t^{\min}$, building on the setup in Figure~\ref{fig:short-move}. 
    \textbf{(a)}~For any $z_{t+1} \in \Y$, the interpolation point
    $w_t$ a distance $d_t^{\min}$ away from $y_t$ on the line
    from $y_t$ toward $z_{t+1}$ must be feasible, providing a segment whose length bounds $q_{t+1}$.
    \textbf{(b)}~Isolating the relevant points, the triangle inequality gives the bound $\norm{z_{t+1} - w_t} + d_t^{\min} \leq q_t + \norm{z_{t+1} - z_t}$, which combines with the prior diagram to give the final recursion.}
  \label{fig:queue-geometry}
\end{figure}

The inventory-specific part of our analysis only depends on the path of the targets and inventory levels. No stationarity or independence of the demand realizations is needed until we convert the pathwise recursion into a high-probability clearing statement.

With the admissible projection norm $\norm{\cdot}$ from Section~\ref{sec:admissible_norms} fixed, let
\begin{equation}
\label{eq:hidden_projection_distance}
  q_t:=\distnorm(z_t,J_t)=\norm{z_t-y_t},
  \qquad
  d_t^{\min}:=\min_{i\in[n]}d_{t,i}.
\end{equation}

\begin{proposition}[Target-implementation queue]
\label{prop:target_implementation_gap_queue}
For the aligned implementation $y_t=\Projnorm{J_t}{z_t}$, under the inventory dynamics constraint \eqref{eq:inv_dynamics_constraint}, for every $t\in[T-1]$,
\begin{equation}
\label{eq:target_implementation_gap_queue}
  q_{t+1}
  \le
  \bigl[q_t+\norm{z_{t+1}-z_t}-d_t^{\min}\bigr]^+.
\end{equation}
\end{proposition}

Figures~\ref{fig:short-move}~and~\ref{fig:queue-geometry} visualize the geometric proof with specific realizations of the demand and choices for the hidden target from the base learner. Throughout the figure, distances are shown using the Euclidean norm, but the geometric intuition remains the same as long as the norm is admissible. The proof relies on the following lemma to ensure that if a point moves from the previous order level by a sufficiently small amount, then it remains feasible.

\begin{lemma}[A short move remains feasible after demand]
\label{lem:short_move}
Let $y,y'\in\R_+^n$ and $d\in\R_+^n$. If $\norm{y'-y}\le d^{\min}:=\min_i d_i$, then $y'\succeq (y-d)^+$.
\end{lemma}

This result is shown visually in Figure~\ref{fig:short-move}. For an order level $y_t$, any possible $x_{t+1}$ must be coordinatewise smaller than any point in the $\ell_\infty$ ball of radius $d_t^{\min}$ centered at $y_t$. This in turn means that any point that is coordinatewise at least as large as the corner of the $\ell_\infty$ ball must satisfy the inventory feasibility constraint \eqref{eq:inv_feasibility_constraint}, including the ball itself. This explains our choice of an \emph{admissible} norm: the ball around $y_t$ of radius $d_t^{\min}$ induced by $\norm{\cdot}$ must be contained in the $\ell_\infty$ ball, and hence also satisfies the inventory feasibility constraint. Any point inside the intersection of this ball and $\Y$ must therefore be in the next feasible order-level set $J_{t+1}$. 

When the next point $z_{t+1}$ is chosen, there are two options. In the first case, $z_{t+1}$ lies within distance $d_t^{\min}$ of the implemented point $y_t$, but the short move lemma makes $z_{t+1}$ feasible at the next round so $y_{t+1} = z_{t+1}$ and the queue clears with $q_{t+1} = 0$. Otherwise, for the second case where $z_{t+1}$ lies outside the $d_t^{\min}$ ball around $y_t$, we turn to Figure~\ref{fig:queue-setup}. By the convexity of $\Y$, we interpolate between $y_t$ and $z_{t+1}$ and choose the point $w_t$ that is distance $d_t^{\min}$ away from $y_t$ to make it feasible. Crucially, the \emph{aligned projection certificate} (Lemma~\ref{lem:aligned_projection_certificate}), which states that projecting in $\norm{\cdot}$ minimizes the distance to any feasible witness, gives $q_{t+1} \leq \norm{z_{t+1} - w_t}$. Finally, Figure~\ref{fig:queue-recursion} uses the triangle inequality to bound this distance. Combining with the first case gives the queue recursion. 


To obtain high-probability regret, we use the non-degeneracy demand condition introduced by \citet{hihat2023online}.

\begin{assumption}[Uniformly probably positive demand]
\label{ass:uppd}
There exist $\mu\in(0,1]$ and $\rho>0$ such that, for every round $t$, almost surely
\[
  \mathbb P(d_t^{\min}\ge\rho\mid\F_{t-1})\ge\mu.
\]
\end{assumption}

UPPD says that, no matter the past, there is conditional probability at least $\mu$ of a common demand shock of size $\rho$. So over any sufficiently long window of rounds, at least one such shock occurs with high probability. To exploit this condition for queueing control, our analysis partitions $[T]$ into queue-clearing windows where the shocks can clear the accumulated queue, reducing the cumulative gap to a bound on the target's movement. If the base learner moves sufficiently slowly, by at most $\rho$ over each such window, the queue cannot persist across windows. Fixing a failure probability $\delta\in(0,1)$, we write $B_\delta:=\lceil\mu^{-1}\log(T/\delta)\rceil$ for the corresponding clearing-window length under UPPD.

Before moving to the regret reduction, we formally specify the conditions that the base learner must satisfy. We treat the base learner as a black box specified entirely by its inputs and outputs, with an additional requirement on the movement of its outputs. At each round the base learner receives the linearized loss $f_t(z)=\ip{g_t}{z}$ on $\Y$, where the gradient satisfies $\dnorm{g_t}\le G_*$, and produces a hidden target $z_{t+1}\in\Y$. The reduction depends on the target sequence only through the following two conditions, both measured in the norm $\norm{\cdot}$.

\begin{assumption}[Base learner conditions]
\label{ass:base_learner}
The base learner produces targets $(z_t)_{t=1}^T\subset\Y$ such that:
\begin{enumerate}[label=(\roman*)]
  \item \textit{Linearized regret bound.} There is a function $\mathcal{R}_T(u_{1:T})$ for which, for every comparator sequence $u_{1:T}\in\Y^T$,
  \[
    \baseregret(u_{1:T})\;\le\;\mathcal{R}_T(u_{1:T}).
  \]
  \item \textit{Windowed movement bound.} For every $t$ with $1\le t\le T-B_\delta$,
  \begin{equation}
  \label{eq:learner_window_motion_uppd}
    \sum_{r=t}^{t+B_\delta-1}\norm{z_{r+1}-z_r}\le \rho.
  \end{equation}
\end{enumerate}
\end{assumption}

The two parts of Assumption~\ref{ass:base_learner} are imposed on the target sequence, not on the internal mechanism of the base learner. Any specific algorithm, such as projected OGD, smoothed OGD, or mirror descent, defines a valid base learner once parts (i) and (ii) are verified in the projection norm. Section~\ref{sec:instantiations} carries the verification for these instantiations.


\begin{theorem}[Hidden-target reduction]
\label{thm:uppd_reduction}
Under Assumption~\ref{ass:convex_and_bounded} and UPPD (Assumption~\ref{ass:uppd}), suppose the base learner satisfies Assumption~\ref{ass:base_learner}(ii) and initializes with a feasible target $z_1 \in J_1$. Then, with probability at least $1-\delta$,
\begin{equation}
\label{eq:tc-uppd}
  \tarimpgap\le B_\delta\tarmvment,
  \qquad
  R_T(u_{1:T})\le \baseregret(u_{1:T})+G_* B_\delta\tarmvment
\end{equation}
for every comparator sequence $u_{1:T}\in\Y^T$. In particular, if Assumption~\ref{ass:base_learner}(i) additionally holds with regret function $\mathcal{R}$, then
\[
  R_T(u_{1:T})\le\mathcal{R}_T(u_{1:T})+G_*B_\delta\tarmvment
\]
\end{theorem}

Theorem~\ref{thm:uppd_reduction} is a reduction from OIO to OCO with switching cost. The base learner must control two ordinary quantities, linearized regret and cumulative movement, while the inventory layer contributes only the clearing-window factor $B_\delta$. This modularity is the advantage of the norm-aligned queue construction: the inventory proof is agnostic to how $z_t$ was generated and to the specific admissible norm chosen, and the base learner never reasons about products or sell-out events. The lack of extra multiplicative factors in the queue recursion \eqref{eq:target_implementation_gap_queue} is essential for our construction. Using norm conversion would yield a recursion of the form $q_{t+1}\le C[q_t+a_t-d_t^{\min}]^+$ with $C>1$, but this does not support the same queue clearing argument. The constant multiplies the residual backlog before the next service opportunity, whereas our approach leverages norm alignment to ensure the queue clears regularly.

\paragraph{Relation to sell-out reductions.}
The linear-capacity hidden-target analysis of \citet{ichikawa2026nonstationary} tracks productwise sell-out quantities and controls switching in $\ell_1$. In contrast, our reduction uses the state variable $\distnorm(z_t,J_t)$, which makes it applicable to arbitrary convex feasible sets $\Y$. This change also eliminates the explicit $n^{1/4}$ factor arising from the $\ell_1$ movement control in \citet{ichikawa2026nonstationary}, and replaces $P_{T,1}$ in dynamic regret by the smaller Euclidean path variation $P_{T,2}$. This comparison should not be read as saying that the wrong norm merely costs powers of $n$. Rather, a productwise or misaligned queue needs a conservation identity that is absent on curved convex capacity sets. Appendix~\ref{app:geometry} provides a two-dimensional example to illustrate this obstruction. The tradeoff is the local movement condition of Assumption~\ref{ass:base_learner}(ii) on the base learner, which is not required in \citet{ichikawa2026nonstationary}. For the OGD and SOGD base learners used in Section~\ref{sec:instantiations}, however, this condition is automatically satisfied once the stepsize or switching parameter is chosen at the clearing-window scale. Thus, for these algorithms, the additional requirement needed to handle general geometry imposes no further loss in the final regret guarantees.

\section{Optimal and adaptive guarantees}
\label{sec:instantiations}

We now instantiate the reduction. Except in Section~\ref{sec:beyond_euclidean}, the instantiations in this section use gradient-descent base learners with the Euclidean norm \(\norm{\cdot}_2\) as the projection norm. Section~\ref{sec:beyond_euclidean} then turns to mirror-descent base learners and non-Euclidean projection geometries.
 We also write $D_2$ and $G_2$ for the capacity-set diameter and dual gradient bound respectively from Assumption~\ref{ass:convex_and_bounded} specialized to the $\ell_2$ norm. Taken projected OGD as the base learner, the Hidden-Target OGD is given in Algorithm~\ref{alg:ht_ogd}. Its update is
\begin{equation}
\label{eq:ht_ogd_update}
  y_t=\Proj_{J_t}^{\norm{\cdot}_2}(z_t),
  \qquad
  z_{t+1}=\Proj_{\Y}^{\norm{\cdot}_2}(z_t-\eta_t g_t).
\end{equation}

\begin{algorithm}[H]
\caption{Hidden-Target OGD}
\label{alg:ht_ogd}
\KwInput{Capacity set $\Y$, initial target $z_1\in J_1$, stepsizes $(\eta_t)$}
\For{$t=1,2,\ldots,T$}{
Observe inventory state $x_t$ and set $J_t=\Y\cap\{y\in\R^n\mid y\succeq x_t\}$\;
Implement $y_t=\Proj_{J_t}^{\norm{\cdot}_2}(z_t)$\;
Observe $g_t\in\partial\ell_t(y_t)$ and the next state $x_{t+1}$\;
Update $z_{t+1}=\Proj_{\Y}^{\norm{\cdot}_2}(z_t-\eta_t g_t)$\;
}
\end{algorithm}

\subsection{Convex losses and the optimal dependence on \texorpdfstring{$\mu$}{mu}}

\begin{theorem}[Hidden-Target OGD under UPPD]
\label{thm:ht_ogd_convex_uppd}
Assume Assumption~\ref{ass:convex_and_bounded} and UPPD (Assumption~\ref{ass:uppd}). Fix $\delta\in(0,1)$ and set $B_\delta=\lceil\mu^{-1}\log(T/\delta)\rceil$. Run HT-OGD (Algorithm~\ref{alg:ht_ogd}) from $z_1\in J_1$ with stepsize $\eta_t=\gamma D_2/(G_2\sqrt t)$, where $0<\gamma\le \rho/(2D_2\sqrt{B_\delta})$. Then, with probability at least $1-\delta$, simultaneously for all $u\in\Y$,
\[
  R_T(u)
  \le
  D_2 G_2\Bigl(\frac{1}{2\gamma}+\gamma+2\gamma B_\delta\Bigr)\sqrt T.
\]
In particular, choosing $\gamma = \rho/(2D_2\sqrt{B_\delta})$ gives
\[
  R_T(u)
  \le D_2 G_2\left(\frac{D_2}{\rho} + \frac{3\rho}{2D_2}\right)\sqrt{B_\delta T}
  =\widetilde O\!\left(D_2 G_2\sqrt{T/\mu}\right).
\]
Moreover, taking $\delta=1/T$ and $\gamma=\rho/(2D_2\sqrt{B_{1/T}})$ gives
\[
  \mathbb E[R_T(u)]
  \le
  D_2 G_2\Bigl(\frac{D_2}{\rho}+\frac{3\rho}{2D_2}\Bigr)\sqrt{B_{1/T}T}+D_2 G_2.
\]
\end{theorem}

The theorem balances two costs. OGD contributes $O(D_2 G_2\sqrt T/\gamma)$ target regret, while the queue reduction charges $O(D_2 G_2\gamma B_\delta\sqrt T)$ for target movement. The clearing-scale choice $\gamma = \Theta(B_\delta^{-1/2})$ gives the optimal $\widetilde O(D_2 G_2\sqrt{T/\mu})$ rate.

\begin{theorem}[A $\sqrt{T/\mu}$ lower bound under UPPD]
\label{thm:convex_uppd_lower_bound}
There exists a universal constant $c>0$ such that, for every $\mu\in(0,1]$ and every $T\ge\mu^{-1}$, there is a two-product OIO instance with feasible set
\begin{equation}
\label{eq:convex_uppd_lower_capacity_set}
  \Y_L:=\{y\in\R_+^2\mid y_1+y_2=1\},
\end{equation}
gradient bound $G_2=1$, and UPPD parameter $\rho=1$ for which every possibly randomized inventory algorithm satisfies
\[
  \mathbb E\!\left[
    \sum_{t=1}^T\ell_t(y_t)-\min_{u\in\Y_L}\sum_{t=1}^T\ell_t(u)
  \right]
  \ge
  c\sqrt{\frac{T}{\mu}}.
\]
The instance uses only the two linear losses $\ell_t(y)=y_1$ and $\ell_t(y)=y_2$.
\end{theorem}

\textbf{Proof idea.} The construction uses Bernoulli common-demand shocks. After a no-demand round, inventory feasibility and the equality constraint $y_1 + y_2 = 1$ pin the next order-up-to level to the previous one, so the learner plays the same action between any two shocks. This decomposes the period into random blocks of constant action. Each block then receives an independent random loss direction. The learner cannot exploit the direction within the block, while the best static comparator chooses the endpoint favored by the aggregate signed block length. Since the squared block lengths are typically of order $T/\mu$, this aggregate has an imbalance of order $\sqrt{T/\mu}$.

This lower bound is different from the $L_{\max}$-based lower bound of \citet{ichikawa2026nonstationary}. Their construction uses deterministic cycles with long zero-demand stretches, so it is not a UPPD lower bound with a positive $\mu$. Our instance itself satisfies UPPD with parameter $\mu$, and the randomness of the clearing times is exactly what produces the $1/\sqrt{\mu}$ factor. Together, Theorems~\ref{thm:ht_ogd_convex_uppd} and~\ref{thm:convex_uppd_lower_bound} characterize the minimax complexity of general-convex OIO under the UPPD condition.

\subsection{Strongly convex losses}

The same algorithm gives a fast rate when the losses have curvature. The argument follows the same reasoning as the convex case, with one structural difference: strongly convex losses require a different interface condition on the base learner. With strongly convex loss, the loss difference can be controlled not only by the linearized regret $\sum_t \ip{g_t}{z_t - u}$ of Assumption~\ref{ass:base_learner}(i) but also by a quadratic term $\norm{z_t - u}^2$. We therefore replace this assumption with a curvature-augmented analog.

The base learner satisfies the \emph{strongly convex linearized regret bound} with curvature parameter $\alpha$ if there exists a function $\mathcal{R}_T^{\textnormal{sc}}(u_{1:T}; \alpha)$ such that for every comparator sequence $u_{1:T} \in \Y^T$,
\begin{equation}
\label{eq:sc_regret_condition}
  \sum_{t=1}^T \Bigl(\ip{g_t}{z_t - u_t} - \tfrac{\alpha}{4}\norm{z_t - u_t}^2\Bigr) \;\le\; \mathcal{R}_T^{\textnormal{sc}}(u_{1:T}; \alpha).
\end{equation}
This generalizes Assumption~\ref{ass:base_learner}(i) to the curvature-aware setting, and recovers the original assumption when no strong convexity is assumed, i.e. $\alpha = 0$. The loss's extra curvature allows the base learner to control a smaller quantity, and therefore attain lower overall regret bounds. This is formalized in the following generalization of the hidden target reduction in Theorem~\ref{thm:uppd_reduction}:

\begin{corollary}[Strongly convex hidden-target reduction]
\label{cor:strong_convex_reduction}
Under Assumption~\ref{ass:convex_and_bounded} and UPPD (Assumption~\ref{ass:uppd}), suppose each $\ell_t$ is $\alpha$-strongly convex on $\Y$ with respect to $\norm{\cdot}$. If the base learner produces targets $(z_t)_{t=1}^T \subset \Y$ initialized at $z_1 \in J_1$ which satisfy Assumption~\ref{ass:base_learner}(ii) and the strongly convex linearized regret bound~\eqref{eq:sc_regret_condition}, then with probability at least $1-\delta$, for every comparator sequence $u_{1:T} \in \Y^T$,
\begin{equation}
\label{eq:sc_reduction_bound}
  R_T(u_{1:T}) \;\le\; \mathcal{R}_T^{\textnormal{sc}}(u_{1:T}; \alpha) + \Bigl(G_* + \frac{\alpha D}{2}\Bigr) B_\delta\, \tarmvment.
\end{equation}
\end{corollary}

This corollary provides an interface to use HT-OGD with a curvature-tuned stepsize to achieve polylogarithmic regret in the strongly convex setting.

\begin{theorem}[Strongly convex losses]
\label{thm:ht_ogd_strong_uppd}
Assume Assumption~\ref{ass:convex_and_bounded} and UPPD (Assumption~\ref{ass:uppd}). Take the projection norm to be Euclidean, $\norm{\cdot}=\norm{\cdot}_2$. Suppose each $\ell_t$ is $\alpha$-strongly convex on $\Y$. Fix $\delta\in(0,1)$ and define
\[
  B_\delta:=\Bigl\lceil\frac{1}{\mu}\log\!\Bigl(\frac{T}{\delta}\Bigr)\Bigr\rceil,
  \qquad
  s_\delta:=\max\!\left\{1,\left\lceil\frac{B_\delta}{e^{\alpha\rho/(2G_2)}-1}\right\rceil\right\}.
\]
Run HT-OGD (Algorithm~\ref{alg:ht_ogd}) from $z_1\in J_1$ with $\eta_t=2/(\alpha(t+s_\delta))$. Then, with probability at least $1-\delta$, for all $u\in\Y$,
\[
  R_T(u)=O\!\left(\frac{\log(T/\delta)\log T}{\mu}\right),
\]
where the hidden constant depends only on $\alpha,\rho,D_2,$ and $G_2$.
\end{theorem}

Thus strong convexity converts the ordinary $\sqrt T$ target regret into polylogarithmic target regret, and the inventory layer contributes only $O(B_\delta\log T)$. This gives the first polylogarithmic-in-$T$ adversarial-OIO guarantee for strongly convex losses on general convex capacity sets, resolving a major open direction raised by \citet{hihat2023online}. The result also complements the logarithmic rates known under more stochastic inventory assumptions~\citep{lyu_minibatch_2024}.

\subsection{Dynamic regret}

For a comparator sequence $u_{1:T}$, define
\begin{equation}
\label{eq:comparator_path_length}
  P_{T,1}:=\sum_{t=2}^T\norm{u_t-u_{t-1}}_1,
  \qquad
  P_{T,2}:=\sum_{t=2}^T\norm{u_t-u_{t-1}}_2.
\end{equation}
The Euclidean queue naturally controls $P_{T,2}$.

\begin{theorem}[Known path variation]
\label{thm:ht_ogd_dynamic_uppd}
Assume Assumption~\ref{ass:convex_and_bounded} and UPPD (Assumption~\ref{ass:uppd}). Fix $\delta\in(0,1)$ and set $B_\delta:=\lceil\mu^{-1}\log(T/\delta)\rceil$. For a comparator sequence $u_{1:T}\in\Y^T$ with known path variation $P_{T,2}$, let
\[
  \eta^*:=\frac{1}{G_2}\sqrt{\frac{D_2(D_2+2P_{T,2})}{(2B_\delta+1)T}}.
\]
If $\eta^*\le \rho/(B_\delta G_2)$ and HT-OGD (Algorithm~\ref{alg:ht_ogd}) is run with constant stepsize $\eta_t\equiv\eta^*$, then with probability at least $1-\delta$,
\[
  R_T(u_{1:T})
  \le
  G_2\sqrt{(2B_\delta+1)D_2(D_2+2P_{T,2})T}
  =
  \widetilde O\!\left(G_2\sqrt{\frac{D_2(D_2+P_{T,2})T}{\mu}}\right).
\]
\end{theorem}

Note that the best known dynamic regret guarantee for OIO was achieved by \citet{ichikawa2026nonstationary} using the same algorithm, with path variation measured by $\ell_1$ rather than $\ell_2$, an additional $n^{1/4}$ factor, and the restriction on a linear capacity constraint. Compared with their $P_{T,1}$-based guarantee, our $P_{T,2}$-based regret rate can be sharper by a factor as large as $\sqrt n$ for dense comparator movement.

When $P_{T,2}$ is unknown, we use the switching-aware SOGD base learner of \citet{zhang2022soco}, following the hidden-target architecture of \citet{ichikawa2026nonstationary}. SOGD combines a grid of OGD experts with a smoothed meta-learner and satisfies a regret-plus-switching-cost guarantee. Appendix~\ref{sec:sogd_proof} gives the base learner details and proves that its movement is slow enough for the queue reduction.

For $\lambda\ge1$, define
\begin{equation}
\label{eq:sogd_native_regret}
\begin{split}
  \mathcal R_T^{\mathrm{SOGD}}(u_{1:T};\lambda)
  := {}& 2G_2 D_2\sqrt{(1+\lambda)T\Bigl(1+\frac{2P_{T,2}}{D_2}\Bigr)} \\
  &+120G_2 D_2\sqrt{\lambda T\Bigl(1+\frac{2P_{T,2}}{D_2}\Bigr)\log T} .
\end{split}
\end{equation}

\begin{theorem}[Unknown path variation]
\label{thm:ht_sogd_regret_uppd}
Assume Assumption~\ref{ass:convex_and_bounded} and UPPD (Assumption~\ref{ass:uppd}). Fix $\delta\in(0,1)$ and set
\[
  B_\delta:=\Bigl\lceil\frac{1}{\mu}\log\!\Bigl(\frac{T}{\delta}\Bigr)\Bigr\rceil,
  \qquad
  \lambda_\delta:=\Bigl\lceil B_\delta\max\!\Bigl\{1,\frac{2D_2}{\rho}\Bigr\}\Bigr\rceil .
\]
Assume $T\ge\max\{32\lambda_\delta\log T,e\}$. Run the hidden-target projection layer with the SOGD base learner and switching parameter $\lambda_\delta$. Then, with probability at least $1-\delta$, for every comparator sequence $u_{1:T}\in\Y^T$,
\begin{equation}
\label{eq:sogd_regret_bound_uppd}
  R_T(u_{1:T})\le \mathcal R_T^{\mathrm{SOGD}}(u_{1:T};\lambda_\delta).
\end{equation}
Since $\lambda_\delta\ge 1$ and $T\ge e$, this further implies the explicit bound
\begin{align}
\label{eq:sogd_regret_bound_uppd_explicit}
  R_T(u_{1:T})
  &\le
  500\,G_2D_2
  \sqrt{\Bigl(1+\frac{P_{T,2}}{D_2}\Bigr)
  \max\!\Bigl\{1,\frac{D_2}{\rho}\Bigr\}}
  \sqrt{\frac{T}{\mu}\log\!\Bigl(\frac{T}{\delta}\Bigr)\log T}.
\end{align}
\end{theorem}

\subsection{Geometry-sensitive refinement via mirror descent}
\label{sec:beyond_euclidean}

The hidden-target reduction in Theorem~\ref{thm:uppd_reduction} is not intrinsically Euclidean. It asks only that the target sequence have OCO regret and controlled local movement in the same norm that measures the target-implementation gap. Online mirror descent (OMD) supplies exactly this interface whenever its regularizer is strongly convex with respect to the chosen projection norm. Thus we can provide a single theorem specializing to OGD, \(p\)-norm OMD, entropic OMD, and other regularizer choices, with the regret constant adapting to the geometry of \(\Y\) and to the corresponding dual gradient norm.

Mirror descent has received much less attention than OGD in inventory learning. To our knowledge, \citet{guo2026online} first introduced OMD to stochastic multiproduct inventory learning, using cyclic updates under a linear warehouse constraint. Our reduction permits a broader use of OMD: it applies in the adversarial OIO model, allows arbitrary bounded convex capacity sets once a suitable regularizer is chosen, removes the need for cyclic updates, and gives guarantees for any strongly convex regularizer rather than only negative entropy.

Let $\norm{\cdot}$ be an admissible norm with dual $\dnorm{\cdot}$. Let $R$ be a proper lower-semicontinuous convex regularizer on $\Y$, and let $\mathcal D\subseteq\Y$ denote the set on which $R$ is finite and differentiable. We assume that $R$ is $\sigma$-strongly convex with respect to $\norm{\cdot}$ in the Bregman sense
\begin{equation*}
D_R(u,v)\ge \frac{\sigma}{2}\norm{u-v}^2\qquad\text{for every }u\in\Y\text{ and }v\in\mathcal D,
\end{equation*}
and that the mirror-descent update below is initialized at $z_1\in J_1\cap\mathcal D$ and returns points in $\mathcal D$. The OMD base learner produces hidden targets as
\begin{equation}
\label{eq:md_update}
  z_{t+1} = \arg\min_{z\in\Y}\bigl\{\eta_t\ip{g_t}{z} + D_R(z,z_t)\bigr\},
\end{equation}
where $D_R(u,v):=R(u)-R(v)-\ip{\nabla R(v)}{u-v}$ is the Bregman divergence of $R$ when $v\in\mathcal D$. Taking OMD as the base learner gives the HT-OMD algorithm in Algorithm~\ref{alg:ht_omd}.

\begin{algorithm}[htbp]
\caption{Hidden-Target OMD}
\label{alg:ht_omd}
\KwInput{Bounded $\Y\subseteq\R_+^n$, initial target $z_1\in J_1\cap\mathcal D$, constant stepsize $\eta>0$}
\For{$t=1,2,\dots,T$}{
  Observe inventory state $x_t$ and define $J_t=\Y\cap\{y\in\R^n\mid y\succeq x_t\}$\;
  Implement $y_t=\Projnorm{J_t}{z_t}$\;
  Observe $g_t\in\partial\ell_t(y_t)$ and the next state $x_{t+1}$\;
  Update $z_{t+1}=\arg\min_{z\in\Y}\bigl\{\eta\ip{g_t}{z}+D_R(z,z_t)\bigr\}$\;
}
\end{algorithm}

With an appropriate constant stepsize, the OMD base learner satisfies Assumption~\ref{ass:base_learner}. Combining that verification with Theorem~\ref{thm:uppd_reduction} and optimizing the stepsize yields the following regret guarantee.

\begin{theorem}[Hidden-Target OMD under UPPD]
\label{thm:md_oio_regret}
Assume Assumption~\ref{ass:convex_and_bounded} and UPPD (Assumption~\ref{ass:uppd}). Assume the OMD update returns points in $\mathcal D$. Fix $\delta\in(0,1)$, set $B_\delta:=\lceil\mu^{-1}\log(T/\delta)\rceil$, and write $G_*:=\sup_t\dnorm{g_t}$. Assume $\mathcal B_R:=\sup_{u\in\Y}D_R(u,z_1)<\infty$ for an initialization $z_1\in J_1\cap\mathcal D$ and $G_*>0$. Run HT-OMD (Algorithm~\ref{alg:ht_omd}) at the constant stepsize
\begin{equation}\label{eq:MD-lr}
\eta = \min\left\{\sqrt{\frac{\mathcal B_R\sigma}{G_*^2(1/2+B_\delta)T}},\;\frac{\rho\sigma}{B_\delta G_*}\right\}.
\end{equation}
Then with probability at least $1-\delta$, for every $u\in\Y$,
\begin{equation}
\label{eq:md_master_bound_constrained}
 R_T(u)\le \frac{\mathcal B_R}{\eta}+\frac{\eta T G_*^2}{\sigma}\Bigl(\tfrac12+B_\delta\Bigr).
\end{equation}
In particular, when \(T\) is sufficiently large such that the first term in the \(\min\) operator in \eqref{eq:MD-lr} is selected,
then 
\begin{equation}
\label{eq:md_master_rate}
R_T(u) \leq 2 G_*\sqrt{\frac{\mathcal{B}_RT(1/2+B_\delta)}{\sigma}} = \widetilde O\!\left(G_*\sqrt{\frac{\mathcal{B}_RT}{\sigma\mu}}\right).
\end{equation}
\end{theorem}

Several familiar algorithms are recovered as instantiations of this constrained-stepsize bound \eqref{eq:md_master_rate}.

\paragraph{Specializations.}
With $R(x)=\tfrac{1}{2}\norm{x}_2^2$ paired with an $\ell_2$ norm giving $\sigma=1$ and $\mathcal{B}_R=\tfrac{1}{2}D_2^2$, the OMD update collapses to projected OGD. Then its regret bound \eqref{eq:md_master_rate} reproduces Theorem~\ref{thm:ht_ogd_convex_uppd}'s rate $\widetilde O(D_2 G_2\sqrt{T/\mu})$ with a constant stepsize rather than a decaying one. For $p\in(1,2]$, we may choose $R(x)=\tfrac{1}{2}\norm{x}_p^2$ and an $\ell_p$ norm giving $\sigma=p-1$ and $\mathcal{B}_R\le D_p^2/2$. Then one obtains the $p$-norm OMD algorithm achieving the rate $\widetilde O(D_p G_{p^*}\sqrt{T/((p-1)\mu)})$, where $p^*:=p/(p-1)$.

\paragraph{Entropic OMD on bounded $\R_+^n$.}
A qualitatively different instantiation uses the negative-entropy regularizer $R(x):=\sum_{i=1}^n x_i\log x_i$ \citep{guo2026online}. Choosing the projection norm as $\ell_1$ and writing $M:=\sup_{x\in\Y}\norm{x}_1$ for the outer $\ell_1$ radius of $\Y$, the regularizer $R$ is $(1/M)$-strongly convex with respect to $\ell_1$ on the restricted positive domain $\{x\in\R_+^n:\norm{x}_1\le M\}$, with the usual extended-value interpretation at the boundary. For the linear capacity constraint $\Y=\{x:\norm{x}_1\le C\}$, we naturally have $M=C$, and the same analysis applies on any general convex set $\Y$ that contains the strictly positive initialization used below.

\begin{theorem}[Entropic OMD under UPPD]
\label{thm:ht_entropic_md}Assume Assumption~\ref{ass:convex_and_bounded} and UPPD (Assumption~\ref{ass:uppd}). 
Use $\ell_1$ as the projection norm and fix $\delta\in(0,1)$. Write $B_\delta:=\lceil\mu^{-1}\log(T/\delta)\rceil$, $M:=\sup_{x\in\Y}\norm{x}_1$, and $G_\infty:=\sup_t\norm{g_t}_\infty$. Assume $M,G_\infty>0$ and that the uniform vector $z_1:=(M/n,\dots,M/n)\in \mathcal{Y}\cap\R_{++}^n$. Run HT-OMD (Algorithm~\ref{alg:ht_omd}) with negative-entropy regularizer $R(x):=\sum_{i=1}^n x_i\log x_i$ from the initialization $z_1$ (the regularizer keeps all iterates in $\mathcal{D}$ with $\mathcal{D}=\Y\cap\R_{++}^n$, as verified in Appendix~\ref{app:entropic_md}). Choose
\begin{equation}\label{eq:MD-lr2}
\eta=\min\left\{\frac{\sqrt{1+\log n}}{G_\infty\sqrt{(1/2+B_\delta)T}},\;\frac{\rho}{M B_\delta G_\infty}\right\}.
\end{equation}
Then with probability at least $1-\delta$, for every $u\in\Y$,
\begin{equation}
\label{eq:entropic_constrained_bound}
R_T(u)\le \frac{M(1+\log n)}{\eta}+\eta M G_\infty^2T\Bigl(\tfrac12+B_\delta\Bigr).
\end{equation}
In particular, when \(T\) is sufficiently large such that the first term in the \(\min\) operator in \eqref{eq:MD-lr2} is selected, then
\[
 R_T(u) \le 2 M G_\infty\sqrt{T(1+\log n)(1/2+B_\delta)} = \widetilde O\!\left(MG_\infty\sqrt{\frac{T\log n}{\mu}}\right).
\]
\end{theorem}

\paragraph{When does OMD outperform OGD?}
The general HT-OMD rate $\widetilde O\!\bigl(\sqrt{\mathcal{B}_R/\sigma}\,G_*\sqrt{T/\mu}\bigr)$ should be read as a geometry-sensitive refinement of the $\ell_2$-based HT-OGD bound in Theorem~\ref{thm:ht_ogd_convex_uppd}. The $\sqrt{T/\mu}$ dependence is the same for all regularizers; the algorithmic constant is the product of an \emph{effective Bregman diameter} $\sqrt{\mathcal{B}_R/\sigma}$ and the dual gradient bound $G_*$. For $\ell_2$-OGD this product is $D_2G_2/\sqrt 2$, while for entropic OMD it is $M\sqrt{1+\log n}\,G_\infty$. The reason that entropic OMD can improve over $\ell_2$-OGD is dual-norm sensitivity: since $G_\infty\le G_2\le\sqrt n\,G_\infty$, replacing an $\ell_2$ gradient bound by an $\ell_\infty$ bound can save as much as $\sqrt n$ when gradients are spread across many coordinates. 

For the linear warehouse constraint $\Y=\{x\in\R_+^n:\norm{x}_1\le C\}$, the effective Bregman diameters are comparable: $M=C$, $D_2=C\sqrt 2$, and entropic OMD pays an additional $\sqrt{\log n}$ KL-diameter factor --- this overhead is small relative to the possible dual-norm gain. If gradients are coordinate-dense, $G_2\approx\sqrt n\,G_\infty$, entropic OMD improves the $\ell_2$-OGD rate by a factor on the order of $\sqrt{n/\log n}$. If gradients are concentrated on a few coordinates, $G_\infty\approx G_2$, $\ell_2$-OGD is sharper by at most the mild $\sqrt{\log n}$ factor in this geometry. Hence, entropic OMD can be viewed as more suitable than $\ell_2$-OGD for the linear warehouse constraint, improving on \cite{ichikawa2026nonstationary}, since negative entropy is the natural regularizer on the simplex. The improvement extends the insights of \cite{guo2026online} to the adversarial OIO setting.

For general convex capacity sets, the same principle applies with the appropriate regularizer: HT-OMD lets the bound follow the geometry of $\Y$ and the dual norm in which gradients are small, rather than forcing every instance into Euclidean constants.

\section{Experiments}
\label{sec:experiments}

\begin{figure}[t]
    \centering
    \begin{subfigure}[b]{0.32\linewidth}
        \centering
        \FigureGraphic[width=\linewidth]{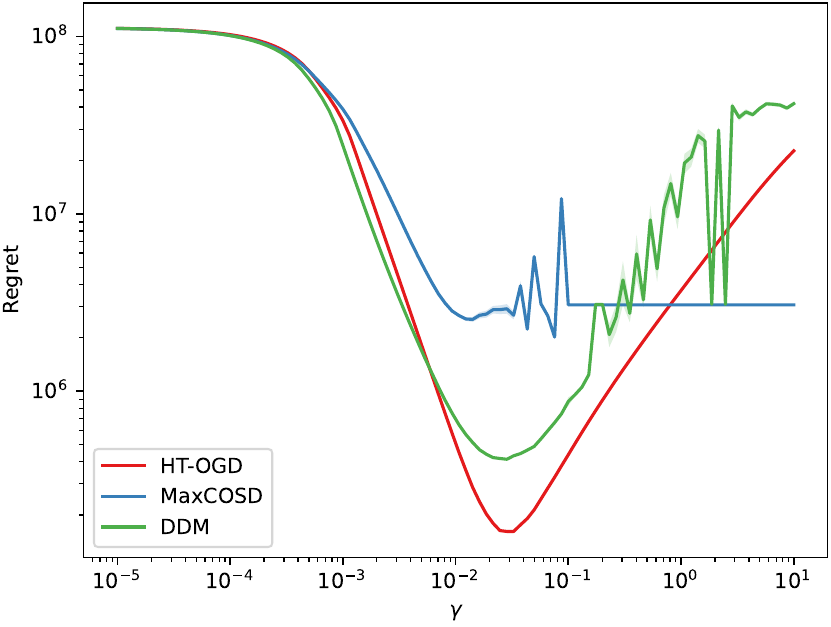}
        \caption{Learning-rate sweep.}
        \label{fig:l1_loss_lr_curve}
    \end{subfigure}
    \hfill
    \begin{subfigure}[b]{0.32\linewidth}
        \centering
        \FigureGraphic[width=\linewidth]{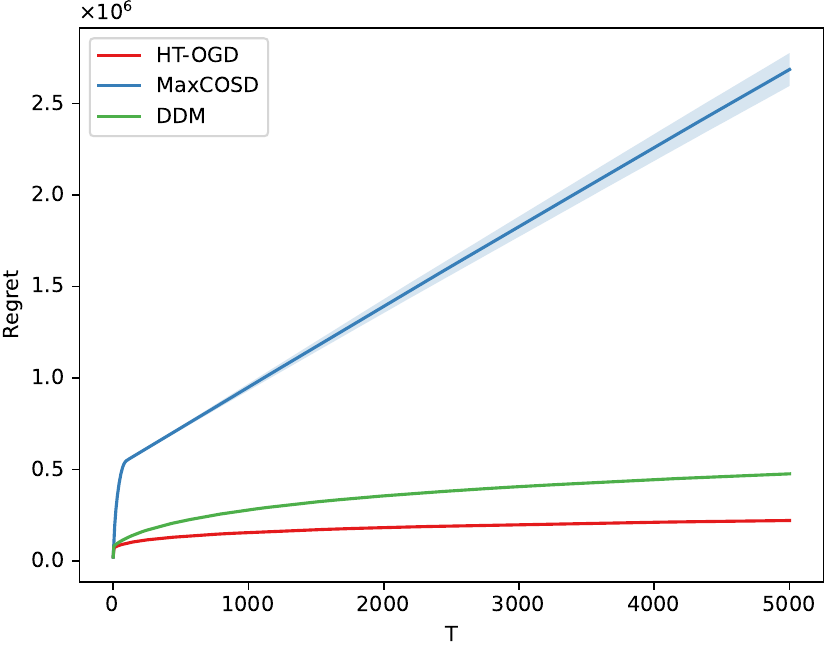}
        \caption{Regret over time.}
        \label{fig:l1_loss_t_curve}
    \end{subfigure}
    \hfill
    \begin{subfigure}[b]{0.32\linewidth}
        \centering
        \FigureGraphic[width=\linewidth]{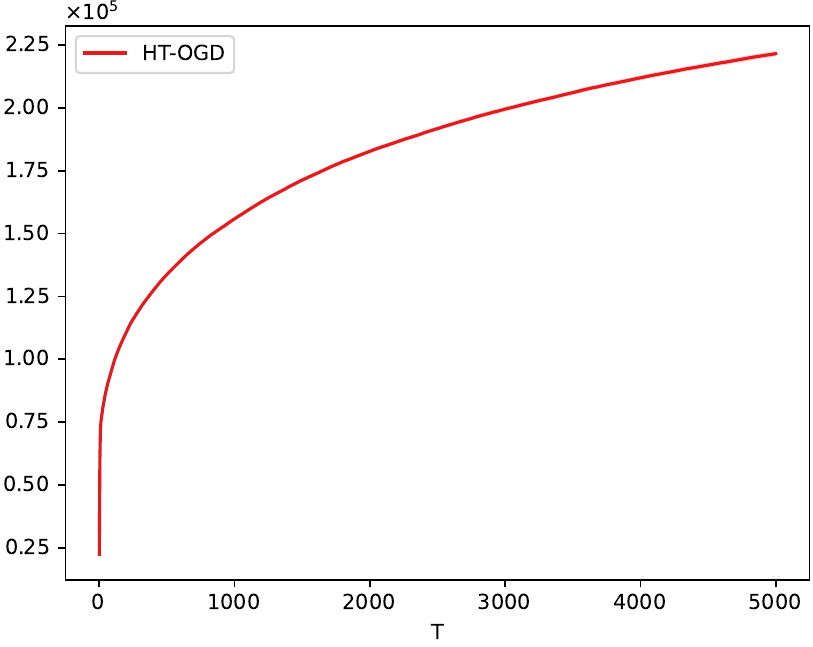}
        \caption{HT-OGD only.}
        \label{fig:l1_loss_t_ht_ogd_curve}
    \end{subfigure}
    \caption{Static regret on the linear warehouse capacity $\Y_C=\{y\in\R_+^n:\norm{y}_1\le C\}$ (Setting 3 of \citet{hihat2023online}). Panel (a) reports regret at $T=5\times10^3$ over logarithmically spaced learning-rate multipliers. Panel (b) uses the best multiplier for each algorithm and plots regret over time. Panel (c) isolates HT-OGD on its own axes since its low-magnitude trajectory can be hard to read in the previous panel.}
  \label{fig:l1_loss}
\end{figure}

The experiments compare the three algorithmic principles available in the literature. MaxCOSD waits for a new proposed target to become feasible before committing to it~\citep{hihat2023online}. DDM is a projected base-stock method designed for stochastic multiproduct systems with a linear warehouse constraint that only updates its target when unbiased gradient estimates become available~\citep{shi_nonparametric_nodate}. HT-OGD uses the simple every-round hidden-target projection principle. Since DDM was designed for a single linear warehouse constraint, most experiments use the common testing environment $\Y_C=\{y\in\R_+^n:\norm{y}_1\le C\}$ with lost-sales dynamics and newsvendor loss. However, we also test with non-linear capacity constraints to demonstrate the generalizability of our results. The closest hidden-target OIO experiments in \citet{ichikawa2026nonstationary} are single-item, whereas our algorithmic comparisons are multiproduct. In experiments where we compare different stepsizes for HT-OGD suggested by our theoretical results in Section~\ref{sec:instantiations}, we use a single product setting to establish a baseline for these new experimental settings. In this case, the linear feasible set $\Y_C$ reduces to an interval and we choose sufficiently large $C$ to be non-binding. All experiments show an error envelope of $\pm 2$ standard error, though this is sufficiently small in some plots to be difficult to observe.

Unless stated otherwise, $h_i=1$ and $p_i=200$ in newsvendor loss. We tune a scalar learning-rate multiplier $\gamma$ for each method, compute static comparators offline by optimizing over the realized demand sequence, and plot $\pm2$ standard-error envelopes over repeated runs. Appendix~\ref{app:experiment_details} gives the full parameter choices.

Figure~\ref{fig:l1_loss} reproduces Setting 3 of \citet{hihat2023online}: $n=100$ products, independent Poisson demands with rates drawn from $\mathrm{Uniform}[1,2]$, and capacity $C=175$. HT-OGD attains lower regret at its best tuning and is less sensitive to the learning-rate multiplier. In the time-series panel, MaxCOSD accumulates regret much faster on the plotted horizon, DDM behaves sublinearly, and HT-OGD remains the most stable. We additionally isolate the time-series regret plot for HT-OGD to see that its regret is still sublinear rather than flat, but at a significantly smaller scale than the algorithms. This stochastic instance is not meant to be a worst-case construction, it is a common benchmark showing that the hidden-target projection principle also behaves competitively in the capacitated multiproduct regime where cycle methods were originally tested.

\begin{figure}[t]
    \centering
    \begin{subfigure}[b]{0.32\linewidth}
        \centering
        \FigureGraphic[width=\linewidth]{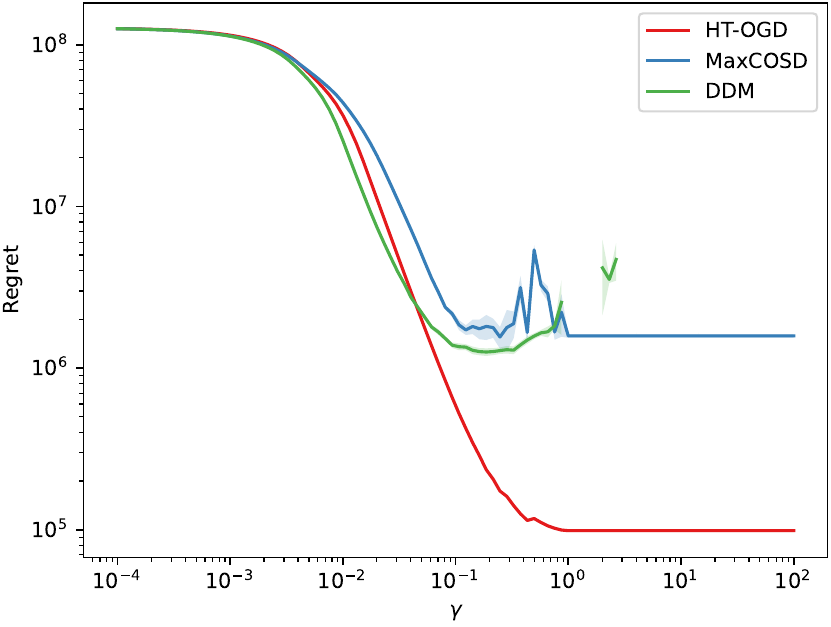}
        \caption{Learning-rate sweep.}
        \label{fig:l2_loss_lr_curve}
    \end{subfigure}
    \hfill
    \begin{subfigure}[b]{0.32\linewidth}
        \centering
        \FigureGraphic[width=\linewidth]{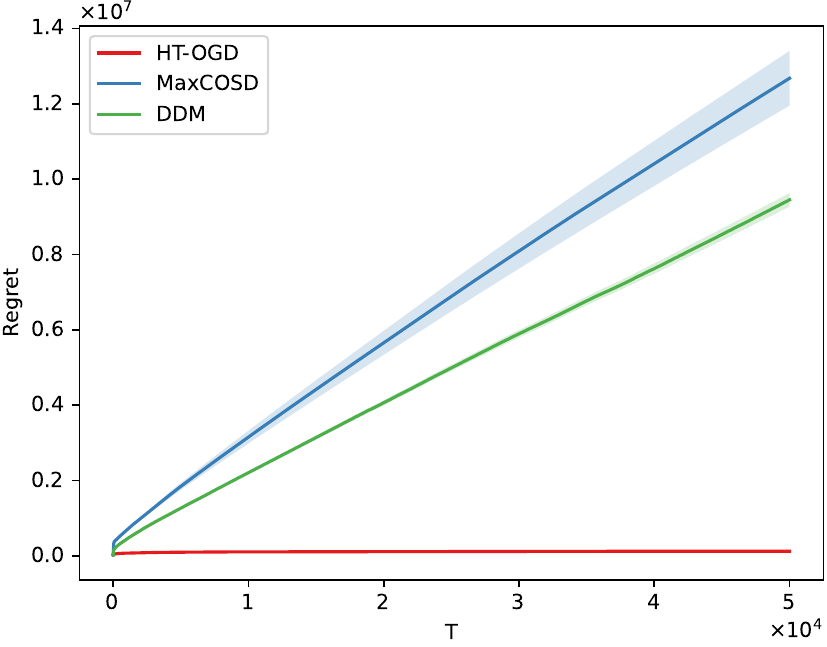}
        \caption{Regret over time.}
        \label{fig:l2_loss_t_curve}
    \end{subfigure}
    \hfill
    \begin{subfigure}[b]{0.32\linewidth}
        \centering
        \FigureGraphic[width=\linewidth]{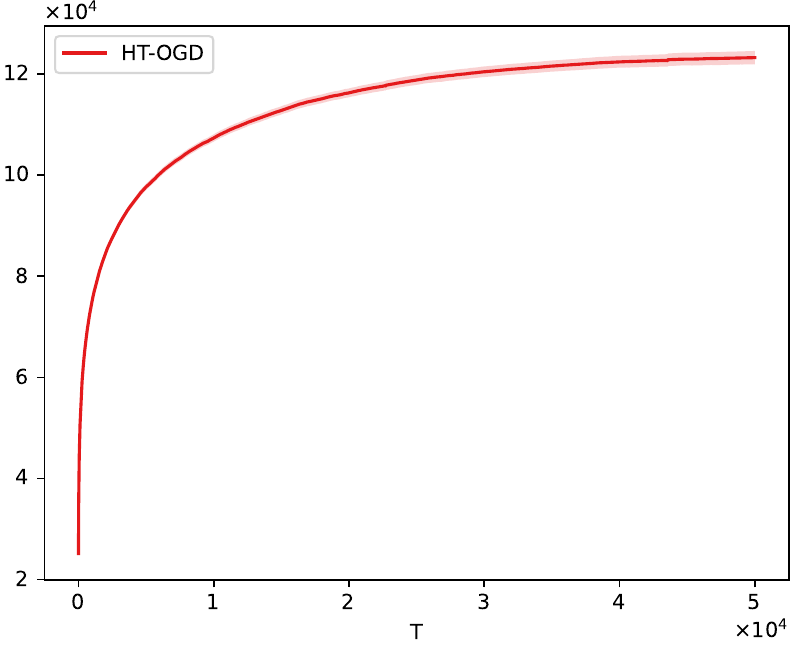}
        \caption{HT-OGD only.}
        \label{fig:l2_loss_t_ht_ogd_curve}
    \end{subfigure}
    \caption{Static regret on the same setting as Figure~\ref{fig:l1_loss}, but with a curved capacity set $\Y=\{y\in\R_+^n:\norm{y}_2\le R\}$. Panels (a)--(c) follow the same convention as in Figure~\ref{fig:l1_loss}.}
  \label{fig:l2_loss}
\end{figure}

Figure~\ref{fig:l2_loss} repeats the experiment on the curved capacity set $\Y=\{y\in\R_+^n:\norm{y}_2\le R\}$, a positive $\ell_2$ ball. While MaxCOSD and HT-OGD both have guarantees in this setting, the productwise sell-out arguments underlying DDM~\citep{shi_nonparametric_nodate} require the linear structure of a warehouse polytope and so DDM has no formal guarantees in this setting. Despite this, we find the same pattern as in the linear capacity setting: the performance of each algorithm ranks in the order of HT-OGD, DDM, then MaxCOSD. This demonstrates HT-OGD's efficacy is not limited to linear capacity sets and validates our theoretical findings.

\begin{figure}[t]
    \centering
    \begin{subfigure}[b]{0.32\linewidth}
        \centering
        \FigureGraphic[width=\linewidth]{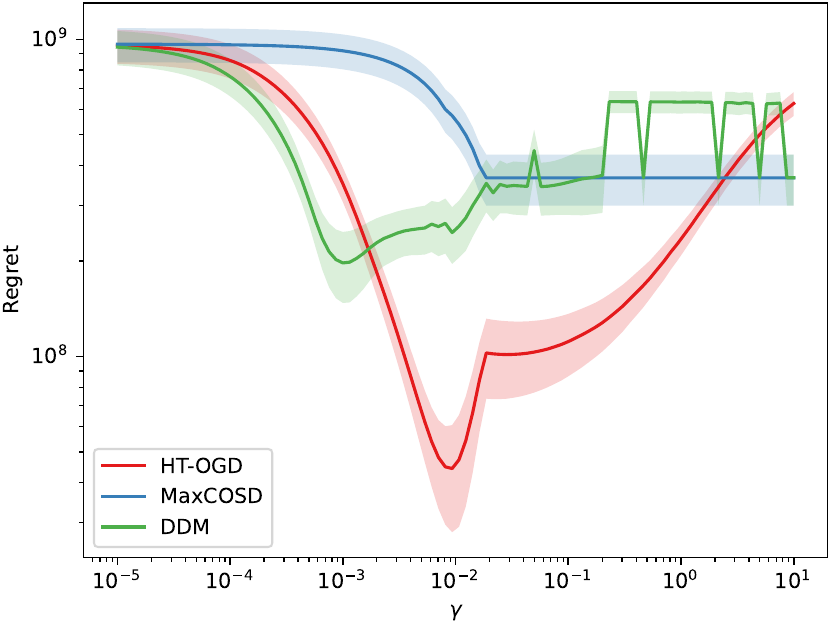}
        \caption{Learning-rate sweep.}
        \label{fig:m5_loss_lr_curve}
    \end{subfigure}
    \hfill
    \begin{subfigure}[b]{0.32\linewidth}
        \centering
        \FigureGraphic[width=\linewidth]{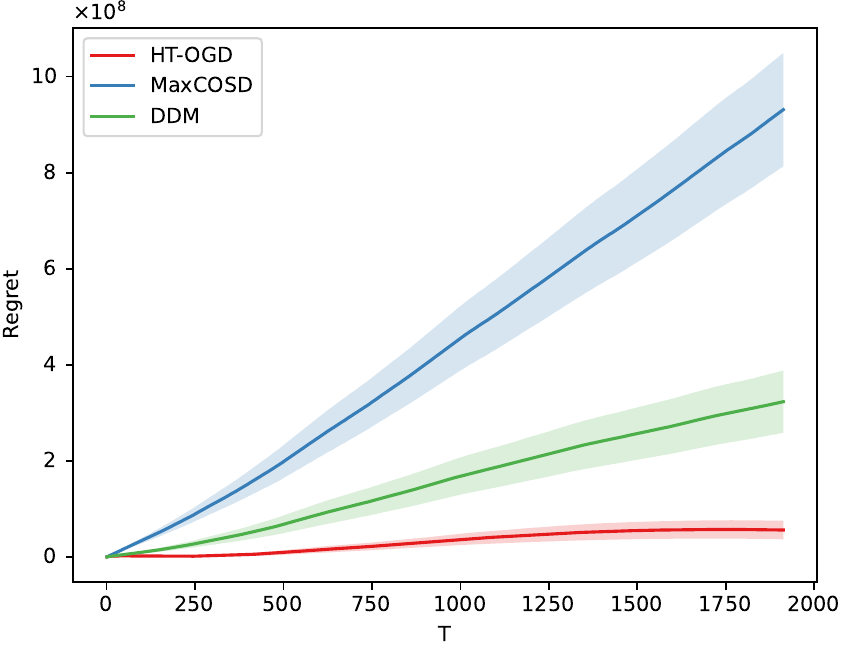}
        \caption{Regret over time.}
        \label{fig:m5_loss_t_curve}
    \end{subfigure}
    \hfill
    \begin{subfigure}[b]{0.32\linewidth}
        \centering
        \FigureGraphic[width=\linewidth]{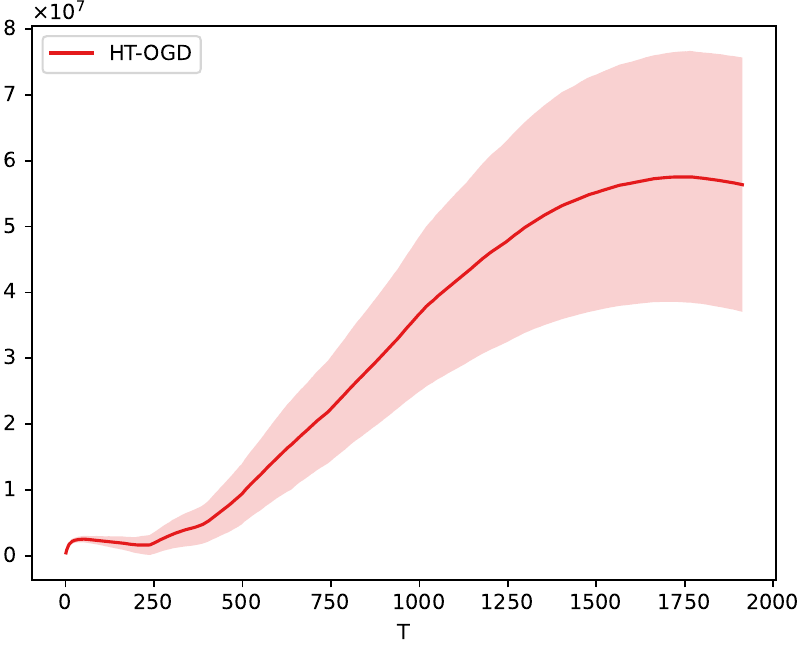}
        \caption{HT-OGD only.}
        \label{fig:m5_loss_t_ht_ogd_curve}
    \end{subfigure}
    \caption{Static regret on demand traces from the M5 forecasting competition, a real-world retail benchmark with $n=3049$ items. The dataset contains realizations from 10 stores which are treated as batches, resulting in larger error envelopes than other experiments. Panels (a)--(c) follow the convention of Figure~\ref{fig:l1_loss}.}
  \label{fig:m5_loss}
\end{figure}

Figure~\ref{fig:m5_loss} replaces the synthetic Poisson stream with demand traces from the M5 forecasting competition~\citep{makridakis_m5_2022}, a retail benchmark containing real demand traces from 10 Walmart stores with $n=3049$ items over a period of $T=1913$ days while using a linear capacity constraint. This probes behavior under non-stationary, non-i.i.d.\ demand and demonstrates the same story as the previous figures. Even without i.i.d. demand, HT-OGD achieves both improved theoretical guarantees and empirical performance.

The remaining experiments probe different aspects of our theoretical results. Figure~\ref{fig:dim_scaling_regret} uses a power-law demand experiment to test explicit dimension dependence while keeping aggregate demand and loss scale comparable. While regret may still grow based on other problem parameters like $D$ and $G$, the plot demonstrates the dimensional-scaling advantage our analysis predicts. The strongly convex loss experiment in Figure~\ref{fig:strong_convex_regret} distinguishes the performance of HT-OGD under the $1/t$ schedule from the general convex $1/\sqrt t$ schedule. The separation is consistent with our predicted polylogarithmic and sublinear regret rates for each corresponding stepsize. Lastly, the dynamic regret experiment in Figure~\ref{fig:comparator_adapted_dynamic_regret} shows the advantage of using a known-variation constant stepsize for HT-OGD when the comparator moves with $P_{T,2}=O(\sqrt T)$. The plot shows that this adapted stepsize is consistent with the $O(\sqrt{P_{T,2}T}) = O(T^{3/4})$ dynamic regret predicted by Theorem~\ref{thm:ht_ogd_dynamic_uppd}. 

\begin{figure}[t]
    \centering
    \begin{subfigure}[b]{0.32\linewidth}
        \centering
        \FigureGraphic[width=\linewidth]{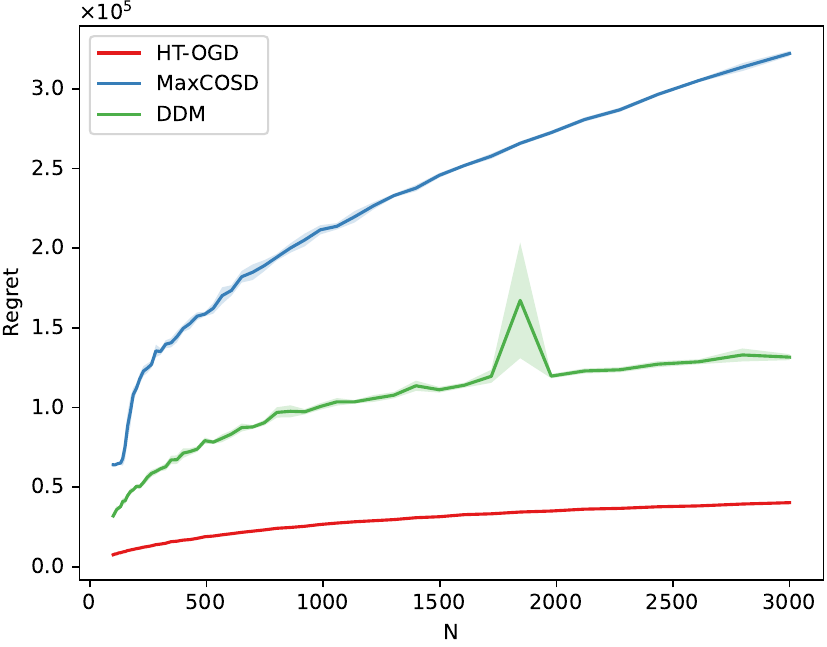}
        \caption{Dimension scaling.}
        \label{fig:dim_scaling_regret}
    \end{subfigure}
    \hfill
    \begin{subfigure}[b]{0.32\linewidth}
        \centering
        \FigureGraphic[width=\linewidth]{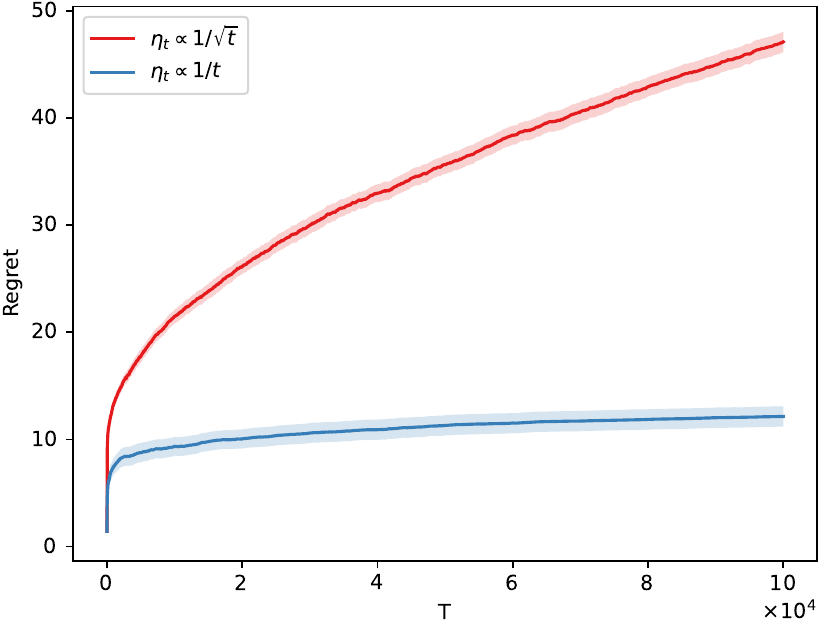}
        \caption{Strong convexity.}
        \label{fig:strong_convex_regret}
    \end{subfigure}
    \hfill
    \begin{subfigure}[b]{0.32\linewidth}
        \centering
        \FigureGraphic[width=\linewidth]{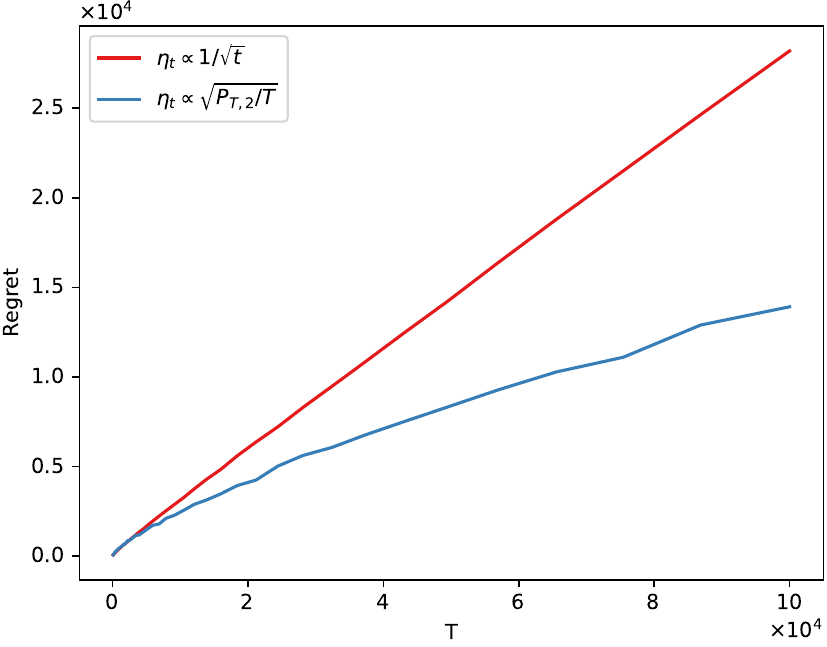}
        \caption{Dynamic regret.}
        \label{fig:comparator_adapted_dynamic_regret}
    \end{subfigure}
    \caption{Additional experiments. (a) Static regret as the number of products grows under power-law demand $d_{t,i} \propto i^{-1.1}$. (b) Regret performance for quadratic strongly convex loss, comparing $\eta_t\propto1/\sqrt t$ and $\eta_t\propto1/t$. (c) Dynamic regret for blockwise-constant demand with $P_{T,2}=O(\sqrt T)$, comparing the static schedule with the known-variation constant stepsize.}
  \label{fig:other_experiments}
\end{figure}

\section{Conclusion}\label{sec:conclusion}

This paper identifies the right geometric state variable for general-convex OIO: the norm-aligned distance from the hidden target to the implementable set. With norm alignment, this distance evolves as a queue with target movement as arrival and common demand as service, reducing general-convex OIO to OCO with switching cost. Under our analysis, the state-dependent feasibility cost which forced prior work into product-by-product accounting under linear capacity or into adaptive cycle control becomes an ordinary switching-cost penalty on the base learner's target movement.

Concretely, applying the reduction with a projected OGD base learner yields the optimal $\widetilde O (\sqrt{T/\mu})$ static regret on arbitrary bounded convex capacity sets, matched by a new lower bound under UPPD. The same reduction with a constant stepsize OGD base learner gives polylogarithmic regret under strongly convex losses, resolving an open direction raised by \citet{hihat2023online}. With an SOGD base learner, it gives Euclidean-path-variation dynamic regret on general convex sets, sharpening \citet{ichikawa2026nonstationary}'s $\ell_1$-path-variation bound. All three guarantees remove the explicit $n^{1/4}$ dimension factor that appeared in prior $\ell_1$-based analyses.

Beyond the regret bounds themselves, the norm alignment principle decouples the inventory analysis from the learning algorithm. Any base learner producing hidden targets with controlled movement in the aligned norm plugs into the same reduction, which we demonstrate concretely with OGD, SOGD, and mirror descent. The inventory layer never inspects how targets are generated, and the base learner never reasons about sell-out events or feasible sets. We hope this modular template proves useful for other online problems where state-dependent feasibility constraints have previously demanded problem-specific treatments.

This work suggests several directions for future research. The most immediate is to move beyond UPPD. In large-scale systems with thousands of products, the probability that the common demand exceeds a positive threshold becomes increasingly small for even simple demand sequences. A local clearing condition, requiring demand only where the implemented action exceeds the hidden target, would better match sparse demand while preserving the separation between learning and feasibility. Another direction is to develop adaptive guarantees that do not require advance knowledge of problem-dependent parameters such as $\mu$ or $\rho$. Our instantiations use these parameters in the choice of stepsize or clearing window that is passed to the base learner. The modular queue analysis is naturally compatible with parameter-free OCO methods \citep{orabona2019modern}, and recovering the same rates without offline tuning would make the framework more powerful.

\bibliography{arxiv_citations}

\clearpage
\appendix

\addtocontents{toc}{\protect\setcounter{tocdepth}{2}}
\renewcommand{\contentsname}{Appendix Contents}
\tableofcontents
\clearpage

\section{Detailed experiment setup}
\label{app:experiment_details}

Unless stated otherwise, experiments use the linear warehouse set $\Y_C=\{y\in\R_+^n:\|y\|_1\le C\}$ so that HT-OGD, MaxCOSD, and DDM can be compared on the same feasible region. The L2 capacity experiment (Figure~\ref{fig:l2_loss}) replaces $\Y_C$ with the positive $\ell_2$ ball $\{y\in\R_+^n:\|y\|_2\le R\}$, on which DDM has no formal guarantees. The inventory dynamics are lost sales, $x_{t+1}=(y_t-d_t)^+$, and the basic loss is the newsvendor loss \eqref{eq:newsvendor_loss}. A scalar multiplier $\gamma$ rescales the nominal learning rate in each method. For the convex static experiments, HT-OGD and DDM use $\eta_t\propto 1/\sqrt t$, while MaxCOSD uses the adaptive stepsize prescribed by \citet{hihat2023online}. Static regret is computed against an offline fixed comparator obtained by projected OGD on the realized demand sequence until convergence. When applicable, tests are rerun to reduce variance from stochastic demand generation and we display $\pm 2$ standard error bars in the plots.

All experimental results were computed on an Apple Mac Studio using a M4 Max chip with 48GB of memory. Each experiment finished execution in under 2 hours.

\paragraph{Warehouse capacity ($\ell_1$) setting.} For Figure~\ref{fig:l1_loss}, we use Setting 3 of \citet{hihat2023online}: $n=100$, independent demands $d_{t,i}\sim\mathrm{Poisson}(\lambda_i)$, independent rates $\lambda_i\sim\mathrm{Uniform}[1,2]$, and capacity $C=175$. While \citet{hihat2023online} did not report the value of $C$, we chose the value so that most demand realizations ($\approx97.5\%$) would fall within the capacity set. The learning-rate sweep evaluates at horizon $T=5\times 10^3$ and logarithmically spaced multipliers $\gamma\in[10^{-5},10]$ evaluating over 200 reruns. The time-series panel plots the regret up to the same horizon, using the best multiplier for each algorithm from the sweep.

\paragraph{L2 capacity setting.} The environment in Figure~\ref{fig:l2_loss} uses the same number of products, horizon, and demand generation as in the warehouse capacity setting, but replaces the linear capacity set with a positive $\ell_2$ ball $\Y=\{y\in\R_+^n:\|y\|_2\le R\}$ with $R = 22.5$, also yielding most demand realizations falling in the capacity set as in the linear warehouse experiment. The learning-rate sweep evaluates over logarithmically spaced multipliers $\gamma\in[10^{-4}, 10^2]$ across 20 reruns. The time-series panel uses the best multiplier for each algorithm from the sweep.

\paragraph{M5 forecasting traces.} For Figure~\ref{fig:m5_loss}, the demand traces are taken from the M5 forecasting competition dataset~\citep{makridakis_m5_2022}. This dataset consists of $n=3049$ products over a period of $T=1913$ days from 10 different Walmart stores in the states of California, Texas, and Wisconsin. While the dataset contains contextual information such as the product category, price, and day of the week, we only use the demand realizations in our experiments. We use a linear warehouse capacity set with capacity $C=6600$, which causes most demand realizations to fall within the capacity set as in the previous experiments. We treat the 10 different stores as independent demand traces, meaning that the results are averaged over the simulations from each store. The learning-rate sweep uses logarithmically spaced multipliers $\gamma\in [10^{-5}, 10]$.

\paragraph{Dimension scaling.} Demands are computed as $d_{t,i}=i^{-\alpha}u_{t,i}$ with $\alpha=1.1$ and independent $u_{t,i}\sim\mathrm{Uniform}[0,1]$. Since $\alpha>1$, the total demand magnitude remains comparable as $n$ grows. This experiment is a stress test for algorithmic dimension scaling rather than a direct verification of UPPD, because common clearing can become rare in sparse high-dimensional systems. At all dimensions, we use a capacity of $C = 2.75$. The plot evaluates at logarithmically spaced number of products $n \in [100, 3000]$ with 10 reruns.

\paragraph{Strong convexity.} For the strongly convex experiment, we use a single-product quadratic loss $\ell_t(y)=\frac12(y-d_t)^2$ with i.i.d. Poisson($\lambda$) demand with $\lambda=1.5$. The capacity $C=20$ is chosen to be sufficiently large to be non-binding. We compare HT-OGD with the general convex schedule $\eta_t\propto 1/\sqrt t$ and the strongly convex schedule $\eta_t\propto 1/t$, tuning the scalar multiplier on a shorter horizon of $T=10^4$ before evaluating at the plotted horizon up to $T=10^5$ over 20 reruns.

\paragraph{Dynamic regret.} In the dynamic-regret panel, the single-product demand sequence is a sequence of constant blocks, with block length $\lceil\sqrt T\rceil$ and demand values cycling through the list $\{2,3,4\}$. Under lost-sales, the comparator $u_t=d_t$ is optimal since it is feasible and has zero newsvendor loss, yielding a path variation of $P_{T,2}=O(\sqrt T)$. We compare to the static-regret schedule $\eta_t\propto1/\sqrt t$ with the constant stepsize suggested by Theorem~\ref{thm:ht_ogd_dynamic_uppd}. Note that since the constant stepsize depends on $T$, we rerun the algorithm at many different values of $T$ rather than plotting the cumulative regret from one run with large $T$. The underage cost is set to $p=2$ to avoid a degenerate myopic ordering rule. Since this is a deterministic demand sequence and the algorithm is also deterministic, there is no variance in the results so we do not display error bars.

\section{Why productwise dynamics fail on general convex sets}
\label{app:geometry}

Prior hidden-target analyses in inventory learning track an $\ell_1$ target-implementation gap, which differs from the norm-aligned one presented in this paper: both \citet{shi_nonparametric_nodate,ichikawa2026nonstationary} focus on $\norm{y_t-z_t}_1$ under Euclidean projection, rather than the norm-aligned $\norm{y_t-z_t}_2$ (see Section~\ref{sec:queue_reduction}). \citet{shi_nonparametric_nodate} derived a queue recursion on $\norm{y_t-z_t}_1$ by a coordinatewise update on $|y_t^{(i)} - z_t^{(i)}|$ (their Lemmas~4--6), and \citet{ichikawa2026nonstationary} conducted a per-product cycle analysis (their Lemmas~5--6) whose telescoping bound critically depends on the linear warehouse constraint $\sum_i z_t^{(i)} \le D$. Both strategies decompose the gap by coordinate and control each coordinate's evolution locally.

This appendix explains why the queue recursion in
general-convex OIO must be geometric rather than productwise. In Section~\ref{sec:queue_reduction},
the regret decomposition charges the target-implementation gap
\[
  q_t := \norm{z_t-y_t},
  \qquad
  \tarimpgap=\sum_{t=1}^T q_t .
\]
Thus a queue proof must show that common demand serves this distance.  A
productwise sell-out proof has a different native state variable.  For a target
$z\in\Y$ and inventory state $x$, define the coordinatewise feasibility
violation
\[
  \norm{(x-z)^+}_1 .
\]
This is the quantity on which the inventory dynamics acts directly: demand
reduces the coordinates where the current state exceeds the target, while
target movement can create new violation.

Prior hidden-target analyses \citep{shi_nonparametric_nodate,ichikawa2026nonstationary} under a single linear warehouse constraint effectively
bridge these two quantities by productwise accounting.  They aim to control the
$\ell_1$ implementation gap
\[
  q_t^{(1)}:=\norm{z_t-y_t}_1,
\]
but the local sell-out dynamics controls productwise feasibility violations $\norm{(x_t-z_t)^+}_1$.
For a linear warehouse set, the budget identity $\sum_i y_i\le D$ supplies a
conservation law: mass added to some coordinates must be balanced by mass
removed from others.  This conservation law lets productwise bounds telescope.
On a curved convex capacity set, there is no such identity.

The obstruction can be stated as the failure of the bridge one would need to
turn a productwise queue into a regret bound.  Write
\[
  J(x):=\Y\cap\{y'\in\R^n\mid y'\succeq x\},
  \qquad
  y=\Proj_{J(x)}^{\norm{\cdot}}(z).
\]
A productwise proof would need to compare the regret-relevant gap
$\norm{y-z}_1$ with the productwise obstruction $\norm{(x-z)^+}_1$, either
coordinate by coordinate or after summing over products.  The next proposition
shows that neither comparison holds uniformly.

\begin{proposition}[No productwise bridge to the implementation gap]
There is no finite constant $C$ with the following property: for every closed,
convex, and bounded set $\Y\subseteq\R_+^n$, every $z\in\Y$, every
$x\in\R_+^n$ for which $J(x)$ is nonempty, and every nearest point
\[
  y=\Proj_{J(x)}^{\norm{\cdot}}(z),
\]
one has the coordinatewise bound
\[
  \abs{y_i-z_i}\le C\,(x_i-z_i)^+
  \qquad\text{for all }i\in[n].
\]
Moreover, there is no finite constant $C$ for which the aggregate bound
\[
  \norm{y-z}_1\le C\,\norm{(x-z)^+}_1
\]
holds uniformly over the same class of instances.
\end{proposition}

\begin{proof}
Take
\[
  \Y=\{(u,v)\in\R_+^2\mid u+v^2\le1\}.
\]
This set is closed, bounded, and convex.  For $\delta\in(0,1]$, set
\[
  z^{(\delta)}=(1-\delta^2,\delta),
  \qquad
  x^{(\delta)}=(1,0).
\]
The set $J(x^{(\delta)})$ contains only $(1,0)$, so the projection is
$y=(1,0)$ in every norm.

The coordinatewise claim fails in the second coordinate:
\[
  \abs{y_2-z_2^{(\delta)}}=\delta,
  \qquad
  (x_2^{(\delta)}-z_2^{(\delta)})^+=0.
\]
Thus no finite $C$ can make
$\delta\le C\cdot 0$ hold uniformly.

The aggregate claim also fails, since
\[
  \norm{y-z^{(\delta)}}_1=\delta+\delta^2,
  \qquad
  \norm{(x^{(\delta)}-z^{(\delta)})^+}_1=\delta^2.
\]
The ratio is
\[
  \frac{\norm{y-z^{(\delta)}}_1}
       {\norm{(x^{(\delta)}-z^{(\delta)})^+}_1}
  =
  1+\delta^{-1},
\]
which diverges as $\delta\downarrow0$.
\end{proof}

The consequence is not merely a loss of constants.  The productwise queue
$\norm{(x-z)^+}_1$ can be of order $\delta^2$ while the target-implementation gap charged in regret is of order
$\delta$.  A proof that only follows the coordinatewise obstruction therefore
undercounts the implementation error.  In the example, the inventory lower
bound asks for only a $\delta^2$ increase in the first coordinate, but the
curved boundary $u+v^2=1$ forces a displacement of size $\delta$ in the second
coordinate.  The large part of the projection gap occurs in a coordinate with
no positive-part violation at all.

This is the precise point at which productwise reasoning stops.  Common demand
can serve the feasibility violation $(x-z)^+$, but regret has already charged
the projection distance $\norm{z-y}_1$.  Without a geometry-free comparison
between these two quantities, a recursion for the productwise obstruction does
not imply a recursion for the target-implementation gap:
\[
  \text{control of } \norm{(x_t-z_t)^+}_1
  \quad\not\Rightarrow\quad
  \text{control of } \norm{z_t-y_t}_1 .
\]
Thus one cannot conclude a bound of the form
\[
  \sum_{t=1}^T \norm{z_t-y_t}_1
  \;\lesssim\;
  B\sum_{t=1}^{T-1}\norm{z_{t+1}-z_t}_1
\]
from productwise sell-out accounting on general convex capacity sets.

The geometric queue recursion in Proposition~\ref{prop:target_implementation_gap_queue}
avoids this failed bridge.  It does not try to compare the implementation gap
with $\norm{(x_t-z_t)^+}_1$.  Instead, after observing demand, it constructs a
point $w_t\in J_{t+1}$ directly.  If $z_{t+1}$ lies within distance
$d_t^{\min}$ of the implemented point $y_t$, the short move lemma makes
$z_{t+1}$ feasible and the queue clears.  Otherwise, convexity lets us move
from $y_t$ toward $z_{t+1}$ by distance $d_t^{\min}$ and obtain a feasible
witness $w_t\in J_{t+1}$.

The decisive step is then the aligned projection certificate:
\[
  q_{t+1}
  =
  \distnorm(z_{t+1},J_{t+1})
  \le
  \norm{z_{t+1}-w_t}.
\]
This certificate is available only in the norm used to define the projection.
Once it is available, the triangle inequality gives
\[
  q_{t+1}
  \le
  \bigl[q_t+\norm{z_{t+1}-z_t}-d_t^{\min}\bigr]^+.
\]
Thus common demand becomes service for the regret-relevant queue itself.  The
proof has transformed coordinatewise demand into geometric queue service
without any productwise conservation identity.

This also clarifies why a norm-conversion workaround does not repair the productwise
argument.  Suppose the implementation is the Euclidean projection
\[
  y_{t+1}=\Proj_{J_{t+1}}^{\norm{\cdot}_2}(z_{t+1}),
\]
but we try to analyze the $\ell_1$ queue
\[
  q_{t+1}^{(1)}:=\norm{z_{t+1}-y_{t+1}}_1 .
\]
After constructing the witness $w_t\in J_{t+1}$, the geometric proof would need
\[
  q_{t+1}^{(1)}
  \le
  \norm{z_{t+1}-w_t}_1 .
\]
Euclidean projection does not certify this inequality.  It certifies only
\[
  \norm{z_{t+1}-y_{t+1}}_2
  \le
  \norm{z_{t+1}-w_t}_2 .
\]
Converting norms gives at best
\[
  q_{t+1}^{(1)}
  \le
  \sqrt n\,\norm{z_{t+1}-y_{t+1}}_2
  \le
  \sqrt n\,\norm{z_{t+1}-w_t}_2
  \le
  \sqrt n\,\norm{z_{t+1}-w_t}_1 .
\]
The factor $\sqrt n$ now appears inside the recursion formula, not at the end of
the argument.  Between clearing shocks, repeated substitution weights earlier
target movements by powers of $\sqrt{n}$, causing \emph{exponential} growth.  Therefore the window condition
\[
  \sum_{r=t}^{t+B-1}\norm{z_{r+1}-z_r}\le \rho
\]
no longer guarantees that one common-demand shock of size $\rho$ clears the
queue.

The key insight from the discussion is that, for general convex capacity sets, norm alignment is the mechanism that converts common demand into service for the queue to which regret is actually charged.

\section{Discussion on UPPD and \texorpdfstring{$L_{\max}$}{Lmax}}
\label{app:lmax}

\paragraph{How to read the parameter $\mu$.}
UPPD is best viewed as a common-clearing primitive. One practical model is a store-wide traffic or promotion shock: on ordinary days, demand may be sparse and product-specific, but with conditional probability $\mu$ a system-wide event brings at least $\rho$ demand to every product class. In such a model the expected clearing-window length is of order $1/\mu$. Improving a regret bound from $\sqrt T/\mu$ to $\sqrt{T/\mu}$ is therefore the difference between paying linearly for the waiting time and paying only its square root between common shocks. The improvement is most meaningful when common clearing is infrequent but not impossible; if $\mu$ is exponentially small in the number of products, all UPPD-based guarantees become conservative. This is why the main text states both sides: the $\mu^{-1/2}$ exponent is optimal under UPPD, while replacing UPPD by a more local clearing condition is the next modeling step.

\citet{ichikawa2026nonstationary} use the maximum sell-out window $L_{\max}$ as their demand primitive. Under UPPD, their stated high-probability estimate gives $L_{\max}=\widetilde O(1/\mu^2)$. The elementary argument below tightens this relation to $\widetilde O(1/\mu)$ for the linear-capacity model. This observation is useful for comparison; the main theorems avoid $L_{\max}$ altogether and work directly on arbitrary convex capacity sets.

\begin{proposition}
\label{prop:uppd_lmax}
Consider the linear-capacity model $\Y_{\mathrm{lin}} := \left\{y\in\R_+^n \,|\, \norm{y}_1 \le C\right\}$
where $C>0$ is the capacity parameter. Suppose Assumption~\ref{ass:uppd} holds with parameters $(\mu,\rho)$, and set $m := \left\lceil C/\rho \right\rceil$. Then for every item $i$, every start time $t$ with $t+L-1\le T$, and every $L\in\mathbb N$,
\begin{equation}
    \label{eq:uppd_lmax_demand_sum_bound}
    \mathbb P\!\left(\sum_{s=t}^{t+L-1} d_s^i < C\right)
  \le
  \exp\!\left(m-\frac{\mu L}{2}\right)
\end{equation}
Consequently, if
\[
  L_\delta
  :=
  \Biggl\lceil
    \frac{2}{\mu}
    \Bigl(
      \Bigl\lceil \frac{C}{\rho} \Bigr\rceil
      +
      \log\!\Bigl(\frac{nT}{\delta}\Bigr)
    \Bigr)
  \Biggr\rceil 
\]
then with probability at least $1-\delta$, every item receives at least $C$ units of demand over every length-$L_\delta$ window contained in $[T]$. Equivalently, in the high-probability sense of Remark~3 of \citet{ichikawa2026nonstationary}, we have $L_{\max}\le L_\delta$.
\end{proposition}

\begin{proof}
Fix a product $i$ and timestep $t$, and define the indicator
\[
  X_s := \1\{d_s^{\min} \ge \rho\},
  \qquad
  \text{for all }
  s=t,\dots,t+L-1.
\]
Since $d_s^i \ge d_s^{\min} \ge \rho X_s$, the event $\sum_{s=t}^{t+L-1} d_s^i < C$ implies $\sum_{s=t}^{t+L-1} X_s \le m-1$ because $\sum_s X_s$ takes integer values. For any $\theta>0$, a Chernoff bound gives
\[
  \mathbb P\!\left(\sum_{s=t}^{t+L-1} X_s \le m-1\right)
  \le
  e^{\theta m}
  \mathbb E\!\left[\exp\!\Bigl(-\theta\sum_{s=t}^{t+L-1} X_s\Bigr)\right]
\]
Write $p_s := \mathbb P(d_s^{\min} \ge \rho \mid \F_{s-1})$. Assumption~\ref{ass:uppd} gives $p_s\ge \mu$ almost surely, and therefore
\[
  \mathbb E\!\left(e^{-\theta X_s}\mid\F_{s-1}\right)
  =
  1-p_s+p_s e^{-\theta}
  \le
  1-\mu+\mu e^{-\theta}
\]
Iterating conditional expectations and using the tower property inductively yields
\[
  \mathbb E\!\left[\exp\!\Bigl(-\theta\sum_{s=t}^{t+L-1} X_s\Bigr)\right]
  \le
  (1-\mu+\mu e^{-\theta})^L
\]
Taking $\theta=\log 2$ gives
\[
  \mathbb P\!\left(\sum_{s=t}^{t+L-1} d_s^i < C\right)
  \le
  \mathbb P\!\left(\sum_{s=t}^{t+L-1} X_s \leq m-1\right) 
  \leq
  2^m(1-\mu/2)^L
  \le
  \exp\!\left(m-\frac{\mu L}{2}\right)
\]
where the last inequality used $2^m < e^m$ and $1-x \leq e^{-x}$ with $x = \mu/2$, and thus proves \eqref{eq:uppd_lmax_demand_sum_bound}.

For $L=L_\delta$, the right-hand side of \eqref{eq:uppd_lmax_demand_sum_bound} is at most $\delta/(nT)$. A union bound over at most $nT$ pairs of products and timesteps $(i,t)$ gives the claim. Thus with probability at least $1-\delta$, each product receives demand exceeding the warehouse capacity in at most $L_\delta$ rounds, so we must have $L_{\text{max}} \leq L_\delta$.
\end{proof}

\section{Omitted proofs from Section~\ref{sec:model} and Section~\ref{sec:queue_reduction}}
\label{sec:regret_decomp_proofs}

\subsection{Deterministic clearing and pathwise transfer}
\label{app:pathwise_transfer}

The main text works directly with UPPD. The proofs pass through the following deterministic clearing event.

\begin{assumption}[Windowed clearing]
\label{ass:window_clearing}
There exist $B\in\N$ and $\rho>0$ such that every interval of $B$ consecutive rounds contains a time $s$ with $d_s^{\min}\ge\rho$.
\end{assumption}

\begin{theorem}[Pathwise implementation bound]
\label{thm:target_implementation_gap_bound}
Suppose $q_1=0$ and Assumption~\ref{ass:window_clearing} holds. If $B<T$, assume that for every $t\in [T-B]$,
\begin{equation}
\label{eq:learner_window_motion}
  \sum_{r=t}^{t+B-1}\norm{z_{r+1}-z_r}\le \rho.
\end{equation}
(If $B\ge T$, \eqref{eq:learner_window_motion} is not needed.) Then
\begin{equation}
\label{eq:tc-pathwise}
  \tarimpgap
  =
  \sum_{t=1}^Tq_t
  \le
  B\sum_{t=1}^{T-1}\norm{z_{t+1}-z_t}
  =B\tarmvment.
\end{equation}
\end{theorem}

\begin{theorem}[OIO regret as target regret plus movement]
\label{thm:regret_reduction}
Suppose $q_1=0$, Assumption~\ref{ass:window_clearing} holds, and if $B<T$, \eqref{eq:learner_window_motion} is satisfied. Then for every comparator sequence $u_{1:T}\in\Y^T$,
\begin{equation}
\label{eq:reduction}
  R_T(u_{1:T})\le \baseregret(u_{1:T})+G_*B\tarmvment.
\end{equation}
\end{theorem}

\begin{lemma}[Probabilistic clearing under UPPD]
\label{lem:window_uppd}
Under Assumption~\ref{ass:uppd}, for every $B\in[T]$,
\[
  \mathbb P\Bigl(\exists t\in [T-B+1]\text{ such that }d_s^{\min}<\rho\text{ for all }s=t,\dots,t+B-1\Bigr)
  \le
  Te^{-\mu B}.
\]
In particular, if $B_\delta:=\lceil\mu^{-1}\log(T/\delta)\rceil$ and $B_\delta<T$, then with probability at least $1-\delta$, Assumption~\ref{ass:window_clearing} holds with $B=B_\delta$. When $B_\delta\ge T$, the reduction uses the deterministic loose bound instead.
\end{lemma}

\subsection{Proof of Lemma~\ref{lem:regret_decomp}}
\begin{proof}
\label{proof:regret_decomp}
For any $t \in [T]$, convexity of $\ell_t$ at $y_t$ and the subgradient $g_t \in \partial \ell_t(y_t)$ gives
\[
  \ell_t(y_t)-\ell_t(u_t)
  \le
  \ip{g_t}{y_t-u_t}
  =
  \ip{g_t}{z_t-u_t}
  +
  \ip{g_t}{y_t-z_t}.
\]
Since $\dnorm{g_t} \le G_*$ by Assumption~\ref{ass:convex_and_bounded} and using a generalized version of Hölder's inequality, we have
\[
  \ip{g_t}{y_t-z_t}
  \le
  G_* \norm{y_t-z_t} .
\]
Summing over $t$ proves \eqref{eq:decomp}.
\end{proof}

\subsection{Proof of Proposition~\ref{prop:target_implementation_gap_queue}}
The proof of the target-implementation gap recursion (Proposition~\ref{prop:target_implementation_gap_queue}) relies on the geometric Lemma~\ref{lem:short_move} relating the distance between order levels and the inventory dynamics constraint \eqref{eq:inv_dynamics_constraint}. It also relies on the following projection certificate, which is the precise place where norm alignment enters the queue recursion. We first present these two elementary ingredients and then proceed to the proposition.

\begin{lemma}[Aligned projection certificate]
\label{lem:aligned_projection_certificate}
Fix a norm $\norm{\cdot}$ and a nonempty closed set $J\subseteq\R^n$. If
\[
  y\in\argmin_{u\in J}\norm{z-u},
\]
equivalently $y=\Projnorm{J}{z}$ for one choice of nearest point, then for every $w\in J$,
\begin{equation}
\label{eq:aligned_projection_certificate}
  \norm{z-y}=\distnorm(z,J)\le \norm{z-w}.
\end{equation}
The statement is norm-specific: a projection in a different norm certifies distance only in that different norm.
\end{lemma}

\begin{proof}
The inequality is the defining optimality property of a nearest point in the norm $\norm{\cdot}$. If $y$ minimizes $\norm{z-u}$ over $u\in J$, then no feasible witness $w\in J$ can be closer to $z$ in the same norm.
\end{proof}

\begin{proof}[Proof of Lemma~\ref{lem:short_move}]
\label{proof:short_move}
For each coordinate $i$, using $\abs{y_i'-y_i}\le\norm{y'-y}$ from admissibility of $\norm{\cdot}$ (Definition~\ref{def:admissible_norm}),
\[
  y_i' \ge y_i - \abs{y_i'-y_i} \ge y_i - \norm{y'-y} \ge y_i-d^{\min} \ge y_i-d_i .
\]
Since $y_i' \ge 0$, this implies
\[
  y_i' \ge \max\{y_i-d_i,0\} = (y_i-d_i)^+.
\]
\end{proof}

\begin{proof}[Proof of Proposition~\ref{prop:target_implementation_gap_queue}]
\label{proof:target_implementation_gap_queue}
Fix $t \in [T-1]$ and set $r_t := \norm{z_{t+1}-y_t}$. All distances in this proof are measured in the projection norm, and $y_{t+1}$ is the nearest point to $z_{t+1}$ in that same norm. This alignment is what will allow us to apply Lemma~\ref{lem:aligned_projection_certificate}. We break the proof into cases dependent on the size of $r_t$ relative to $d_t^{\min}$:

\paragraph{Case 1.} If $r_t \le d_t^{\min}$, then Lemma~\ref{lem:short_move} yields
\[
  z_{t+1} \succeq (y_t-d_t)^+ \succeq x_{t+1},
\]
where the second inequality uses \eqref{eq:inv_dynamics_constraint}. Since $z_{t+1}\in\Y$, we have $z_{t+1} \in J_{t+1}$ and therefore $q_{t+1}=0$.

\paragraph{Case 2.} Now suppose $r_t > d_t^{\min}$. Define
\[
  \alpha_t := \frac{d_t^{\min}}{r_t} \in [0,1)
  \qquad\text{and}\qquad
  w_t := (1-\alpha_t)y_t + \alpha_t z_{t+1}.
\]
Because $\Y$ is convex and $y_t,z_{t+1}\in\Y$, we have $w_t\in\Y$. By positive homogeneity of $\norm{\cdot}$,
\[
  \norm{w_t-y_t} = \norm{\alpha_t(z_{t+1} - y_t)} = \alpha_t r_t = d_t^{\min} .
\]
Then applying Lemma~\ref{lem:short_move} gives
\[
  w_t \succeq (y_t-d_t)^+ \succeq x_{t+1},
\]
so $w_t \in J_{t+1}$. By the aligned projection certificate (Lemma~\ref{lem:aligned_projection_certificate}) applied in the norm $\norm{\cdot}$ to the feasible witness $w_t$,
\begin{equation}
\label{eq:aligned_projection_step}
  q_{t+1}
  =\distnorm(z_{t+1},J_{t+1})
  \le
  \norm{z_{t+1}-w_t}
  = \norm{(1-\alpha_t)(z_{t+1} - y_t)}
  = (1-\alpha_t)r_t
  = r_t-d_t^{\min} .
\end{equation}
Finally by the triangle inequality,
\[
  r_t
  =
  \norm{z_{t+1}-y_t}
  \le
  \norm{z_{t+1}-z_t} + \norm{z_t-y_t}
  =
  \norm{z_{t+1}-z_t} + q_t .
\]
So we may write:
\[
q_{t+1} \leq q_t + \norm{z_{t+1} - z_t} - d^{\min}_t
\]
Combining the two cases proves \eqref{eq:target_implementation_gap_queue}.
\end{proof}

Bounding $q_{t+1}$ by the distance from $z_{t+1}$ to $w_t$ in \eqref{eq:aligned_projection_step} is the crucial step that utilizes norm alignment. The same norm defines $q_{t+1}$, the projection $\Projnorm{J_{t+1}}{z_{t+1}}$, and the distance to the witness $w_t$. The coefficient of the queue's residual in this inequality is exactly one because \eqref{eq:aligned_projection_step} was proved in the same norm that defines $q_{t+1}$. Replacing that step with a different norm via norm conversion would give a multiplicative constant inside the recursion rather than the queue recursion \eqref{eq:target_implementation_gap_queue}.

\subsection{Proof of Theorem~\ref{thm:target_implementation_gap_bound}}

We first provide the proof's outline and analytical approach. We divide the horizon into excursions of the queue away from zero. On each excursion, repeated use of the queue recursion bounds the queue by cumulative hidden movement. The windowed-clearing assumption then forces every full excursion to end within at most $B$ rounds, and the last partial excursion has length at most $B$ as well.

\begin{proof}
\label{proof:target_implementation_gap_bound}
For $t \in [T-1]$, set $a_t := \norm{z_{t+1}-z_t}$. We first record a horizon-free loose bound. If $q_\tau=0$, then repeated use of \eqref{eq:target_implementation_gap_queue} and $d_r^{\min}\ge0$ gives
\begin{equation}
    \label{eq:queue_loose_window_bound}
    q_{t+1} \le \sum_{r=\tau}^{t} a_r
  \qquad \text{for every } t\in\{\tau,\dots,T-1\}.
\end{equation}
Indeed, this follows by induction from $q_\tau=0$ and $q_{r+1}\le q_r+a_r$. If $B\ge T$, then \eqref{eq:queue_loose_window_bound} with $\tau=1$ gives
\[
  \sum_{t=1}^T q_t
  \le
  \sum_{t=1}^T\sum_{r=1}^{t-1}a_r
  \le
  T\sum_{r=1}^{T-1}a_r
  \le
  B\sum_{r=1}^{T-1}a_r,
\]
so the result holds. We therefore assume $B<T$ for the rest of the proof.

We next show that if the queue is empty and at least $B$ rounds remain, then it empties again within the next $B$ rounds. Fix $\tau \le T-B$ with $q_\tau=0$. By Assumption~\ref{ass:window_clearing}, there exists $s\in\{\tau,\dots,\tau+B-1\}$ with $d_s^{\min}\ge \rho$. Applying \eqref{eq:target_implementation_gap_queue} at time $s$, then using $q_s = 0$ if $s = \tau$ or \eqref{eq:queue_loose_window_bound} otherwise, and then finally applying \eqref{eq:learner_window_motion},
\begin{equation}
\label{eq:exact_queue_clearing_step}
  q_{s+1}
  \le
  \left[q_s + a_s - d^{\min}_s\right]^+ \leq
  \Bigl[\sum_{r=\tau}^{s} a_r - d_s^{\min}\Bigr]^+
  \le
  \Bigl[\sum_{r=\tau}^{\tau+B-1} a_r - \rho\Bigr]^+
  = 0.
\end{equation}
Thus, if $q_\tau=0$ and $\tau\le T-B$, then there exists $\tau'\in\{\tau+1,\dots,\tau+B\}$ with $q_{\tau'}=0$. In other words, whenever the queue is empty and a full clearing window remains, it must empty again within the next $B$ rounds. 

Now define zero times recursively. Let $\tau_1=1$. Given $\tau_j$, if $\tau_j\le T-B$, let $\tau_{j+1}$ be the smallest index in $\{\tau_j+1,\dots,\tau_j+B\}$ such that $q_{\tau_{j+1}}=0$; if $\tau_j>T-B$, set $\tau_{j+1}=T+1$ and stop. This yields
\[
  1=\tau_1 < \tau_2 < \cdots < \tau_m < \tau_{m+1}=T+1,
\]
with $q_{\tau_j}=0$ and $\tau_{j+1}-\tau_j \le B$ for every $j \in [m]$. For the last interval this follows because $\tau_m>T-B$, hence $\tau_{m+1}-\tau_m=T+1-\tau_m\le B$.

Now, fix $j\in[m]$. Since $q_{\tau_j}=0$, we may apply \eqref{eq:queue_loose_window_bound} in the first inequality below to bound the sum of the queue lengths between $\tau_j$ and $\tau_{j+1} - 1$:
\[
  \sum_{t=\tau_j}^{\tau_{j+1}-1} q_t
  \le
  \sum_{t=\tau_j}^{\tau_{j+1}-1}\sum_{r=\tau_j}^{t-1} a_r
  =
  \sum_{r=\tau_j}^{\tau_{j+1}-2} (\tau_{j+1}-1-r)a_r
  \le
  B\sum_{r=\tau_j}^{\tau_{j+1}-2} a_r.
\]
where the last inequality used $\tau_{j+1} - \tau_j \leq B$. The index sets $\{\tau_j,\dots,\tau_{j+1}-2\}$ are disjoint subsets of $[T-1]$, so summing over $j$ yields
\[
  \tarimpgap
  =
  \sum_{j=1}^m \sum_{t=\tau_j}^{\tau_{j+1}-1} q_t
  \le
  B\sum_{t=1}^{T-1} a_t,
\]
which is exactly \eqref{eq:tc-pathwise}.
\end{proof}

The clearing argument leverages the exact coefficient of one in the recursion \eqref{eq:target_implementation_gap_queue}. If we utilized norm conversion and instead produced an inequality of the form $q_{t+1}\le C[q_t+a_t-d_t^{\min}]^+$ with $C>1$, the factor $C$ would sit inside the dynamics introducing a factor exponential in the window length $C^B$ in the worst-case. The condition $\sum a_t\le\rho$ over a window would no longer imply that one demand shock of size $\rho$ clears the queue and would instead leave residual backlog after the shock. This helps justify the necessity of norm alignment in the analysis.

\subsection{Proof of Lemma~\ref{lem:window_uppd}}

\begin{proof}
\label{proof:window_uppd}

    For each $s \in [T]$, define the event $A_s := \{d_s^{\min} < \rho\}$. Then since $d_s$ is $\mathcal{F}_s$-measurable, so is $d_s^{\min}$ and hence $A_s$. By UPPD, we have 
    \begin{equation}
        \label{eq:small_demand_bound}
        \mathbb{P}(A_s | \mathcal{F}_{s-1}) = 1-P(d_s^{\min} \geq \rho | \mathcal{F}_{s-1}) \leq 1-\mu \quad \text{a.s.}
    \end{equation}
    Assume $B\le T$; otherwise the event in the statement is empty. Fixing any  $t\in[T-B+1]$, we first show by induction on $k \in [B]$ that $\mathbb{P}(\cap_{s=t}^{t+k-1} A_s) \leq (1-\mu)^k$. The case for $k=1$ follows from \eqref{eq:small_demand_bound} and the tower property:
    \[
    \mathbb{P}(A_t) = \mathbb{E}[\mathbb{P}(A_t | \mathcal{F}_{t-1})] \leq 1-\mu
    \]
    For the inductive step, we have
    \begin{align*}
        \mathbb{P}\left(\bigcap_{s=t}^{t+k} A_s\right) & = \mathbb{E}[\1_{A_t} \cdots \1_{A_{t+k}}]\\
        & =\mathbb{E}\left[\mathbb{E}[\1_{A_t} \cdots \1_{A_{t+k}} | \mathcal{F}_{t+k-1}]\right]\\
        & = \mathbb{E}[\1_{A_t}\cdots \1_{A_{t+k-1}}\cdot \mathbb{P}(A_{t+k} | \mathcal{F}_{t+k-1})]\\
        & \leq \mathbb{E}[\1_{A_t}\cdots \1_{A_{t+k-1}}(1-\mu)]\\
        & = (1-\mu) \mathbb{P}\left(\bigcap_{s=t}^{t+k-1} A_s\right)\\
        & \leq (1-\mu)(1-\mu)^{k} = (1-\mu)^{k+1}
    \end{align*}
    where the second equality is due to the tower property, the first inequality follows from \eqref{eq:small_demand_bound}, and the last bound follows by the inductive hypothesis. This completes the induction. Now setting $k = B$ and using $1-\mu \leq e^{-\mu}$, we have
    \[
    \mathbb{P}\left(\bigcap_{s=t}^{t+B-1} A_s\right) \leq (1-\mu)^B \leq e^{-\mu B}
    \]
    The event that we wish to show has low probability is then $\mathcal{A} := \cup_{t=1}^{T-B+1}\cap_{s=t}^{t+B-1} A_s$. Combining the previous bound with a union bound, we have
    \[
    \mathbb{P}(\mathcal{A}) \leq \sum_{t=1}^{T-B+1} \mathbb{P}\left(\bigcap_{s=t}^{t+B-1} A_s\right) \leq \sum_{t=1}^{T-B+1} e^{- \mu B} \leq Te^{-\mu B}
    \]
    which completes the first part of the proof. 
    
    For the second part, since $\mu B_\delta \geq \log(T/\delta)$, we may bound $Te^{-\mu B_\delta}$ by $Te^{-\log(T/\delta)} = \delta$ thereby proving the second part of the lemma.
\end{proof}

\subsection{Proof of Theorem~\ref{thm:uppd_reduction}}
\begin{proof}
If $B_\delta\ge T$, then the loose bound \eqref{eq:queue_loose_window_bound} from the proof of Theorem~\ref{thm:target_implementation_gap_bound}, applied with $\tau=1$, gives deterministically $\tarimpgap\le T\tarmvment\le B_\delta\tarmvment$. If $B_\delta<T$, Lemma~\ref{lem:window_uppd} gives windowed clearing with $B=B_\delta$ with probability at least $1-\delta$. On this event, the movement condition Assumption~\ref{ass:base_learner}(ii) is exactly \eqref{eq:learner_window_motion} with $B=B_\delta$, so Theorem~\ref{thm:target_implementation_gap_bound} gives $\tarimpgap\le B_\delta\tarmvment$. Combining this with Lemma~\ref{lem:regret_decomp} yields $R_T(u_{1:T})\le \baseregret(u_{1:T})+G_* B_\delta\tarmvment$. If Assumption~\ref{ass:base_learner}(i) additionally holds, substituting $\baseregret(u_{1:T})\le\mathcal{R}_T(u_{1:T})$ gives the stated $\mathcal{R}$-form of the bound.
\end{proof}

\section{Omitted proofs from Section~\ref{sec:instantiations}}
\label{appendix:meta_alg_instantiations_proofs}

\subsection{General convex losses}
\label{appendix:general_convex_losses_proofs}

We first record the standard movement and regret estimates for Euclidean projected OGD:

\begin{lemma}
\label{lem:ht_ogd_proj_grad_error}
For an $\ell_2$ norm $\norm{\cdot} = \norm{\cdot}_2$, the target states $z_t$ and subgradients $g_t$ in HT-OGD (Algorithm~\ref{alg:ht_ogd}), a positive nonincreasing stepsize sequence $\eta_1\ge\eta_2\ge\cdots>0$, and a comparator $u \in \Y$, we bound the following sum of inner products:

\begin{equation}
    \sum_{t=1}^T \ip{g_t}{z_t - u} \leq \frac{D_2^2}{2\eta_T} + \frac{1}{2}\sum_{t=1}^T \eta_t \norm{g_t}_2^2
\end{equation}
    
\end{lemma}

\begin{proof}
    Using the linearity of the inner product, we may write:
    \[
    \ip{g_t}{z_t - u} = \frac{1}{2\eta_t}\left(\norm{z_t - u}_2^2 + \eta_t^2\norm{g_t}_2^2 - \norm{(z_t - u) - \eta_t g_t}_2^2\right)
    \]
    We relate this to $z_{t+1}$ using nonexpansiveness of the Euclidean projection:
    \[
    \norm{z_{t+1} - u}_2^2 = \norm{\Proj_\Y^{\norm{\cdot}_2}(z_t - \eta_t g_t) - \Proj_\Y^{\norm{\cdot}_2}(u)}_2^2 \leq \norm{(z_t - \eta_tg_t) - u}_2^2
    \]
    Combining the previous two lines results in:
    \[
    \ip{g_t}{z_t - u} \leq \frac{1}{2\eta_t}\left(\norm{z_t - u}_2^2 + \eta_t^2\norm{g_t}_2^2 - \norm{z_{t+1} - u}_2^2\right)
    \]
    Summing over $t \in [T]$ yields:

    \begin{align*}
        \sum_{t=1}^T \ip{g_t}{z_t - u} & \leq \sum_{t=1}^T \frac{1}{2\eta_t}\left(\norm{z_t - u}_2^2 - \norm{z_{t+1} - u}_2^2\right) + \frac{1}{2}\sum_{t=1}^T \eta_t \norm{g_t}_2^2\\
        & = \frac{1}{2}\left(\frac{\norm{z_1-u}_2^2}{\eta_1} - \frac{\norm{z_{T+1} - u}_2^2}{\eta_{T}} + \sum_{t=2}^{T} \left(\frac{1}{\eta_t} - \frac{1}{\eta_{t-1}}\right) \norm{z_t - u}_2^2 \right)\\
        & + \frac{1}{2}\sum_{t=1}^T \eta_t \norm{g_t}_2^2\\
        & \leq \frac{1}{2}\left(\frac{D_2^2}{\eta_1} + \sum_{t=2}^T \left(\frac{1}{\eta_t} - \frac{1}{\eta_{t-1}}\right)D_2^2\right) + \frac{1}{2}\sum_{t=1}^T \eta_t \norm{g_t}_2^2\\
        & = \frac{D_2^2}{2\eta_T} + \frac{1}{2}\sum_{t=1}^T \eta_t \norm{g_t}_2^2\\
    \end{align*}
    where the last inequality used the monotonicity of the stepsizes and Assumption~\ref{ass:convex_and_bounded}.
\end{proof}

To build to regret bound under UPPD, we first prove the following regret bound under a window clearing assumption:

\begin{proposition}[OGD verifies the base learner conditions on the Euclidean queue]
\label{prop:ogd_base_learner}
Take the projection norm to be Euclidean, $\norm{\cdot}=\norm{\cdot}_2$, and write $D_2$ and $G_2$ for the capacity-set diameter and dual gradient bound from Assumption~\ref{ass:convex_and_bounded} specialized to the $\ell_2$ queue. Fix $\delta\in(0,1)$, set $B_\delta:=\lceil\mu^{-1}\log(T/\delta)\rceil$, and choose a learning-rate parameter $\gamma$ and step sizes $\eta_t$ according to
\begin{equation}
\label{eq:ht_ogd_lr_choice}
  0 < \gamma \le \frac{\rho}{2D_2\sqrt{B_\delta}},
  \qquad
  \eta_t = \frac{\gamma D_2}{G_2\sqrt{t}}.
\end{equation}
Then HT-OGD (Algorithm~\ref{alg:ht_ogd}), initialized from any $z_1\in J_1$, produces a target sequence $(z_t)_{t=1}^T\subset\Y$ satisfying Assumption~\ref{ass:base_learner}:
\begin{enumerate}[label=(\roman*),leftmargin=2em,itemsep=2pt,topsep=2pt]
  \item Linearized regret bound holds with
  \[
    \mathcal{R}_T(u) := D_2 G_2\Bigl(\tfrac{1}{2\gamma}+\gamma\Bigr)\sqrt{T}
  \]
  for every static comparator $u\in\Y$.
  \item Windowed movement bound holds: for every $t$ with $1\le t\le T-B_\delta$,
  \[
    \sum_{r=t}^{t+B_\delta-1}\norm{z_{r+1}-z_r}_2 \le \rho.
  \]
\end{enumerate}
Moreover, the cumulative target movement satisfies
\begin{equation}
\label{eq:ht_ogd_movement_bound}
  \tarmvment \le 2\gamma D_2\sqrt{T}.
\end{equation}
\end{proposition}

\begin{proof}
We start with a proof that HT-OGD satisfies the windowed movement bound (Assumption~\ref{ass:base_learner}(ii)) and establish a bound on $\tarmvment$. Then, we apply Lemma~\ref{lem:ht_ogd_proj_grad_error} to establish the linearized regret bound (Assumption~\ref{ass:base_learner}(i)).

\textbf{Movement and part (ii).} By nonexpansiveness of the $\ell_2$ projection and \eqref{eq:ht_ogd_lr_choice}, for every $t$,
\[
  \norm{z_{t+1}-z_t}_2
  = \norm{\Proj_{\Y}^{\norm{\cdot}_2}(z_t-\eta_t g_t)-\Proj_{\Y}^{\norm{\cdot}_2}(z_t)}_2
  \le \eta_t\norm{g_t}_2
  \le \eta_t G_2
  = \frac{\gamma D_2}{\sqrt{t}}.
\]
For any $t$ with $1\le t\le T-B_\delta$,
\[
  \sum_{r=t}^{t+B_\delta-1}\norm{z_{r+1}-z_r}_2
  \le \gamma D_2\sum_{r=t}^{t+B_\delta-1}\frac{1}{\sqrt{r}}
  \le \gamma D_2\sum_{r=1}^{B_\delta}\frac{1}{\sqrt{r}}
  \le 2\gamma D_2\sqrt{B_\delta}
  \le \rho,
\]
where the last inequality uses \eqref{eq:ht_ogd_lr_choice}, proving part (ii). Using the identical argument, summing the per-step bound $\norm{z_{t+1} - z_t}_2 = \gamma D_2 / \sqrt{t}$ similarly over $t \in [T-1]$ gives $\tarmvment\le 2\gamma D_2\sqrt{T}$ and proves \eqref{eq:ht_ogd_movement_bound}.

\textbf{Part (i).} By Lemma~\ref{lem:ht_ogd_proj_grad_error} and Assumption~\ref{ass:convex_and_bounded},
\begin{align*}
  \baseregret(u)
  = \sum_{t=1}^T\ip{g_t}{z_t-u}
  & \le \frac{D_2^2}{2\eta_T}+\frac{1}{2}\sum_{t=1}^T\eta_t G_2^2\\
  & = \frac{D_2 G_2\sqrt{T}}{2\gamma}+\frac{\gamma D_2 G_2}{2}\sum_{t=1}^T\frac{1}{\sqrt{t}}\\
  &\le D_2 G_2\Bigl(\tfrac{1}{2\gamma}+\gamma\Bigr)\sqrt{T}\\
\end{align*}
using $\sum_{t=1}^T 1/\sqrt{t}\le 2\sqrt{T}$.
\end{proof}

\begin{proof}[Proof of Theorem~\ref{thm:ht_ogd_convex_uppd}]
\label{proof:ht_ogd_convex_uppd}
By Proposition~\ref{prop:ogd_base_learner}, the OGD target sequence satisfies Assumption~\ref{ass:base_learner} with $\mathcal{R}_T(u)=D_2 G_2\bigl(\tfrac{1}{2\gamma}+\gamma\bigr)\sqrt{T}$, and the cumulative target movement is bounded by $\tarmvment\le 2\gamma D_2\sqrt{T}$ via \eqref{eq:ht_ogd_movement_bound}. Then Theorem~\ref{thm:uppd_reduction} applies and gives with probability at least $1-\delta$ for every $u\in\Y$,
\[
  R_T(u)
  \le \mathcal{R}_T(u) + G_2 B_\delta\,\tarmvment
  \le D_2 G_2\Bigl(\frac{1}{2\gamma}+\gamma+2\gamma B_\delta\Bigr)\sqrt{T}.
\]
Optimizing the inner expression, $f(\gamma):=\tfrac{1}{2\gamma}+\gamma(1+2B_\delta)$ over $\gamma>0$, the unconstrained minimum is at $\gamma^\star=1/\sqrt{2(1+2B_\delta)}$, giving $R_T \leq D_2G_2\sqrt{2(1+2B_\delta)T}$. This $\gamma^\star$ satisfies the condition $\gamma\le\rho/(2D_2\sqrt{B_\delta})$ whenever $D_2\le\rho$ (and more generally whenever $2B_\delta(D_2^2-\rho^2)\le\rho^2$). When the constraint binds, the boundary choice $\gamma=\rho/(2D_2\sqrt{B_\delta})$ gives 

\[
f(\gamma) = \frac{D_2\sqrt{B_\delta}}{\rho} + \frac{\rho (1+2B_\delta)}{2D_2\sqrt{B_\delta}} \leq \left(\frac{D_2}{\rho} + \frac{3\rho}{2D_2}\right)\sqrt{B_\delta}
\]

using $1\le B_{\delta}$ in the second inequality. This gives the stated regret bound:
\[
  R_T(u)\le D_2 G_2\left(\frac{D_2}{\rho} + \frac{3\rho}{2D_2}\right)\sqrt{B_\delta T} =\widetilde O\!\left(D_2 G_2\sqrt{T/\mu}\right).
\]

For the expectation bound, take $\delta=1/T$ and $\gamma=\rho/(2D_2\sqrt{B_{1/T}})$, which satisfies the hypothesis of Proposition~\ref{prop:ogd_base_learner}. Let $E_T$ denote the event that the high-probability bound holds, so $\mathbb{P}(E_T)\ge 1-1/T$. On $E_T$,
\[
  R_T(u)
  \le D_2 G_2\left(\frac{D_2}{\rho}+\frac{3\rho}{2D_2}\right)\sqrt{B_{1/T}T}
\]
 On the complement, we use the crude bound $R_T(u)\le D_2 G_2 T$. Combining,
\[
  \mathbb{E}\bigl[R_T(u)\bigr]
  \le D_2 G_2\Bigl(\frac{D_2}{\rho}+\frac{3\rho}{2D_2}\Bigr)\sqrt{B_{1/T}T}
  + D_2 G_2,
\]
which proves the claim.
\end{proof}

\subsection{A matching lower bound under UPPD}
\label{sec:lower}

\begin{proof}[Proof of Theorem~\ref{thm:convex_uppd_lower_bound}]
As previously observed, a no-demand round on the line segment $\Y_L$ freezes the action. Bernoulli demand creates random blocks on which the learner must commit to a single point of $\Y_L$.

Formally, take $x_1=(0,0)$ and let $\xi_1,\dots,\xi_T$ be i.i.d. Bernoulli$(\mu)$. Set
\[
  d_t := \xi_t(1,1),
  \qquad
  x_{t+1}=(y_t-d_t)^+
\]
Then $d_t^{\min}\in\{0,1\}$ and
\[
  \mathbb P(d_t^{\min}\ge 1 \mid \F_{t-1}) = \mu,
\]
so UPPD (Assumption~\ref{ass:uppd}) holds with $\rho=1$. If $\xi_t=1$, then $x_{t+1}=(0,0)$. If $\xi_t=0$, then $x_{t+1}=y_t$. Then for $y_{t+1}$ to be a feasible action, we must have $y_{t+1}\in\Y_L$ and $y_{t+1}\succeq y_t$. Since both vectors have coordinate sum $1$, this forces $y_{t+1}=y_t$. Thus, the horizon is partitioned into blocks between clearings where $\xi_t = 1$, and the learner plays a constant action on each block. So, we count the number of these blocks:
\[
  N:=\sum_{t=1}^{T-1}\xi_t,
  \qquad
  K:=N+1,
\]
and let $L_1,\dots,L_K$ be the corresponding block lengths, so $\sum_{j=1}^K L_j=T$. Set $t_{K+1}:=T+1$. Conditional on the demand path, we now generate the losses block by block. When block $j$ begins, after the learner has committed to its action for that block and before the loss and feedback from the first round of the block are revealed, draw an independent Rademacher sign $\sigma_j\in\{-1,+1\}$ and keep that sign fixed throughout the block. For every round $t$ in block $j$, define
\[
  \ell_t(y)
  :=
  \frac{1-\sigma_j}{2}\,y_1 + \frac{1+\sigma_j}{2}\,y_2
\]
These losses are linear and convex, with gradient $e_1 = (1, 0)$ when $\sigma_j=-1$ and $e_2 = (0, 1)$ when $\sigma_j=+1$, hence $\|\nabla \ell_t\|_2=1$ so we may choose $G = 1$.

Since the learner is frozen within block $j$, it plays some point
$a_j = (a_{j,1}, a_{j,2}) \in \Y_L$ throughout that block, and
$a_{j,1} + a_{j,2} = 1$. The point $a_j$ is $\mathcal{F}_{t_j - 1}$-measurable:
it depends only on the demand realizations $d_{1:t_j - 1}$ before block $j$,
the earlier block signs $\sigma_{1:j-1}$ (revealed through past feedback),
and the learner's internal randomness. By construction, $\sigma_j$ is drawn
independently of $\mathcal{F}_{t_j - 1}$.

To compute the expected total loss, we condition on the demand path. Let
$\mathcal{G} := \sigma(\xi_{1:T})$, which determines the block count $K$,
the block start times $t_1, \dots, t_K$, and the block lengths
$L_1, \dots, L_K$. By construction, the signs $\sigma_1, \dots, \sigma_K$
are independent of $\mathcal{G}$ and of each other. Within block $j$, every
round $t \in \{t_j, \dots, t_{j+1} - 1\}$ satisfies $y_t = a_j$, so
\[
  \sum_{t = t_j}^{t_{j+1} - 1} \ell_t(y_t)
  \;=\; L_j \!\left( \tfrac{1 - \sigma_j}{2}\, a_{j,1}
                  + \tfrac{1 + \sigma_j}{2}\, a_{j,2} \right)
  \;=\; L_j \!\left( \tfrac{1}{2}
                  + \tfrac{a_{j,2} - a_{j,1}}{2}\, \sigma_j \right)
\]
Conditioning on $\mathcal{G}$ and on $\sigma_{1:j-1}$, the action $a_j$ is
determined (up to learner randomness, which we average over), and $\sigma_j$
remains independent and uniform on $\{-1, +1\}$. Therefore
\[
  \mathbb{E}\!\left[ \sum_{t = t_j}^{t_{j+1} - 1} \ell_t(y_t)
                    \,\Big|\, \mathcal{G}, \sigma_{1:j-1} \right]
  \;=\; \frac{L_j}{2}
  \;+\; \frac{a_{j,2} - a_{j,1}}{2}\, L_j\, \mathbb{E}[\sigma_j \mid \mathcal{G}, \sigma_{1:j-1}]
  \;=\; \frac{L_j}{2}
\]
Since the right-hand side is $\mathcal{G}$-measurable, integrating out
$\sigma_{1:j-1}$ leaves it unchanged. Summing over $j \in [K]$ then gives,
\[
  \mathbb{E} \left[\sum_{t=1}^T \ell_t(y_t)\right] = \mathbb{E}\left[\mathbb{E}\!\left[ \sum_{t=1}^T \ell_t(y_t) \,\Big|\, \mathcal{G} \right]\right]
  \;=\; \mathbb{E}\left[\sum_{j=1}^K \frac{L_j}{2}\right]
  \;=\; \frac{T}{2}
\]
Thus, any learner incurs loss $T/2$ in expectation. To bound the regret, we proceed with evaluating the loss of a comparator. Define
\[
  S := \sum_{j=1}^K L_j\sigma_j
\]
For any fixed comparator $u=(u_1,u_2)\in\Y$, where $u_2=1-u_1$,
\[
  \sum_{t=1}^T \ell_t(u)
  =
  \sum_{j=1}^K L_j\left(\frac{1-\sigma_j}{2}u_1 + \frac{1+\sigma_j}{2}u_2\right)
  =
  \frac{T+S}{2} - u_1S
\]
Hence the best static comparator is $u^\star=(1,0)$ when $S\ge 0$ and $u^\star=(0,1)$ when $S<0$, so
\[
  \min_{u\in\Y} \sum_{t=1}^T \ell_t(u)
  =
  \frac{T}{2} - \frac{|S|}{2}
\]
Therefore, the expected regret is
\[
  \mathbb E\!\left[\sum_{t=1}^T \ell_t(y_t) - \min_{u\in\Y}\sum_{t=1}^T \ell_t(u)\right]
  =
  \frac{1}{2}\mathbb E|S|
\]

It remains to lower bound $\mathbb E|S|$. Conditional on the block lengths $L_{1:K}$, $S$ is a Rademacher sum with
\[
  \mathbb E[S^2 \mid L_{1:K}] = \sum_{j=1}^K L_j^2 =: V
\]
and
\[
  \mathbb E[S^4 \mid L_{1:K}]
  =
  3V^2 - 2\sum_{j=1}^K L_j^4
  \le 3V^2
\]
Applying Paley--Zygmund to $S^2$ with $\theta = 1/2$ yields
\[
  \mathbb P\!\left(S^2 \ge \frac{V}{2} \middle| L_{1:K}\right)
  \ge
  (1/2)^2\cdot\frac{V^2}{\mathbb E[S^4|L_{1:K}]}
  \ge
  \frac{1}{4}\cdot\frac{V^2}{3V^2}
  = \frac{1}{12}
\]
Now on the event $\{S^2 \geq V/2\}$, we have $\abs{S} \geq \sqrt{V/2}$ so
\[
  \mathbb E[|S| \mid L_{1:K}]
  \ge
  \mathbb E[|S| \cdot \1_{\{S^2 \geq V/2\}} \mid L_{1:K}] \geq \mathbb P(S^2 \ge V/2\mid L_{1:K})\sqrt{\frac{V}{2}} \geq  
  \frac{1}{12}\sqrt{\frac{V}{2}}
\]

Now we have a lower bound on the conditional expectation $\mathbb E[|S| \mid L_{1:K}]$. To turn this into a bound on the unconditional expectation $\mathbb{E}|S|$, we restrict to an event on which $V$ is easy to control. To start, note that $\mathbb E[N]=(T-1)\mu\le \mu T$, so Markov's inequality gives $\mathbb P (N \geq 2\mu T) \leq 1/2$, which implies
\[
  \mathbb P(N\le 2\mu T) \ge \frac{1}{2}
\]
Equivalently, the event $\{K \leq 2\mu T +1\}$ has probability at least 1/2. Because $T\ge \mu^{-1}$, then the event $E := \{K \le 3\mu T\}$ also has probability at least $1/2$. On $E$, Cauchy--Schwarz gives
\[
  V = \sum_{j=1}^K L_j^2 \ge \frac{T^2}{K} \ge \frac{T}{3\mu}
\]
Therefore
\[
  \mathbb E|S|
  =
  \mathbb E\!\left[\mathbb E[\abs{S}\mid L_{1:K}]\right]
  \ge
  \frac{1}{12\sqrt2}\,\mathbb E[\sqrt{V}]
  \ge
  \frac{1}{12\sqrt2}\,\mathbb E[\sqrt{V}\,\1_E]
  \ge
  \frac{1}{24\sqrt6}\sqrt{\frac{T}{\mu}}
\]
Substituting into the regret identity proves the claim with $c=1/(48\sqrt6)$.
\end{proof}

Theorem~\ref{thm:convex_uppd_lower_bound} already applies on a line segment, namely the one-dimensional simplex $\{y\in\R_+^2 \mid y_1+y_2=1\}$. We use this geometry for a reason: on the interval $[0,1]$, a no-demand round only enforces $y_{t+1}\ge y_t$, whereas on the simplex it forces equality and hence true frozen blocks. The theorem is also genuinely different from Theorem~5 of \citet{ichikawa2026nonstationary}. Their construction fixes deterministic cycles of length $L_{\max}$ and proves an $\Omega(\sqrt{L_{\max}T})$ barrier through that fixed-cycle geometry. Here the cycle lengths are random and induced by Bernoulli clearings, and the proof uses the random signed block sum $S$ tied directly to UPPD. Appendix~\ref{app:lmax} explains how UPPD relates to the sell-out window $L_{\max}$; the two lower bounds address different primitives.

\subsection{Strongly convex losses}
\label{appendix:strong_convex_losses_proofs}

\begin{lemma}[Per-step inequality under strong convexity]
\label{lem:per_step_sc}
Suppose $\ell_t$ is $\alpha$-strongly convex on $\Y$ with respect to $\norm{\cdot}$. For $y_t = \Projnorm{J_t}{z_t}$, $q_t = \norm{z_t - y_t}$, and every $u \in \Y$,
\begin{equation}
\label{eq:per_step_sc}
  \ell_t(y_t) - \ell_t(u) \le \ip{g_t}{z_t - u} + G_*q_t - \frac{\alpha}{4}\norm{z_t - u}^2 + \frac{\alpha}{2}q_t^2
\end{equation}
\end{lemma}

\begin{proof}
By $\alpha$-strong convexity of $\ell_t$ at $y_t$, 
\[
\ell_t(y_t) - \ell_t(u) \le \ip{g_t}{y_t - u} - \frac{\alpha}{2}\norm{y_t - u}^2
\]
Then we may split the inner product $\ip{g_t}{y_t - u} = \ip{g_t}{z_t - u} + \ip{g_t}{y_t - z_t}$. Then by the generalized Hölder's inequality, we may bound 
\[
\ip{g_t}{y_t - z_t} \le \dnorm{g_t}\norm{y_t - z_t} \le G_*q_t
\]
where the last inequality used the subgradient dual bound. The triangle inequality combined with $(a+b)^2 \le 2a^2 + 2b^2$ gives 
\[
\norm{z_t - u}^2 \le 2\norm{y_t - u}^2 + 2q_t^2
\]
Hence $-\frac{\alpha}{2}\norm{y_t - u}^2 \le -\frac{\alpha}{4}\norm{z_t - u}^2 + \frac{\alpha}{2}q_t^2$. Combining the previous inequalities gives the claim.
\end{proof}

\begin{proof}[Proof of Corollary~\ref{cor:strong_convex_reduction}]
Summing~\eqref{eq:per_step_sc} over $t = 1, \ldots, T$ gives
\[
  R_T(u_{1:T}) \le \sum_{t=1}^T \Bigl(\ip{g_t}{z_t - u_t} - \frac{\alpha}{4}\norm{z_t - u_t}^2\Bigr) + G_*\tarimpgap + \frac{\alpha}{2}\sum_{t=1}^T q_t^2
\]
The first term is at most $\mathcal{R}_T^{\textnormal{sc}}(u_{1:T}; \alpha)$ by~\eqref{eq:sc_regret_condition}. Since $q_t \le D$ by Assumption~\ref{ass:convex_and_bounded}, then
\[
\sum_{t=1}^T q_t^2 \le D \sum_{t=1}^T q_t = D\tarimpgap
\]
Combining, we rewrite the regret bound as:

\[
R_T(u_{1:T}) \le \mathcal{R}_T^{\textnormal{sc}}(u_{1:T}; \alpha) + \left(G_* + \frac{\alpha D}{2}\right)\tarimpgap
\]
Then \eqref{eq:tc-uppd} of Theorem~\ref{thm:uppd_reduction} applies under Assumption~\ref{ass:base_learner}(ii) and gives $\tarimpgap \le B_\delta \tarmvment$ with probability at least $1-\delta$. Applying this bound yields~\eqref{eq:sc_reduction_bound}.
\end{proof}

\begin{proposition}[HT-OGD base learner under strong convexity]
\label{prop:ht_ogd_sc_base_learner}
Suppose Assumption~\ref{ass:convex_and_bounded} holds and each $\ell_t$ is $\alpha$-strongly convex on $\Y$ with respect to $\norm{\cdot}_2$. Fix $\delta \in (0, 1)$, set $B_\delta := \lceil \mu^{-1}\log(T/\delta) \rceil$, and define
\[
  s_\delta := \max\!\left\{1, \left\lceil \frac{B_\delta}{e^{\alpha\rho/(2G_2)} - 1}\right\rceil\right\}, \qquad \eta_t := \frac{2}{\alpha(t + s_\delta)}.
\]
Then HT-OGD (Algorithm~\ref{alg:ht_ogd}) initialized from any $z_1 \in J_1$ produces a target sequence $(z_t)_{t=1}^T \subset \Y$ satisfying Assumption~\ref{ass:base_learner}(ii) and the strongly convex linearized regret bound~\eqref{eq:sc_regret_condition} with
\begin{equation}
\label{eq:ht_ogd_sc_regret_functional}
  \mathcal{R}_T^{\textnormal{sc}}(u; \alpha) := \frac{\alpha s_\delta D_2^2}{4} + \frac{G_2^2}{\alpha}\log\!\Bigl(1 + \frac{T}{s_\delta}\Bigr).
\end{equation}
Moreover, the cumulative target movement is bounded by
\begin{equation}
\label{eq:ht_ogd_sc_movement_bound}
  \tarmvment \le \frac{2G_2}{\alpha} \log\!\Bigl(1 + \frac{T}{s_\delta}\Bigr).
\end{equation}
\end{proposition}

\begin{proof}
We first verify the windowed movement bound (Assumption~\ref{ass:base_learner}(ii)) and the cumulative movement bound~\eqref{eq:ht_ogd_sc_movement_bound}. By nonexpansiveness of the Euclidean projection,
\[
  \norm{z_{t+1} - z_t}_2 = \norm{\Proj_\Y^{\norm{\cdot}_2}(z_t - \eta_t g_t) - \Proj_\Y^{\norm{\cdot}_2}(z_t)}_2 \le \norm{\eta_t g_t}_2 \le \eta_t G_2 = \frac{2G_2}{\alpha(t + s_\delta)}.
\]

For any window of length $B_\delta$ starting at $t \in [T-B_\delta]$, using the bound $\sum_{n=1}^N \frac{1}{n + A} \le \log(1 + N/A)$ gives
\[
  \sum_{r=t}^{t+B_\delta-1} \norm{z_{r+1} - z_r}_2 \le \frac{2G_2}{\alpha} \sum_{r=t}^{t+B_\delta-1} \frac{1}{r+s_\delta} \le \frac{2G_2}{\alpha} \sum_{r=1}^{B_\delta} \frac{1}{r+s_\delta} \le \frac{2G_2}{\alpha} \log\!\Bigl(1 + \frac{B_\delta}{s_\delta}\Bigr) \le \rho,
\]
where the last inequality follows from the fact that $s_\delta \ge B_\delta/(e^{\alpha\rho/(2G_2)} - 1)$ and rearranging for $\rho$. This proves Assumption~\ref{ass:base_learner}(ii). Using the same argument but summing over $t \in [T-1]$ rather than a window of length $B_\delta$ proves~\eqref{eq:ht_ogd_sc_movement_bound}.

It remains to establish the strongly convex linearized regret bound~\eqref{eq:sc_regret_condition} with $\mathcal{R}_T^{\textnormal{sc}}(u; \alpha)$ defined in \eqref{eq:ht_ogd_sc_regret_functional}. By the OGD update and non-expansiveness of the Euclidean projection,
\[
  \norm{z_{t+1} - u}_2^2 \le \norm{z_t - \eta_t g_t - u}_2^2 = \norm{z_t - u}_2^2 - 2\eta_t \ip{g_t}{z_t - u} + \eta_t^2 \norm{g_t}_2^2,
\]
and by bounding $\norm{g_t}_2$ with its gradient bound, this rearranges to
\[
  \ip{g_t}{z_t - u} \le \frac{\norm{z_t - u}_2^2 - \norm{z_{t+1} - u}_2^2}{2\eta_t} + \frac{\eta_t G_2^2}{2}.
\]
It remains to bound the sum of $\ip{g_t}{z_t - u}$ over $t \in [T]$. By our choice of stepsize $\eta_t = 2/(\alpha(t+s_\delta))$, we have the identity
\[
  \frac{1}{2\eta_t} - \frac{1}{2\eta_{t-1}} = \frac{\alpha(t+s_\delta) - \alpha(t-1+s_\delta)}{4} = \frac{\alpha}{4},
\]
which lets us telescope the sum of the first term in the previous bound:
\begin{align*}
  \sum_{t=1}^T \frac{\norm{z_t - u}_2^2 - \norm{z_{t+1} - u}_2^2}{2\eta_t} & = \frac{\norm{z_1 - u}_2^2}{2\eta_1} - \frac{\norm{z_{T+1} - u}_2^2}{2\eta_T} + \sum_{t=2}^T \left(\frac{1}{2\eta_t} - \frac{1}{2\eta_{t-1}}\right)\norm{z_t - u}_2^2\\
  & = \frac{\norm{z_1 - u}_2^2}{2\eta_1} - \frac{\norm{z_{T+1} - u}_2^2}{2\eta_T} + \frac{\alpha}{4}\sum_{t=2}^T \norm{z_t - u}_2^2
\end{align*}
Splitting the leading coefficient as $1/(2\eta_1) = \alpha s_\delta/4 + \alpha/4$ lets us merge the leftover $\frac{\alpha}{4}\norm{z_1 - u}_2^2$ piece into the running sum:
\[
  \sum_{t=1}^T \frac{\norm{z_t - u}_2^2 - \norm{z_{t+1} - u}_2^2}{2\eta_t} = \frac{\alpha s_\delta}{4}\norm{z_1 - u}_2^2 + \frac{\alpha}{4}\sum_{t=1}^T \norm{z_t - u}_2^2 - \frac{\norm{z_{T+1} - u}_2^2}{2\eta_T}.
\]
Dropping the nonpositive last term and using $\norm{z_1 - u}_2 \le D_2$ from Assumption~\ref{ass:convex_and_bounded},
\[
  \sum_{t=1}^T \frac{\norm{z_t - u}_2^2 - \norm{z_{t+1} - u}_2^2}{2\eta_t} \le \frac{\alpha s_\delta D_2^2}{4} + \frac{\alpha}{4}\sum_{t=1}^T \norm{z_t - u}_2^2.
\]
Combining with the per-step OGD inequality and rearranging,
\[
  \sum_{t=1}^T \ip{g_t}{z_t - u} - \frac{\alpha}{4}\sum_{t=1}^T \norm{z_t - u}_2^2 \le \frac{\alpha s_\delta D_2^2}{4} + \frac{G_2^2}{2}\sum_{t=1}^T \eta_t.
\]
The stepsize sum satisfies $\sum_{t=1}^T \eta_t = (2/\alpha)\sum_{t=1}^T 1/(t+s_\delta) \le (2/\alpha)\log(1 + T/s_\delta)$. Substituting yields~\eqref{eq:sc_regret_condition} with $\mathcal{R}_T^{\textnormal{sc}}$ as in~\eqref{eq:ht_ogd_sc_regret_functional}.
\end{proof}

\begin{proof}[Proof of Theorem~\ref{thm:ht_ogd_strong_uppd}]
\label{proof:ht_ogd_strong_uppd}
By Proposition~\ref{prop:ht_ogd_sc_base_learner}, HT-OGD with stepsize $\eta_t = 2/(\alpha(t+s_\delta))$ satisfies Assumption~\ref{ass:base_learner}(ii) and the strongly convex linearized regret bound~\eqref{eq:sc_regret_condition} with $\mathcal{R}_T^{\textnormal{sc}}(u; \alpha)$ given in~\eqref{eq:ht_ogd_sc_regret_functional}, and the cumulative target movement is bounded by~\eqref{eq:ht_ogd_sc_movement_bound}. Applying Corollary~\ref{cor:strong_convex_reduction} yields, with probability at least $1-\delta$,
\begin{align*}
  R_T(u) & \le \frac{\alpha s_\delta D_2^2}{4} + \frac{G_2^2}{\alpha}\log\!\Bigl(1 + \frac{T}{s_\delta}\Bigr) + \Bigl(G_2 + \frac{\alpha D_2}{2}\Bigr) B_\delta \cdot \frac{2G_2}{\alpha}\log\!\Bigl(1 + \frac{T}{s_\delta}\Bigr)\\
  & = \frac{\alpha s_\delta D_2^2}{4} + G_2\left(\frac{G_2}{\alpha}(1+2B_\delta) + D_2B_\delta\right)\log\!\Bigl(1 + \frac{T}{s_\delta}\Bigr)\\
  & \le \frac{\alpha s_\delta D_2^2}{4} + G_2B_\delta \left(\frac{3G_2}{\alpha} + D_2\right)\log\!\Bigl(1 + \frac{T}{s_\delta}\Bigr)\\
  & = O\!\Bigl(s_\delta + B_\delta\log(1+T/s_\delta)\Bigr)
\end{align*}
Since $s_\delta = O(B_\delta)$ and $B_\delta = \Theta(\mu^{-1}\log(T/\delta))$, this simplifies to $R_T(u) = O\bigl((\log(T/\delta)\log T)/\mu\bigr)$, as claimed.
\end{proof}

\subsection{Dynamic regret}

Throughout, we extend the comparator sequence as necessary by setting $u_{T+1} := u_T$.

\begin{proposition}[Constant-stepsize OGD verifies the base learner conditions for dynamic comparators]
\label{prop:ogd_dynamic_base_learner}
Take the projection norm to be Euclidean, $\norm{\cdot}=\norm{\cdot}_2$. Fix $\delta\in(0,1)$, set $B_\delta:=\lceil\mu^{-1}\log(T/\delta)\rceil$, and choose a constant stepsize $\eta\in(0,\rho/(B_\delta G_2)]$. Run HT-OGD (Algorithm~\ref{alg:ht_ogd}) with $\eta_t\equiv\eta$ from $z_1\in J_1$. Then for every comparator sequence $u_{1:T}\in\Y^T$ with path variation $P_{T,2}$, the target sequence $(z_t)_{t=1}^T\subset\Y$ satisfies Assumption~\ref{ass:base_learner}:
\begin{enumerate}[label=(\roman*),leftmargin=2em,itemsep=2pt,topsep=2pt]
  \item Linearized regret bound holds with
  \[
    \mathcal{R}_T(u_{1:T}) := \frac{D_2^2+2D_2 P_{T,2}}{2\eta} + \frac{\eta G_2^2 T}{2}.
  \]
  \item Windowed movement bound holds: for every $t$ with $1\le t\le T-B_\delta$,
  \[
    \sum_{r=t}^{t+B_\delta-1}\norm{z_{r+1}-z_r}_2 \le \rho.
  \]
\end{enumerate}
Moreover, the cumulative target movement satisfies
\begin{equation}
\label{eq:ht_ogd_dynamic_movement_bound}
  \tarmvment \le \eta G_2 T.
\end{equation}
\end{proposition}

\begin{proof}
\textbf{Movement and part (ii).} By nonexpansiveness of the $\ell_2$ projection and the hypothesis $\eta\le\rho/(B_\delta G_2)$, for every $t$,
\[
  \norm{z_{t+1}-z_t}_2
  = \norm{\Proj_{\Y}^{\norm{\cdot}_2}(z_t-\eta g_t)-\Proj_{\Y}^{\norm{\cdot}_2}(z_t)}_2
  \le \eta\norm{g_t}_2
  \le \eta G_2.
\]
Therefore for every $t$ with $1\le t\le T-B_\delta$,
\[
  \sum_{r=t}^{t+B_\delta-1}\norm{z_{r+1}-z_r}_2 \le B_\delta\,\eta G_2 \le \rho,
\]
which is part (ii). Summing $\norm{z_{t+1}-z_t}_2\le\eta G_2$ over $t \in [T-1]$ gives \eqref{eq:ht_ogd_dynamic_movement_bound}.

\textbf{Part (i).} By the one-step projection argument used in the proof of Lemma~\ref{lem:ht_ogd_proj_grad_error}, applied with comparator $u_t$, the projected OGD update with constant stepsize gives
\[
  \ip{g_t}{z_t-u_t}
  \le
  \frac{\norm{z_t-u_t}_2^2-\norm{z_{t+1}-u_t}_2^2}{2\eta}
  +
  \frac{\eta G_2^2}{2}.
\]
By inserting and subtracting $\tfrac{1}{2\eta}\norm{z_{t+1}-u_{t+1}}_2^2$ then summing over $t$, we may telescope the series:
\begin{align*}
    \sum_{t=1}^T \ip{g_t}{z_t-u_t} & \leq \frac{1}{2\eta} \sum_{t=1}^T \left(\norm{z_{t+1} - u_{t+1}}_2^2 - \norm{z_{t+1} - u_t}_2^2\right)\\
    & \quad + \frac{1}{2\eta} \sum_{t=1}^T  \left(\norm{z_t - u_t}_2^2 - \norm{z_{t+1} - u_{t+1}}_2^2\right) + \frac{\eta G_2^2 T}{2}\\
    & = \frac{1}{2\eta} \sum_{t=1}^T \left(\norm{z_{t+1} - u_{t+1}}_2^2 - \norm{z_{t+1} - u_t}_2^2\right)\\
    & \quad + \frac{1}{2\eta}\left(\norm{z_1-u_1}_2^2 - \norm{z_{T+1} - u_{T+1}}_2^2\right) + \frac{\eta G_2^2 T}{2}\\
    & \leq \frac{1}{2\eta}\norm{z_1-u_1}_2^2 + \frac{1}{2\eta} \sum_{t=1}^T \left(\norm{z_{t+1} - u_{t+1}}_2^2 - \norm{z_{t+1} - u_t}_2^2\right) + \frac{\eta G_2^2 T}{2}.
\end{align*}
Since $\norm{z_1-u_1}_2\le D_2$ and
\[
  \norm{z_{t+1}-u_{t+1}}_2^2-\norm{z_{t+1}-u_t}_2^2 \le 2 D_2\,\norm{u_{t+1}-u_t}_2
\]
using the reverse triangle inequality and the diameter bound on $\Y$, we obtain
\[
  \baseregret(u_{1:T}) = \sum_{t=1}^T \ip{g_t}{z_t-u_t} \le \frac{D_2^2+2 D_2 P_{T,2}}{2\eta} + \frac{\eta G_2^2 T}{2} = \mathcal{R}_T(u_{1:T}).
\]
\end{proof}

\begin{proof}[Proof of Theorem~\ref{thm:ht_ogd_dynamic_uppd}]
By Proposition~\ref{prop:ogd_dynamic_base_learner}, the OGD target sequence satisfies Assumption~\ref{ass:base_learner} with $\mathcal{R}_T(u_{1:T}) = (D_2^2 + 2 D_2 P_{T,2})/(2\eta) + \eta G_2^2 T/2$ and $\tarmvment\le \eta G_2 T$ via \eqref{eq:ht_ogd_dynamic_movement_bound}. Theorem~\ref{thm:uppd_reduction} then gives with probability at least $1-\delta$ for every comparator sequence $u_{1:T}\in\Y^T$,
\[
  R_T(u_{1:T}) \le \mathcal{R}_T(u_{1:T}) + G_2 B_\delta\,\tarmvment
  \le \frac{D_2^2+2 D_2 P_{T,2}}{2\eta} + \Bigl(B_\delta+\frac{1}{2}\Bigr)\eta G_2^2 T.
\]
Balancing the two terms, the unconstrained minimizer is $\eta^*=\tfrac{1}{G_2}\sqrt{\tfrac{D_2(D_2+2P_{T,2})}{(2B_\delta+1)T}}$. Provided $\eta^*\le\rho/(B_\delta G_2)$, substituting $\eta^*$ gives the stated bound
\[
  R_T(u_{1:T}) \le G_2\sqrt{(2B_\delta+1)\,D_2(D_2+2P_{T,2})\,T}.
\]
\end{proof}

\subsubsection{The SOGD base learner and proof of Theorem~\ref{thm:ht_sogd_regret_uppd}}
\label{sec:sogd_proof}

The SOGD base learner is the smoothed OCO algorithm of \citet{zhang2022soco}. It runs a grid of OGD experts with different time scales and combines them with a meta-learner that penalizes switching. \citet{ichikawa2026nonstationary} used this learner for hidden-target inventory under a single linear capacity constraint. We use the same base learning primitive, but the inventory layer is different: the queue reduction only asks for a regret-plus-switching-cost bound and a local movement bound.

\begin{algorithm}[H]
\caption{Hidden-target method with an SOGD base learner}
\label{alg:sogd}
\KwInput{Domain $\Y$, switching parameter $\lambda>0$, horizon $T$, initial target $z_1\in J_1$}
Initialize the SOGD base learner of \citet{zhang2022soco} on $\Y$ at $z_1$\;
\For{$t=1,2,\dots,T$}{
  Observe $x_t$ and define $J_t=\Y\cap\{y\in\R^n\mid y\succeq x_t\}$\;
  Let $z_t$ be the current target produced by SOGD and implement $y_t=\Proj_{J_t}^{\norm{\cdot}_2}(z_t)$\;
  Observe $g_t\in\partial\ell_t(y_t)$ and the next state $x_{t+1}$\;
  Feed the shifted linear loss $\widetilde f_t(z)=\ip{g_t}{z}-\min_{w\in\Y}\ip{g_t}{w}$ to SOGD and receive $z_{t+1}$\;
}
\end{algorithm}

The shifted loss is nonnegative and has the same linear regret difference, since $\widetilde f_t(z_t)-\widetilde f_t(u_t)=\ip{g_t}{z_t-u_t}$ and is required by the SOGD algorithm. The hidden-target reduction (Theorem~\ref{thm:uppd_reduction}) does not require a separate abstract interface for switching-aware base learners. As we show below, the native SOGD bound can be rearranged into the form of Assumption~\ref{ass:base_learner}(i), with the switching-cost coefficient producing exactly the cancellation needed in the reduction.

We now verify that SOGD satisfies Assumption~\ref{ass:base_learner}. The base learner architecture again follows \citet{ichikawa2026nonstationary}: a grid of OGD experts is combined by a switching-aware smoothed OGD meta-learner. What must be re-established here is that this learner satisfies not only a switching-aware regret bound, but also the local hidden target movement condition required by the queue reduction on arbitrary convex sets.

For the analysis, we use the concrete SOGD parameterization of \citet{zhang2022soco} with the projection norm taken to be Euclidean: let
\[
  K := \Bigl\lfloor \log_2\!\Bigl(\frac{T}{32\max\{\lambda,1\}\log T}\Bigr) \Bigr\rfloor + 1,
  \qquad
  n^{(k)} := T2^{1-k},
  \qquad
  \eta^{(k)} := \frac{D_2}{G_2\sqrt{(1+2\lambda)n^{(k)}}}
\]
For each $k \in [K]$, we run OGD expert $A_k$ with step size $\eta^{(k)}$ and aggregate the experts with the discounted-normal-predictor combiner using $M=2$ and $Z=1/T$. With an additional assumption on the size of the horizon $T$, we prove that SOGD verifies the base learner conditions:

\begin{proposition}
\label{prop:dynamic_regret_sogd_transfer}
Assume $\lambda \ge 1$, $T \ge \max\{32\lambda\log T,e\}$, and run Algorithm~\ref{alg:sogd} with the SOGD parameterization above. Then the target sequence $(z_t)_{t=1}^T$ satisfies the per-step movement bound
\begin{equation}
\label{eq:sogd_local_tgt_movement}
  \norm{z_{t+1}-z_t}_2 \le \frac{2D_2}{\lambda}
  \qquad \text{for all } t \in [T-1],
\end{equation}
and the linearized regret bound
\begin{align}
\label{eq:sogd_regret_interface}
\begin{split}
  \baseregret(u_{1:T}) \le \mathcal R_T^{\mathrm{SOGD}}(u_{1:T};\lambda) - G_2\lambda\tarmvment
\end{split}
\end{align}
holds for every comparator sequence $u_{1:T}\in\Y^T$ with path variation $P_{T,2}$, where $\mathcal R_T^{\mathrm{SOGD}}(u_{1:T};\lambda)$ is defined by \eqref{eq:sogd_native_regret}.
\end{proposition}

\begin{proof}[Proof of Proposition~\ref{prop:dynamic_regret_sogd_transfer}]
\label{proof:dynamic_regret_sogd_transfer}
Rearranging \eqref{eq:sogd_regret_interface} into the equivalent form $\baseregret(u_{1:T}) + G_2\lambda\,\tarmvment \le \mathcal R_T^{\mathrm{SOGD}}(u_{1:T};\lambda)$, this is the $\lambda\ge1$ specialization of Theorem~4 of \citet{zhang2022soco} on the interval $[1,T]$, with $G=G_2$, $D=D_2$, and $P_{1,T}=P_{T,2}$. The source theorem's second term is $120GD\max\{\sqrt\lambda,1\}\sqrt{T(1+2P_{1,T}/D)\log T}$, which equals the second term in \eqref{eq:sogd_native_regret} because $\lambda\ge1$. The theorem applies since each $\widetilde f_t$ is convex, takes values in $[0,G_2 D_2]$ on $\Y$, has gradient norm at most $G_2$, and satisfies
\[
  \widetilde f_t(z_t)-\widetilde f_t(u_t)=\ip{g_t}{z_t-u_t}.
\]

It remains to prove the local movement bound \eqref{eq:sogd_local_tgt_movement}. Following the framework of Algorithm~4 in \citet{zhang2022soco}, write $a_t^{(k)}$ for the output of expert $A_k$, $v_t^{(k)}$ for the output of the $k$th aggregate, and $\omega_t^{(k)}\in[0,1]$ for the combiner weight that forms $v_t^{(k)}$ from $v_t^{(k-1)}$ and $a_t^{(k)}$. Thus $z_t=v_t^{(K)}$. The proof uses two facts: each OGD expert moves by at most its step size times $G_2$, and each combiner weight changes slowly by Eq.~(67) of \citet{zhang2022soco}.

For the experts, nonexpansiveness of Euclidean projection gives
\[
  m_k
  :=
  \sup_t \norm{a_{t+1}^{(k)}-a_t^{(k)}}_2
  \le
  \eta^{(k)} G_2
  =
  \frac{D_2}{\sqrt{(1+2\lambda)n^{(k)}}}.
\]
For the combiner weights, Eq.~(67) of \citet{zhang2022soco} with $Z=1/T$ and $\lambda\ge1$ yields
\[
  \kappa_k
  :=
  \sup_t \abs{\omega_{t+1}^{(k)}-\omega_t^{(k)}}
  \le
  \frac{1}{\sqrt{\lambda}}\left(
  \sqrt{\frac{\log T}{n^{(k)}}}+ \frac{1}{4T}\right)
\]
Since $n^{(k)}\le T$ and $T\ge e$, we have
\[
  \frac{1}{4T}
  \le
  \frac{1}{4\sqrt{T}}
  \le
  \frac{1}{4}\sqrt{\frac{\log T}{n^{(k)}}}
\]
so we may loosen the bound on $\kappa_k$ as:
\[
  \kappa_k
  \le
  \frac{1}{\sqrt{\lambda}}\left(\sqrt{\frac{\log T}{n^{(k)}}} + \frac{1}{4}\sqrt{\frac{\log T}{n^{(k)}}}\right)
  \le
  \frac{5}{4\sqrt{\lambda}}
  \sqrt{\frac{\log T}{n^{(k)}}}
\]

Now define
\[
  M_k := \sup_t \norm{v_{t+1}^{(k)}-v_t^{(k)}}_2
\]
Since $v_t^{(1)}=a_t^{(1)}$, we have $M_1=m_1$. For $k\ge2$,
\[
  v_t^{(k)}=(1-\omega_t^{(k)})v_t^{(k-1)}+\omega_t^{(k)}a_t^{(k)}
\]
Taking the difference of $v_{t+1}^{(k)}$ and $v_t^{(k)}$, we have
\[
v_{t+1}^{(k)} - v_{t}^{(k)} = \left((1-\omega_{t+1}^{(k)})v_{t+1}^{(k-1)}+\omega_{t+1}^{(k)}a_{t+1}^{(k)}\right) - \left((1-\omega_t^{(k)})v_t^{(k-1)}+\omega_t^{(k)}a_t^{(k)}\right)
\]
Add and subtract $(1-\omega_{t+1}^{(k)})v_t^{(k-1)}$ and $\omega_{t+1}^{(k)}a_t^{(k)}$. Then using triangle inequality,
\begin{align*}
  \norm{v_{t+1}^{(k)}-v_t^{(k)}}_2
  \le{}&
  (1-\omega_{t+1}^{(k)})\norm{v_{t+1}^{(k-1)}-v_t^{(k-1)}}_2 \\
  &\;+ \omega_{t+1}^{(k)}\norm{a_{t+1}^{(k)}-a_t^{(k)}}_2 \\
  &\;+ \abs{\omega_{t+1}^{(k)}-\omega_t^{(k)}}\,
  \norm{a_t^{(k)}-v_t^{(k-1)}}_2
\end{align*}
Both $a_t^{(k)}$ and $v_t^{(k-1)}$ belong to $\Y$, whose diameter is $D_2$, and since the weights $\omega^{(k)}_t, \omega^{(k)}_{t+1} \in [0, 1]$, we can bound these distances as
\[
  M_k \le M_{k-1}+m_k+D_2\kappa_k
\]
Summing this recursion from $k=2$ to $k=K$ telescopes the $M_k$ terms and gives
\[
  M_K \le \sum_{k=1}^K m_k + D_2\sum_{k=2}^K \kappa_k
\]

where the sum for the $m_k$ terms starts at 1 since $M_1 = m_1$. Now, we will bound the sum $\sum_{k=1}^K \frac{1}{\sqrt{n^{(k)}}}$ since this will be used to bound each of the two sums above. Because $n^{(k)}=T2^{1-k}$ and our choice of $K$ is
\[
  K = \Bigl\lfloor \log_2\!\Bigl(\frac{T}{32\lambda\log T}\Bigr) \Bigr\rfloor + 1
\]
we have $n^{(K)} \ge 32\lambda\log T$. Therefore
\[
  \sum_{k=1}^K \frac{1}{\sqrt{n^{(k)}}}
  = \frac{1}{\sqrt{n^{(K)}}}\sum_{k=0}^{K-1} 2^{-k/2}
  \le
  \frac{1}{\sqrt{n^{(K)}}}\sum_{k=0}^\infty 2^{-k/2}
  =
  \frac{2+\sqrt{2}}{\sqrt{n^{(K)}}}
  \le
  \frac{1}{\sqrt{2\lambda\log T}}
\]
Consequently,
\[
  \sum_{k=1}^K m_k
  \le
  \frac{D_2}{\sqrt{1+2\lambda}}
  \sum_{k=1}^K \frac{1}{\sqrt{n^{(k)}}}
  \le
  \frac{D_2}{\sqrt{1+2\lambda}} \cdot \frac{1}{\sqrt{2\lambda \log T}}
  \leq
  \frac{D_2}{2\lambda\sqrt{\log T}}
  \le
  \frac{D_2}{2\lambda}
\]
where we used $\sqrt{2\lambda} \leq \sqrt{2\lambda +1}$ and $T\ge e$. Similarly,
\[
  D_2\sum_{k=2}^K \kappa_k
  \le
  \frac{5D_2\sqrt{\log T}}{4\sqrt{\lambda}}
  \sum_{k=2}^K \frac{1}{\sqrt{n^{(k)}}}
  \le
  \frac{5D_2\sqrt{\log T}}{4\sqrt{\lambda}} \cdot \frac{1}{\sqrt{2\lambda\log T}}
  = \frac{5D_2}{4\lambda\sqrt{2}}
  \leq
  \frac{D_2}{\lambda}
\]
Combining these bounds, we have $M_K\le 2D_2/\lambda$. Since $z_t=v_t^{(K)}$, this proves \eqref{eq:sogd_local_tgt_movement}.
\end{proof}

This proposition lets us show that the SOGD instantiation satisfies Assumption~\ref{ass:base_learner} and invoke the hidden-target reduction.

\begin{proof}[Proof of Theorem~\ref{thm:ht_sogd_regret_uppd}]
Set $\lambda=\lambda_\delta$. Since $\lambda_\delta\ge B_\delta\ge 1$ and $T\ge\max\{32\lambda_\delta\log T,e\}$ by assumption, Proposition~\ref{prop:dynamic_regret_sogd_transfer} applies. The local movement bound \eqref{eq:sogd_local_tgt_movement} gives, for every $t$ with $1\le t\le T-B_\delta$,
\[
  \sum_{r=t}^{t+B_\delta-1}\norm{z_{r+1}-z_r}_2 \le B_\delta\cdot\frac{2D_2}{\lambda_\delta} \le \rho,
\]
because $\lambda_\delta\ge 2 D_2 B_\delta/\rho$ by construction. This verifies Assumption~\ref{ass:base_learner}(ii), and the linearized regret bound \eqref{eq:sogd_regret_interface} verifies Assumption~\ref{ass:base_learner}(i) with $\mathcal R_T(u_{1:T}) = \mathcal R_T^{\mathrm{SOGD}}(u_{1:T};\lambda_\delta) - G_2\lambda_\delta\,\tarmvment$. Applying Theorem~\ref{thm:uppd_reduction}, with probability at least $1-\delta$,
\begin{align*}
  R_T(u_{1:T})
  & \le \mathcal R_T^{\mathrm{SOGD}}(u_{1:T};\lambda_\delta) - G_2\lambda_\delta\tarmvment + G_2 B_\delta\tarmvment\\
  & = \mathcal R_T^{\mathrm{SOGD}}(u_{1:T};\lambda_\delta) - G_2(\lambda_\delta - B_\delta)\tarmvment\\
  & \le \mathcal R_T^{\mathrm{SOGD}}(u_{1:T};\lambda_\delta)\\
\end{align*}
since $\lambda_\delta\ge B_\delta$, proving the regret bound \eqref{eq:sogd_regret_bound_uppd}.

It remains to justify the explicit simplification in
\eqref{eq:sogd_regret_bound_uppd_explicit}. Since
$\lambda_\delta\ge 1$, $\log T\ge 1$, and
$1+\lambda_\delta\le 2\lambda_\delta$, the native SOGD bound gives
\[
\begin{split}
  \mathcal R_T^{\mathrm{SOGD}}(u_{1:T};\lambda_\delta)
  &\le
  2G_2D_2
  \sqrt{2\lambda_\delta T
  \Bigl(1+\frac{2P_{T,2}}{D_2}\Bigr)\log T}  \\
  &\qquad
  +120G_2D_2
  \sqrt{\lambda_\delta T
  \Bigl(1+\frac{2P_{T,2}}{D_2}\Bigr)\log T} \\
  &=
  (120+2\sqrt2)G_2D_2
  \sqrt{\lambda_\delta T
  \Bigl(1+\frac{2P_{T,2}}{D_2}\Bigr)\log T}.
\end{split}
\]
We now remove the ceilings. Since $\mu\le1$, $T\ge e$, and
$\delta\in(0,1)$,
\[
  \frac{1}{\mu}\log\!\Bigl(\frac{T}{\delta}\Bigr)\ge 1 .
\]
Using $\lceil x\rceil\le 2x$ for $x\ge1$, we get
\[
  B_\delta
  \le
  \frac{2}{\mu}\log\!\Bigl(\frac{T}{\delta}\Bigr),
\]
and hence
\[
\begin{split}
  \lambda_\delta
\le
  2B_\delta\max\!\Bigl\{1,\frac{2D_2}{\rho}\Bigr\}  
\le
  \frac{8}{\mu}
  \log\!\Bigl(\frac{T}{\delta}\Bigr)
  \max\!\Bigl\{1,\frac{D_2}{\rho}\Bigr\}.
\end{split}
\]
Also,
\[
  1+\frac{2P_{T,2}}{D_2}
  \le
  2\Bigl(1+\frac{P_{T,2}}{D_2}\Bigr).
\]
Substituting the last two displays into the preceding SOGD bound yields
\[
\begin{split}
  \mathcal R_T^{\mathrm{SOGD}}(u_{1:T};\lambda_\delta)
  &\le
  4(120+2\sqrt2)G_2D_2
  \sqrt{
  \Bigl(1+\frac{P_{T,2}}{D_2}\Bigr)
  \max\!\Bigl\{1,\frac{D_2}{\rho}\Bigr\}}
  \sqrt{
  \frac{T}{\mu}\log\!\Bigl(\frac{T}{\delta}\Bigr)\log T} \\
  &\le
  500\,G_2D_2
  \sqrt{
  \Bigl(1+\frac{P_{T,2}}{D_2}\Bigr)
  \max\!\Bigl\{1,\frac{D_2}{\rho}\Bigr\}}
  \sqrt{
  \frac{T}{\mu}\log\!\Bigl(\frac{T}{\delta}\Bigr)\log T},
\end{split}
\]
because $4(120+2\sqrt2)\le500$. This proves
\eqref{eq:sogd_regret_bound_uppd_explicit}.
\end{proof}

\subsection{Mirror descent base learners}
\label{app:beyond_euclidean_proofs}

This appendix proves Proposition~\ref{prop:md_base_learner} (the general mirror-descent base learner verification) and uses it to derive both the general OMD regret rate of Theorem~\ref{thm:md_oio_regret} and the entropic-OMD instantiation of Theorem~\ref{thm:ht_entropic_md}.

\subsubsection{General mirror descent}
\label{app:md_base_learner_proof}

\begin{proposition}[One-step mirror descent inequality]
Assume $z_t,z_{t+1}\in\mathcal D$ and that $R$ satisfies the domain, differentiability, and Bregman strong-convexity assumptions stated in Section~\ref{sec:beyond_euclidean}. Then under the OMD update \eqref{eq:md_update}, the following one-step inequality holds for every $u\in\Y$:
\begin{equation}
\label{eq:md_one_step}
  \ip{g_t}{z_t-u}\;\le\;\frac{D_R(u,z_t)-D_R(u,z_{t+1})}{\eta_t} + \frac{\eta_t\dnorm{g_t}^2}{2\sigma}
  .
\end{equation}
\end{proposition}

The proof is standard; we include it below for completeness.

\begin{proof}
The OMD objective $\phi(z) := \eta_t\ip{g_t}{z} + D_R(z, z_t)$ has gradient $\nabla\phi(z) = \eta_t g_t + \nabla R(z) - \nabla R(z_t)$. First-order optimality of $z_{t+1} = \arg\min_{z\in\Y}\phi(z)$ over the convex set $\Y$ gives $\ip{\nabla\phi(z_{t+1})}{u-z_{t+1}}\ge 0$ for every $u\in\Y$. Substituting the expression for $\nabla \phi(z)$ and rearranging,
\begin{equation}
\label{eq:md_proof_optimality}
  \eta_t\,\ip{g_t}{z_{t+1}-u} \;\le\; \ip{\nabla R(z_{t+1}) - \nabla R(z_t)}{u - z_{t+1}}.
\end{equation}

Before continuing, we prove an identity that for $a, b, c \in \Y$. Expanding the definition of $D_R$ we have:
\begin{align*}
  D_R(c, a) - D_R(c, b) & = -R(a) - \ip{\nabla R(a)}{c-a} + R(b) + \ip{\nabla R(b)}{c-b}\\
  D_R(b, a) & = R(b) - R(a) - \ip{\nabla R(a)}{b-a}.
\end{align*}
Subtracting the second equation from the first,
\[
D_R(c, a) - D_R(c, b) - D_R(b, a) = \ip{\nabla R(b)}{c-b} - \ip{\nabla R(a)}{c-a} + \ip{\nabla R(a)}{b-a}
\]
The last two terms combine into $\ip{\nabla R(a)}{b-c}$, allowing us to write:
\begin{equation}
  \label{eq:bregman_three_point}
\ip{\nabla R(b)-\nabla R(a)}{c-b} \;=\; D_R(c,a) - D_R(c,b) - D_R(b,a)
\end{equation}
Now, apply this identity with $a=z_t$, $b=z_{t+1}$, $c=u$ and substitute into \eqref{eq:md_proof_optimality}:
\begin{equation}
\label{eq:md_proof_postbregman}
  \eta_t\,\ip{g_t}{z_{t+1}-u} \;\le\; D_R(u,z_t) - D_R(u,z_{t+1}) - D_R(z_{t+1}, z_t).
\end{equation}
Rewriting the inner product in \eqref{eq:md_proof_postbregman} with $z_t-u = (z_t - z_{t+1}) + (z_{t+1}-u)$, we write:
\begin{equation}
\label{eq:md_proof_post_rewrite}
  \eta_t\ip{g_t}{z_t-u} \;\le\; \eta_t\ip{g_t}{z_t-z_{t+1}} - D_R(z_{t+1}, z_t) + D_R(u,z_t) - D_R(u,z_{t+1}).
\end{equation}
Now we bound the first two terms on the right hand side. By the generalized Hölder's inequality, $\ip{g_t}{z_t-z_{t+1}} \le \dnorm{g_t}\,\norm{z_t-z_{t+1}}$. Then by Young's inequality with $a=\eta_t\dnorm{g_t}$ and $b=\norm{z_t-z_{t+1}}$,
\[
  \eta_t\dnorm{g_t}\cdot\norm{z_t-z_{t+1}} \;\le\; \frac{\eta_t^2\dnorm{g_t}^2}{2\sigma} + \frac{\sigma}{2}\norm{z_t-z_{t+1}}^2
\]
Finally, by the Bregman strong-convexity assumption, $D_R(z_{t+1},z_t)\ge \tfrac{\sigma}{2}\norm{z_{t+1}-z_t}^2$. Then combining these two bounds, we have
\[
  \eta_t\ip{g_t}{z_t-z_{t+1}} - D_R(z_{t+1}, z_t) \le \frac{\eta_t^2\dnorm{g_t}^2}{2\sigma} + \frac{\sigma}{2}\norm{z_t-z_{t+1}}^2 - \frac{\sigma}{2}\norm{z_t-z_{t+1}}^2 = \frac{\eta_t^2\dnorm{g_t}^2}{2\sigma}
\]

Substituting back into \eqref{eq:md_proof_post_rewrite} and dividing by $\eta_t>0$ proves \eqref{eq:md_one_step}.
\end{proof}

Using the one-step inequality, we now verify that the OMD iterates satisfy the base learner conditions of Assumption~\ref{ass:base_learner}.

\begin{proposition}[Mirror descent verifies the base learner conditions]
\label{prop:md_base_learner}
Take an admissible norm $\norm{\cdot}$ with dual $\dnorm{\cdot}$ and a regularizer $R$ satisfying the domain, differentiability, and Bregman strong-convexity assumptions stated in Section~\ref{sec:beyond_euclidean}. Write $G_*:=\sup_t\dnorm{g_t}$. Fix $\delta\in(0,1)$, set $B_\delta:=\lceil\mu^{-1}\log(T/\delta)\rceil$, and pick a constant stepsize $\eta\in(0,\rho\sigma/(B_\delta G_*)]$. Run hidden-target projection with the OMD update \eqref{eq:md_update} from any $z_1\in J_1\cap\mathcal D$, and assume the update returns points in $\mathcal D$. Then the target sequence $(z_t)_{t=1}^T\subset\Y$ satisfies Assumption~\ref{ass:base_learner}:
\begin{enumerate}[label=(\roman*),leftmargin=2em,itemsep=2pt,topsep=2pt]
  \item Linearized regret bound holds with $\mathcal{R}_T(u):=D_R(u,z_1)/\eta + \eta T G_*^2/(2\sigma)$ for every $u\in\Y$.
  \item Windowed movement bound: for every $t$ with $1\le t\le T-B_\delta$,
  \[
    \sum_{r=t}^{t+B_\delta-1}\norm{z_{r+1}-z_r} \le B_\delta\,\eta G_*/\sigma \le \rho.
  \]
\end{enumerate}
Moreover, $\tarmvment\le\eta G_* T/\sigma$.
\end{proposition}

\begin{proof}[Proof of Proposition~\ref{prop:md_base_learner}]
\textbf{Part (i).} Summing \eqref{eq:md_one_step} over $t\in[T]$ with constant stepsize $\eta$, the Bregman terms telescope and since the Bregman divergence is always non-negative, $D_R(u,z_{T+1})\ge 0$ gives
\[
  \baseregret(u) = \sum_{t=1}^T \ip{g_t}{z_t-u}
  \;\le\;\frac{D_R(u,z_1)-D_R(u,z_{T+1})}{\eta} + \frac{\eta T G_*^2}{2\sigma}
  \;\le\;\frac{D_R(u,z_1)}{\eta} + \frac{\eta T G_*^2}{2\sigma}
\]

where the right-hand side is $\mathcal R_T(u)$, finishing the proof of part (i).

\textbf{Part (ii) and movement.} By strong convexity of $R$, we have $D_R(z_{t}, z_{t+1}) \geq \tfrac{\sigma}{2}\norm{z_{t+1}-z_t}^2$. Then note that applying the one-step inequality \eqref{eq:md_one_step} with $u=z_t$, we have $D_R(z_{t}, z_{t+1}) \leq \frac{\eta^2\norm{g_t}^2_*}{2\sigma}$. Combining, we have
\[
  \norm{z_{t+1}-z_t} \le \sqrt{\frac{2}{\sigma} D_R(z_t, z_{t+1})} \le \frac{\eta \norm{g_t}_*}{\sigma} \le \frac{\eta G_*}{\sigma}
\]

Summing over a window of $B_\delta$ consecutive rounds:
\[
  \sum_{r=t}^{t+B_\delta-1}\norm{z_{r+1}-z_r}\;\le\;\frac{B_\delta\,\eta G_*}{\sigma}\;\le\;\rho,
\]
using the stepsize hypothesis. Summing over $t \in [T-1]$ instead gives $\tarmvment\le\eta G_* T/\sigma$.
\end{proof}

\begin{proof}[Proof of Theorem~\ref{thm:md_oio_regret}]
If $G_*=0$, then all linearized losses are constant on $\Y$, the target movement bound is zero, and the claim is immediate. Assume $G_*>0$. Applying Theorem~\ref{thm:uppd_reduction} with $\mathcal{R}_T(u)$ and $\tarmvment$ from Proposition~\ref{prop:md_base_learner} gives
\begin{equation}
  \label{eq:md_base_regret_bound}
R_T(u)\le\frac{D_R(u,z_1)}{\eta} + \frac{\eta T G_*^2}{2\sigma} + G_* B_\delta\cdot\frac{\eta G_* T}{\sigma}
  \;=\;\frac{D_R(u,z_1)}{\eta} + \frac{\eta T G_*^2}{\sigma}\Bigl(\tfrac{1}{2} + B_\delta\Bigr).
\end{equation}

Bounding $D_R(u,z_1)\le\mathcal{B}_R$ gives \eqref{eq:md_master_bound_constrained} for every admissible stepsize. The right-hand side is convex in $\eta$, and its unconstrained minimizer is
\[
\eta^*=\sqrt{\frac{\mathcal{B}_R\sigma}{G_*^2 T(1/2+B_\delta)}}.
\]
The minimizer over the feasible interval $\eta\le \rho\sigma/(B_\delta G_*)$ is therefore $\min\{\eta^*,\rho\sigma/(B_\delta G_*)\}$. If $\eta^*$ satisfies the boundary constraint, substituting $\eta=\eta^*$ into \eqref{eq:md_master_bound_constrained} gives
\[
  R_T(u)\le2 G_*\sqrt{\frac{\mathcal{B}_RT(1/2+B_\delta)}{\sigma}}=\widetilde O\!\left(G_*\sqrt{\frac{\mathcal{B}_RT}{\sigma\mu}}\right),
\]
which is \eqref{eq:md_master_rate}.
\end{proof}

\subsubsection{Entropic OMD on bounded $\R_+^n$}
\label{app:entropic_md}

Take $\Y\subseteq\R_+^n$ bounded with $M:=\sup_{x\in\Y}\norm{x}_1$. The projection norm is $\ell_1$, the dual is $\ell_\infty$, and we write $G_\infty:=\sup_t\norm{g_t}_\infty$. The base learner is mirror descent with the negative-entropy regularizer $R(x):=\sum_{i=1}^n x_i\log x_i$, started from a strictly positive feasible target; Theorem~\ref{thm:ht_entropic_md} uses the uniform target and explicitly assumes it lies in $\Y$. The Bregman divergence of the negative-entropy regularizer $R$ on $\R_+^n$ is the \emph{unnormalized} Kullback--Leibler divergence:
\begin{equation}
\label{eq:unnormalized_kl}
  D_R(u, v) = \sum_{i=1}^n u_i \log\!\frac{u_i}{v_i} - \norm{u}_1 + \norm{v}_1
\end{equation}
which reduces to the classical Kullback-Leibler divergence $D_{KL}$ when $u$ and $v$ are probability distributions since we would have $\norm{u}_1 = \norm{v}_1 = 1$.

\paragraph{Strong convexity of the entropy on $\Y$.}
To prove $R$ is strongly convex with respect to $\ell_1$, we use an equivalent definition based on $R$'s Hessian. On the positive orthant, the Hessian of $R$ at $x\in\R_{++}^n$ is $\nabla^2 R(x)=\mathrm{diag}(1/x_i)$. By Cauchy--Schwarz, for any $v\in\R^n$ and $x\in\R_{++}^n$,
\[
  \norm{v}_1^2 =\left(\sum_{i=1}^n \frac{|v_i|}{\sqrt{x_i}}\cdot\sqrt{x_i}\right)^2
  \le\left(\sum_{i=1}^n \frac{v_i^2}{x_i}\right)\left(\sum_{i=1}^n x_i\right)
  =\ip{v}{\nabla^2 R(x)v}\cdot\norm{x}_1
\]
Hence $\ip{v}{\nabla^2 R(x)v}\ge\norm{v}_1^2/\norm{x}_1\ge\norm{v}_1^2/M$ for every $x\in\Y\cap\R_{++}^n$. The associated Bregman inequality extends to boundary comparators by the lower-semicontinuity of the unnormalized KL divergence, so $R$ is $(1/M)$-strongly convex with respect to $\ell_1$ in the sense used above.

\begin{proof}[Proof of Theorem~\ref{thm:ht_entropic_md}]
We apply Theorem~\ref{thm:md_oio_regret} with $\norm{\cdot}=\ell_1$, $\dnorm{\cdot}=\ell_\infty$, the entropy regularizer $R(x)=\sum_i x_i\log x_i$, $\sigma=1/M$ (from the strong convexity argument above), and $G_*=G_\infty$. The boundary stepsize condition reads $\eta\le\rho/(M B_\delta G_\infty)$. By assumption the uniform initialization $z_1 = (M/n, \ldots, M/n)$ lies in $\Y$; since $x_1=\0$, we have $J_1=\Y$ and hence $z_1\in J_1$.

It remains to justify the domain condition in Theorem~\ref{thm:md_oio_regret}. Suppose $z_t\in\Y\cap\R_{++}^n$ and let $\phi_t(z)=\eta\ip{g_t}{z}+D_R(z,z_t)$. Since $\Y$ is compact and $\phi_t$ is lower semicontinuous on $\Y$, a minimizer exists. No minimizer can have a zero coordinate: if $z_i=0$ for some $i$, then the feasible points $z^\varepsilon=(1-\varepsilon)z+\varepsilon z_t$ satisfy $\phi_t(z^\varepsilon)-\phi_t(z)=O(\varepsilon)+\sum_{i:z_i=0}\varepsilon z_{t,i}\log\varepsilon<0$ for all sufficiently small $\varepsilon>0$. Hence every minimizer lies in $\Y\cap\R_{++}^n$. Induction from the strictly positive $z_1$ proves that all entropy iterates remain in $\mathcal D:=\Y\cap\R_{++}^n$.

Now we must bound $D_R(u, z_1)$. Using the explicit form of $D_R$ in \eqref{eq:unnormalized_kl} with $\norm{z_1}_1 = M$,
\[
  D_R(u, z_1) = \sum_i u_i \log\!\frac{u_i n}{M} - \norm{u}_1 + M
\]
Write $s := \norm{u}_1 \in [0, M]$ and parametrize $u_i = s p_i$ with $p$ a probability vector with $\norm{p}_1 = 1$ (taking $p$ arbitrary if $s = 0$). Substituting,
\[
  D_R(u, z_1) = s \log\!\frac{sn}{M} - sH(p) - s + M
  \qquad 
\]
where $H(p) := -\sum_i p_i \log p_i$ is the Shannon entropy of $p$ and takes values in $[0, \log n]$. So we have $-sH(p) \le 0$, giving
\[
  D_R(u, z_1) \le g(s) := s \log\!\frac{sn}{M} - s + M
\]
On $[0, M]$, the function $g$ has derivative $g'(s) = \log(sn/M)$, so $g$ is decreasing on $[0, M/n]$ and increasing on $[M/n, M]$. Thus, $g$ attains its maximum at either endpoint of the interval: $g(0) = M$ or $g(M) = M\log n$. Therefore
\[
  \mathcal{B}_R \le \max\{M, M\log n\} \le M(1 + \log n)
\]

Substituting $\mathcal{B}_R \le M(1+\log n)$ and $\sigma=1/M$ in the constrained OMD bound \eqref{eq:md_master_bound_constrained} gives \eqref{eq:entropic_constrained_bound}. In the square-root branch, substituting the same quantities in \eqref{eq:md_master_rate} gives
\[
  R_T(u) \le 2 M G_\infty \sqrt{T(1 + \log n)(1/2 + B_\delta)} = \widetilde O\!\left(M G_\infty \sqrt{\frac{T\log n}{\mu}}\right).
\]
\end{proof}

\end{document}